\documentclass[11pt]{article}

\usepackage[numbers]{natbib}

\usepackage{booktabs}       
\usepackage{amsfonts}       
\usepackage{amsmath}
\usepackage{graphicx}
\usepackage{subcaption}
\usepackage{amsthm}
\usepackage{amssymb}
\usepackage{multirow}
\usepackage{wrapfig}
\usepackage{makecell}
\usepackage[vlined,linesnumbered,ruled,resetcount]{algorithm2e}
\usepackage[colorlinks,linkcolor=magenta,filecolor=blue,citecolor=blue,urlcolor=blue]{hyperref}%
\usepackage[top=1in, left=1in, right=1in, bottom=1in]{geometry}
\usepackage{amssymb}
\usepackage{pifont}
\usepackage{mathtools}
\usepackage{stmaryrd}
\usepackage[T1]{fontenc}
\usepackage{float}

\usepackage{authblk}
\usepackage{enumitem}
\usepackage{algorithmic}

\newtheorem{theorem}{Theorem}[section]
\newtheorem{lemma}[theorem]{Lemma}

\newtheorem{assumption}[theorem]{Assumption}

\theoremstyle{definition}
\newtheorem{definition}[theorem]{Definition}

\usepackage[normalem]{ulem}

\newcommand{\wh}{\widehat}
\newcommand{\wt}{\widetilde}
\newcommand{\ov}{\overline}

\renewcommand{\epsilon}{\varepsilon}
\renewcommand{\phi}{\varphi}

\renewcommand{\tilde}{\wt}
\renewcommand{\hat}{\wh}
\renewcommand{\bar}{\ov}

\makeatletter
\newcommand*{\RN}[1]{\expandafter\@slowromancap\romannumeral #1@}
\makeatother

\makeatletter
\newcommand{\printfnsymbol}[1]{%
  \textsuperscript{\@fnsymbol{#1}}%
}
\makeatother

\title{Mitigating Non-IID Drift in Zeroth-Order Federated LLM Fine-Tuning with Transferable Sparsity}

\author[1]{Yide Ran}
\author[2]{Wentao Guo}
\author[3]{Jingwei Sun}
\author[4]{Yanzhou Pan}
\author[1]{Xiaodong Yu}
\author[1]{Hao Wang}
\author[5]{Jianwen Xie}
\author[3]{Yiran Chen}
\author[1]{Denghui Zhang}
\author[1]{Zhaozhuo Xu\thanks{Contact: \texttt{zxu79@stevens.edu}}}
\affil[1]{Stevens Institute of Technology}
\affil[2]{Princeton University}
\affil[3]{Duke University}
\affil[4]{Google LLC}
\affil[5]{Lambda Inc.}
\begin{document}

\maketitle
\begin{abstract}
Federated Learning enables collaborative fine-tuning of Large Language Models (LLMs) across decentralized Non-Independent and Identically Distributed (Non-IID) clients, but such models' massive parameter sizes lead to significant memory and communication challenges. This work introduces \textsc{Meerkat}, a sparse zeroth-order optimization (ZO) method designed for federated LLM fine-tuning. By limiting fine-tuning to a transferable, static, extremely sparse subset of parameters, \textsc{Meerkat} achieves remarkable communication efficiency, enabling cost-effective high-frequency synchronization. With theoretical analysis and experiments, we show that this high-frequency communication effectively mitigates Non-IID data challenges and leads to superior performance compared to full-parameter ZO. Furthermore, experiment results show that \textsc{Meerkat} outperforms existing sparsity baselines with better performance at the same communication frequency. To further handle Non-IID drift, \textsc{Meerkat} leverages traceable local updates and forms a \textit{virtual path} for each client. This virtual path mechanism reveals the GradIP phenomenon: the inner products between LLM pre-training gradients maintained by server and client gradients estimated via ZO converges for extreme Non-IID clients but oscillates for IID ones. This distinct behavior provides a signal for identifying clients with extreme data heterogeneity.
Using this signal, \textsc{Meerkat-vp} is proposed to analyze GradIP trajectories to identify extreme Non-IID clients and applies early stopping to enhance aggregated model quality. Experiments confirm that \textsc{Meerkat} and \textsc{Meerkat-vp} significantly improve the efficiency and effectiveness of ZO federated LLM fine-tuning.
\end{abstract}

\section{Introduction}

Federated Learning (FL)~\cite{mcmahan2017communication} has emerged as a powerful paradigm for enabling decentralized collaboration, particularly relevant for fine-tuning Large Language Models (LLMs) across numerous client devices~\cite{dubey2024llama,brown2020language}. Unlike centralized training, FL allows clients to train models locally and share only model updates with a central server. However, fine-tuning LLMs in a FL setting faces two major challenges: the massive model parameter size and the Non-Independent and Identically Distributed (Non-IID) data distribution across clients. The former leads to high computation demands on clients and significant communication overhead, while the latter causes client drift and hinder global convergence. 
These challenges make LLM fine-tuning impractical on resource-constrained clients and hinder the effective use of decentralized data.

Zeroth-order Optimization (ZO) provides a promising avenue for addressing some of these challenges in federated LLM fine-tuning. By estimating gradients through model perturbations and forward passes, ZO bypasses the need for backpropagation and the storage of intermediate activations, leading to more memory-efficient learning on client devices~\cite{zhang2021desirable,fang2022communication,ling2024convergence, liu2024sparsemezoparametersbetter,malladi2023fine}. However, applying standard ZO directly to the massive parameter space of LLMs can still be computationally inefficient and the optimization process unstable~\cite{malladi2023fine}. 
Moreover, adapting ZO for federated LLM fine-tuning remains challenging, particularly in balancing computational efficiency, communication overhead, and model performance under Non-IID data heterogeneity.

In order to address the above challenges, we propose \textsc{Meerkat}, a sparse ZO method designed for efficient federated LLM fine-tuning. \textsc{Meerkat} addresses the computational and communication burdens by focusing ZO updates on a static, extremely sparse (less than $0.1\%$), and transferable subset of LLM parameters. This subset is strategically identified using gradients derived from pre-training data, ensuring that updates target parameters most sensitive to the loss function. This selective approach dramatically reduces communication overhead and supports cost-effective high-frequency synchronization. As we will demonstrate through theoretical analysis and extensive experiments, the combination of high communication frequency and sparsity in \textsc{Meerkat} enables frequent yet lightweight synchronization. This effectively reduces the convergence error floor in theory and practice, leading to consistently superior performance compared to full-parameter ZO fine-tuning and other sparsity methods under the same communication frequency.

Leveraging \textsc{Meerkat}'s efficient high-frequency synchronization to effectively mitigate Non-IID data challenges, we further enhance its adaptability to weak network conditions. By employing a virtual path mechanism to track client updates, we enable the server to analyze client training dynamics without accessing raw data, thus facilitating robust operation even when frequent direct communication is constrained. Within this virtual path, we observe the \textbf{GradIP phenomenon}, a pattern revealed by the GradIP score, which computes the inner product between local client gradients estimated via ZO and server pre-training gradients. GradIP scores converge for Non-IID clients while oscillating for IID clients, serving as a clear indicator of data heterogeneity. Leveraging this insight, we propose \textsc{Meerkat-vp} that introduces a virtual path client selection method to identify clients with significant Non-IID characteristics and apply early stopping, thereby reducing their adverse impact on the aggregated model and enhancing its quality.

In summary, this paper makes the following contributions:

\begin{itemize}[nosep,leftmargin=*]
    \item \textbf{Performance Improvement with Sparsity.} Meerkat consistently outperforms full-parameter ZO optimization in both IID and Non-IID settings, demonstrating the effectiveness of our sparse update strategy. Extensive experiments show that Meerkat surpasses not only full-parameter ZO but also other sparse methods, such as LoRA and weight-magnitude, achieving superior performance.
    \item \textbf{High Frequency Communication with Sparsity Can Lower the Error Floor.} \textsc{Meerkat} leverages extreme model sparsity to reduce local computational memory. Exchanging scalar gradients drastically decreases communication costs, enabling high-frequency communication.
    \item \textbf{Traceable Local Updates and GradIP Phenomenon}: \textsc{Meerkat} leverages traceable sparse local updates and forms a \textit{virtual path}. The virtual paths reveals the GradIP phenomenon: the inner product between LLM pre-training gradients maintained by server and client gradients estimated via ZO converges for extreme Non-IID clients but oscillates for IID ones. This distinct behavior serves as a signal for detecting clients with extreme data heterogeneity.
    \item \textbf{\textsc{Meerkat-vp}: Early Stopping for Extreme Non-IID Clients.} Leveraging the GradIP phenomenon via virtual path client selection, \textsc{Meerkat-vp} effectively manages extreme Non-IID clients, by early stopping these clients to improve global model quality.
    \item \textbf{Theoretical and Experimental Validation.} We present theoretical analysis and extensive experiments across diverse FL settings, validating the scalability and performance benefits of both \textsc{Meerkat} and \textsc{Meerkat-vp}.
\end{itemize}
\begin{figure*}[!h]
    \vspace{-4pt}
    \centering
    \includegraphics[width=0.9\textwidth]{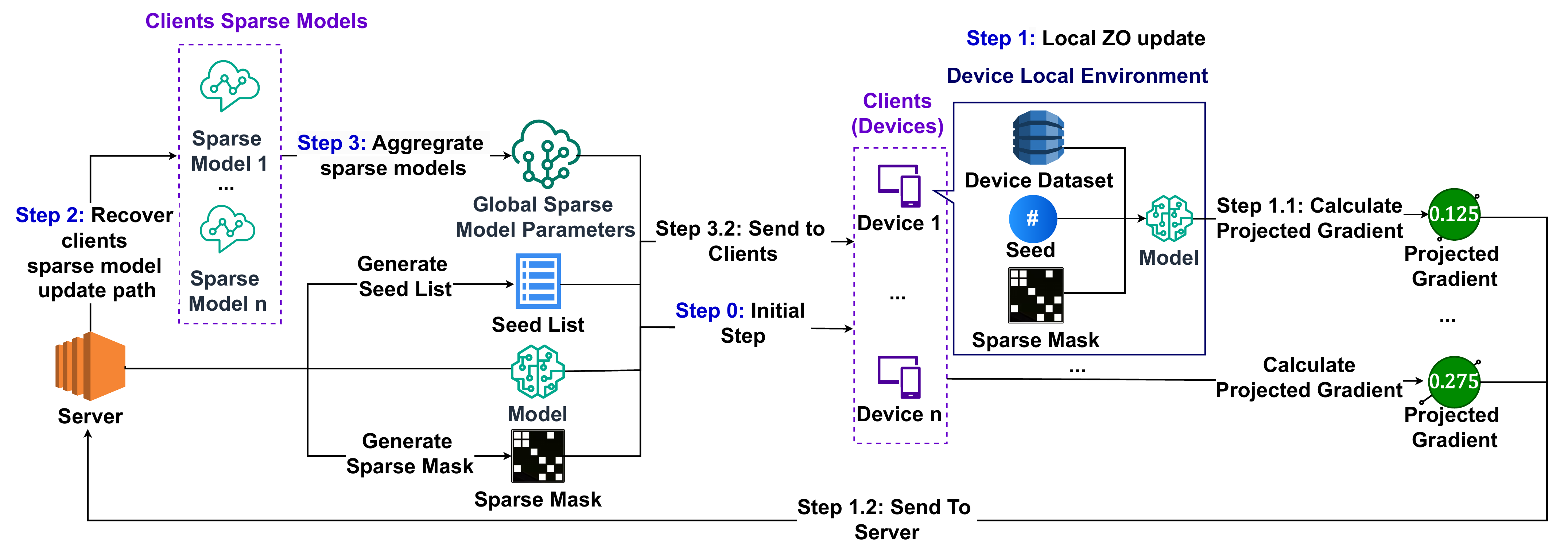}
    \caption{\textsc{Meerkat}: Sparse zeroth-order optimization for federated LLM fine-tuning workflow.}
    \vspace{-4pt}
    \label{fig:overview}
    \vspace{-4pt}
\end{figure*}

\section{ Sparse Zeroth-Order Optimization for Federated LLM Fine-tuning}\label{sec:method}

This section introduces \textsc{Meerkat}, a sparse ZO method for federated LLM fine-tuning, and its upgraded version, \textsc{Meerkat-vp}, which incorporates Virtual Path Client Selection(VPCS) strategy. This strategy leverages the traceable virtual path of client local updates to identify clients with extremely Non-IID data and applies early stopping to mitigate their adverse impact on global model convergence. We first introduce the technical details of \textsc{Meerkat}, as illustrated in Figure~\ref{fig:overview}, and subsequently describe \textsc{Meerkat-vp}, shown in Figure~\ref{fig:vpoverview}. We then present theoretical convergence analysis for both methods and discuss their strengths in terms of cost-effectiveness, traceability, and the use of early stopping to mitigate client drift caused by Non-IID data.  

\subsection{\textsc{Meerkat}: Extreme Sparse Zeroth-Order Federated LLM Fine-Tuning}

\textbf{Sparse ZO On-Device LLM Fine-Tuning.}
\textsc{meerkat} performs sparse ZO for LLM fine-tuning on the client device. Let $\mathcal{D}$ denote the client dataset we would like an LLM to fine-tune with loss function $f$. Given the LLM weight $\mathbf{w} \in \mathbb{R}^d$, we perform an iterative optimization by randomly sampling a batch $\mathcal{B} \subset \mathcal{D}$ for each step and performing the local update step as 
\begin{equation}\label{eq:sparse_zo}
  g = \frac{f(\mathbf{w} + \epsilon(\mathbf{z}\odot\mathbf{m});\mathcal B)
            - f(\mathbf{w} - \epsilon(\mathbf{z}\odot\mathbf{m});\mathcal B)}
           {2\epsilon}
  \,,\quad
  \hat{\nabla} f = g\,(\mathbf{z}\odot\mathbf{m}) \,.
\end{equation}
where $\mathbf{z} \in \mathbb{R}^d$ is a random vector sampled from a Gaussian distribution $\mathcal{N}(0, I_d)$, 
$\epsilon \in \mathbb{R}$ is the perturbation magnitude, and $\mathbf{m} \in \{0, 1\}^d$ is a binary sparse mask with density ratio $u$ that selects a subset of parameters for updates.

\textbf{Extremely Sparse Parameters Obtained from Pre-Training.} 
According to the formulation in Eq~\eqref{eq:sparse_zo}, we focus the perturbation of the LLM on a subset of parameters determined by a binary mask $\mathbf{m}$. 
The mask $\mathbf{m}$ is derived from the pre-training process of the LLM. 
We compute the average squared gradients of each parameter over a subset of the C4 dataset~\cite{raffel2020exploring}. 
Then, we select the top $u$ parameters with the highest average squared gradient values and mark them as $1$ in $\mathbf{m}$. 
In practice, we set $u$ to $0.1\%$, resulting in extremely sparse updates.

\textbf{FL with \textsc{Meerkat}.} The workflow of \textsc{Meerkat} is illustrated in Figure~\ref{fig:overview} and Algorithm~\ref{alg:fl-szogeneral}.  \textsc{Meerkat} first loads each client with the pre-trained weight $\mathbf{w}_0$ and the sparse mask $\mathbf{m}$. Next, \textsc{Meerkat} initializes a random seed list $\{s_1^1, \dots, s_1^T\}$ at the server to generate the random Gaussian vector $\mathbf{z}$ for each local step in the first round.
Next, \textsc{Meerkat} performs an iterative federated optimization with $R$ rounds of client-server synchronization with each round as follows. 

(1) \textit{Local ZO update at each client.} 
Upon receiving global model weights $\mathbf{w}_{r-1}$ and seed list $\{s_r^1, \dots, s_r^T\}$ from the server, each client performs $T$ local iteration steps.
In each local step $t$, the client perturbs the model parameters selected by $\mathbf{m}$ with the random vector $\mathbf{z}_k^t$ generated by the random seed $s_r^t$. Each client then computes projected gradient $g_k^t$ (a scalar) according to Eq.~\eqref{eq:sparse_zo}. Using $g_k^t$, each client calculates the local gradient $\hat{\nabla} f_k^t$ and updates the local model $w_k$ with learning rate $\eta$. After $T$ local steps, each client uploads a list of projected gradients $\{g_k^1, g_k^2, \dots, g_k^T\}$ to the server. 
(2) \textit{Server reconstructs client update with virtual path.}
Since the server shares the same random seed list with clients for the round, it can reconstruct each client's local model update path upon receiving their projected gradients. We term this server-side reconstruction process the \textit{virtual path}, as it allows the server to follow the client's local steps without accessing raw data. As shown in Step 2 of Algorithm~\ref{alg:fl-szogeneral}, the server uses the preserved random seed and receives project gradients of each local step from each client to recover the local model update path for each client. 
(3) \textit{Sever aggregates and initiate the next round:} After virtual path reconstruction, the server aggregates the reconstructed client model weights $\mathbf{w}_k^T$ to sparsely update the global model to $\mathbf{w}_{r}$. Subsequently, the server sends $\mathbf{w}_{r}$ and a a new seed list $\{s_{r+1}^1, \dots, s_{r+1}^T\}$ to clients and initializes next round.

\textbf{\textsc{Meerkat-vp}: Virtual Path Client Selection and Early Stopping}. \textsc{Meerkat-vp} extends \textsc{Meerkat} by incorporating a VPCS strategy designed for heterogeneous environments. Leveraging the virtual path reconstruction capability, the server analyzes client update trajectories to identify those with extremely Non-IID data distributions. \textsc{Meerkat-vp} then applies an early stopping mechanism to these identified clients, restricting them to a single local step to mitigate the negative impact of their skewed updates on global model convergence and performance.

\subsection{Theoretical Convergence Analysis}
We theoretically analyze the convergence of \textsc{Meerkat} and \textsc{Meerkat-vp} under the Polyak–Łojasiewicz (PL)-type non-convex condition. All technical assumptions and the corresponding proof are presented in Appendix~\ref{sec:theoretical_proof}.

\begin{theorem}[Convergence rate of \textsc{Meerkat}]
\label{thm:global-convergence:formal}Under Assumptions~\ref{assumption:assumption1}--\ref{assumption:assumption5}, if the learning rate satisfy $\eta = \min\left\{ \frac{1}{L(u+2)}, \frac{\mu\,\sqrt{c}\bigl(1+\sqrt{c_h}\bigr)}
        {2\,L^{2}(2+u)^{2}} \right\}$, then the global model \(\{ \mathbf{w}^r \}\) generated by the \textsc{Meerkat} algorithm satisfies the following convergence bound:
\begin{equation}\label{eq:formal_theory_meerkat}
\begin{aligned}
\frac{1}{R} \sum_{r=0}^{R-1} \left( f(w^r) - f^* \right) &\leq \mathcal{O}\!\left( \frac{(2+u)^2}{TR} \cdot \mathbb{E}[f(w^0) - f(w^R)] \right)
+ \mathcal{O}\!\left( \frac{T}{2+u} \right)
+ O\!\bigl(1\bigr)
\end{aligned}
\end{equation}
\end{theorem}

\begin{theorem}[Convergence rate of \textsc{Meerkat-vp}]
\label{thm:global-convergence-merkatvp:formal}Under Assumptions~\ref{assumption:assumption1}--\ref{assumption:assumption5}, if the learning rate satisfies $\eta = \min\left\{ \frac{1}{L(u+2)}, \frac{\mu\,\sqrt{c}\,\bigl(K_g\,T + K_b\bigr)}
     {2\,K\,(2+u)^2\,L^2\,T\,\gamma} \right\}$ and each client $k \in K_b$ performs $T = 1$ local step while the remaining $K_g$ clients perform $T$ local steps, then the global model \(\{ \mathbf{w}^r \}\) generated by the \textsc{Meerkat-vp} algorithm satisfies the following convergence bound:
\begin{equation}\label{eq:vpcs_convergence_bigO_informal}
\begin{aligned}
\frac{1}{R}\sum_{r=0}^{R-1}\mathbb{E}_{\bar z}\bigl[f(w^r)-f^*\bigr]
&\le
O\!\biggl(\frac{(K_g+K_b)^2\,(2+u)^2\,\gamma\,T}{c\,(K_gT+K_b)^2\,R}\biggr) +\,O\!\biggl(\frac{1+u}{K_g+K_b}\sum_{k=1}^{K_g+K_b}\Delta_k\biggr) \\[4pt]
&\quad +\,O\!\biggl(\frac{c\,T\,K_g}{(K_g+K_b)(1+u)\,\gamma}\biggr) +\,O\!\biggl(\frac{c\,K_b\,\sigma_h^2}{(K_g+K_b)(1+u)\,T\,\gamma}\biggr)
+O\!\bigl(1\bigr).
\end{aligned}
\end{equation}
\end{theorem}

The detailed theoretical analysis and proofs for Theorem~\ref{thm:global-convergence:formal} (\textsc{Meerkat}) can be found in Appendix~\ref{sec:meerkat convergence analysis}, and for Theorem~\ref{thm:global-convergence-merkatvp:formal} (\textsc{Meerkat-vp}) in Appendix~\ref{sec:vp-client-selection}.

\textbf{Insights of \textsc{Meerkat}.} 
\textsc{Meerkat}'s convergence reveals the intricate interplay of local steps $T$ and density $u$ on performance. (1) \textit{\textsc{Meerkat}'s sparsity can theoretically improve performance.} Lower $u$ (higher sparsity) quadratically benefits the rate-dependent term $(\propto (2+u)^2)$, favoring faster initial convergence. However, it inflates the steady-state error \(\bigl(\propto \tfrac{1}{2 + u}\bigr)\). Comparing to the full-parameter case $(u = 1)$, sparsity $(u < 1)$ can reduce the overall bound by decreasing the rate-dependent term, offering communication and computational benefits. Yet, excessive sparsity can increase the steady-state error, suggesting an optimal density level \(u \in (0,\,1]\). (2) \textit{High frequency communication with sparsity can lower the error floor.} Increasing $T$ improves the transient term scaling with $\mathcal{O}\!\bigl(\frac{(2 + u)^2}{T\,R}\bigr)$, potentially accelerating convergence towards the steady state; however, it expands the steady-state term $\mathcal{O}\!\bigl(\tfrac{T}{2 + u}\bigr)$, thereby increasing the error floor. Conversely, decreasing $T$ reduces the steady-state term, leading to a tighter final accuracy. Although smaller $T$ can lead to larger rate-dependent term. It's impact diminishes as the number of rounds $R$ increases. This analysis suggests that operating with frequent communication can theoretically reduce the steady-state error.

\textbf{Advantages of \textsc{Meerkat-vp}.} 
We compare each component of the error bound under the same \(T\) and \(R\). First, the \emph{transient term ratio} between \textsc{Meerkat-vp} and \textsc{Meerkat} is approximately \(\gamma(1+\sqrt{c_h})^2 < 1\), and as \(c_h \to 1\) so $\gamma \to 0$, the product \(\gamma(1+\sqrt{c_h})^2 \to 0\), causing the transient error to vanish. Second, the \emph{noise term ratio} is given by \(\frac{\sigma_h^2/2}{\sigma_h^2 / (\mu(1+\sqrt{c_h})^2)} = \frac{\mu(1+\sqrt{c_h})^2}{2}\), which remains below 1 whenever \(\mu(1+\sqrt{c_h})^2 < 2\). Since \(\mu < 1\) empirically, this condition typically holds. Moreover, \textsc{Meerkat-vp} introduces an additional variance term \(\frac{c K_b \sigma_h^2}{2K (2 + u) L T \gamma}\) that decays as \(\mathcal{O}(1/T)\), making it negligible for large local steps. Lastly, in terms of heterogeneity, the coefficient of the heterogeneity term \(\sum_k \Delta_k\) in \textsc{Meerkat-vp} is smaller: \(\frac{(2+u)L}{4K} < \frac{L}{K}\), and the extra variance term scales inversely with \(K\), thus diminishing in larger systems. Therefore, $E_{\mathrm{\textsc{Meerkat-vp}}}
\;<\;
E_{\mathrm{\textsc{Meerkat}}}
$ and this gap widens as data heterogeneity \(c_h\) increases. The detailed mathematical derivations and analysis, please refer to the Appendix~\ref{sec:vp-client-selection}.

\subsection{Claim 1: \textsc{Meerkat} Can Outperforms Full-Parameter Federated ZO Under Same Synchronization Frequency}\label{subsec:meerkat_outperform}
We claim that with fixed and extreme sparsity, \textsc{Meerkat} outperforms full-parameter ZO in federated LLM fine-tuning under the same synchronization frequency and effectively mitigates the Non-IID client data problem through frequent synchronization and sparsity.

\textbf{Advantages of Sparsity in Federated ZO.} ZO has an intrinsic need for sparsity due to its reliance on nearly uniform perturbations across dimensions. Research on ZO shows that selecting sensitive parameters using gradient-based methods consistently outperforms alternative strategies such as weight magnitude or random parameter selection~\cite{guo2024zerothorderfinetuningllmsextreme}. 
Following this idea, \textsc{Meerkat} produces LLM-sensitive parameters with gradient-based sparsification on pre-training data such as C4~\cite{raffel2020exploring}. Moreover, 
\textsc{Meerkat} fine-tunes LLMs by estimating gradients through forward passes, completely bypassing backpropagation. This approach minimizes the need to cache gradients and activations, leading to significant memory savings. Focusing on sensitive parameters — those with the greatest influence on loss value changes upon random perturbation — \textsc{Meerkat} ensures efficient and effective fine-tuning even under extreme sparsity levels (e.g., updating only $0.1\%$ of the parameters). Furthermore, these sensitive parameters, derived from pre-training datasets like C4, exhibit transferability across downstream tasks, allowing for parameter-efficient ZO. Theoretical analysis (Appendix~\ref{sec:meerkat convergence analysis}) also confirms that lower density $u$ leads to faster convergence via improved rate-dependent terms $\mathcal{O}((2+u)^2/(TR))$, while excessive sparsity increases the steady-state error $\mathcal{O}(T/(2+u))$, suggesting an optimal sparsity trade-off.

\textbf{Performance Under High Synchronization Frequency.} The lightweight communication of \textsc{Meerkat} enables frequent client-server synchronization at a low cost, which is crucial for addressing data heterogeneity~\cite{yang2024fedfed, mendieta2022local} in FL. In high-frequency communication scenarios, both the clients and the server only exchange a list of scalars (projected gradients) whereas in lower-frequency synchronization, clients have to upload projected gradients but still download sparse model parameters. By eliminating the need to download sparse model parameters in high-frequency synchronization, this approach is significantly more bandwidth-efficient, further minimizing communication overhead.  We present the high-frequency synchronization algorithm of \textsc{Meerkat} in Appendix~\ref{sec:theoretical_proof} Algorithm~\ref{alg:fl-szofreq}. By facilitating frequent synchronization, training can better prevent clients from drifting. Our previous theoretical analysis also demonstrates that a smaller $T$ might influence the rate-dependent term, its beneficial impact on reducing the steady-state error is significant for achieving a tighter final accuracy over many rounds $R$.

\subsection{Claim 2: Empirical GradIP Phenomenon Reveals Data Heterogeneity}\label{subsec:gradip_phenomenon}
\textsc{Meerkat}'s traceable virtual path allows us to analyze client local training dynamics, revealing an empirical phenomenon related to data heterogeneity via a metric we call GradIP.
\begin{definition} Gradient Inner Product (GradIP) score: \label{def:gradip}
    let $\hat{\nabla}f_k^t$ (see Algorithm~\ref{alg:fl-szogeneral}) denote the gradient of LLM computed by ZO with Eq~\eqref{eq:sparse_zo} on client $k$ at local step $t$. Let $\nabla f_p$ denote the gradient of LLM computed by backpropagation on pre-training data. We define the GradIP score as $\langle \nabla f_p, \hat{\nabla} f_k^t\rangle$.
\end{definition}
\textbf{GradIP As Indicator for Data Heterogeneity.}
Leveraging the virtual path reconstruction capability of \textsc{Meerkat}, the server can trace each client's local training trajectory. This process uses the uploaded projected gradients $g_k^t$ along with the shared random seeds (which regenerate $\mathbf{z}_k^t$) and the sparse mask $\mathbf{m}$ to reconstruct the local gradient $\hat{\nabla}f_k^t$. To understand the impact of a client's local data distribution on its training process, we introduce the \textit{GradIP} metric. Inspired by the use of pre-training data gradients to identify sensitive parameters, GradIP quantifies the cosine similarity between the local gradient computed during client training and the LLM pre-training gradient.

\textbf{Empirical GradIP Phenomenon.}
Through the traceable virtual path provided by \textsc{Meerkat}, we empirically investigated the behavior of the GradIP score among clients with different data distributions (IID and Non-IID) over their local training steps. Our analysis, presented in Appendix~\ref{sec:empirical_analysis_of_gradip_phenomenon_proof}, demonstrates distinct patterns in the dynamics of gradient norms based on data heterogeneity. While IID client gradient norms exhibit fluctuations, those of extremely Non-IID clients decay and converge towards zero. The GradIP definition depends on the fixed pre-training gradient norm, local client gradient norm, and the angle $\theta$ between them. Given the data dissimilarity, we hypothesize that $\theta$ between these two gradient vectors is nearly orthogonal. This leads us to expect a different manifestation of the GradIP Phenomenon when comparing IID and extremely Non-IID clients, primarily influenced by their differing local gradient norm trajectories.

\begin{wrapfigure}[30]{R}{0.53\textwidth}
\vspace{-8mm}
  \begin{minipage}{0.53\textwidth}
    \begin{algorithm}[H]
    \setlength{\intextsep}{0pt}
    \setlength{\floatsep}{0pt}
    \caption{\textsc{Meerkat-vp}}
    \label{algo:Meerkat_vp}
    \footnotesize
    \begin{algorithmic}[1]
        \STATE {\bfseries Input:} calibration step $T_{\textsf{cali}}$, pre-training gradients $\nabla f_{\text{C4}}$, projected gradients 
        $\{g^{1}_k,\dots,g_k^{T_{\textsf{cali}}}\}$,
        seed $s_r^t$, sparse mask $\mathbf{m}$, initial phase steps $T_{\textsf{init}}$, later phase steps $T_{\textsf{later}}$, convergence threshold $\sigma$, Initial to later ratio $\rho_{\textsf{later}}$, quiescen step ratio $\rho_{\textsf{quie}}$
        \STATE \textbf{Step 1: Virtual Path Reconstruction \& GradIP Calculation}
        \STATE Generate $\mathbf{z}_k^t$ using $s_r^t$. 
        \STATE Compute $\hat{\nabla}f_k^t = g_k^t \cdot (\mathbf{z}_k^t \odot \mathbf{m})$ 
        \STATE Compute $\mathsf{Gradip} = \hat{\nabla} f_k^t \cdot \nabla f_{\text{C4}}$ (Definition~\ref{def:gradip}). 
        \STATE \textbf{Step 2: Identify Extremely Non-IID Clients}
        \STATE Compute the average value of $\mathsf{Gradip}$ over the initial-phase steps.
        \vspace{-4pt}
        \begin{equation*}
        \mathsf{Gradip}_{\text{init\_avg}} = \frac{1}{T_{\textsf{init}}} \sum_{t=1}^{T_{\textsf{init}}} \mathsf{Gradip}_t
        \end{equation*}
        \vspace{-8pt}
        \STATE Compute the average value of $\mathsf{Gradip}$ over the later-phase steps.
        \vspace{-4pt}
        \begin{equation*}
        \mathsf{Gradip}_{\text{later\_avg}} = \frac{1}{T_{\textsf{later}}} \sum_{t=1}^{T_{\textsf{later}}} \mathsf{Gradip}_t
        \end{equation*}
        \vspace{-8pt}
        \STATE Compute the client's Initial to later ratio $\rho_{\textsf{later\_client}}$ and quiescent step ratio $\rho_{\textsf{quie\_client}}$
        \vspace{0pt}
        \begin{equation*}
        \rho_{\textsf{quie\_client}} = \frac{\{s \in \{1,2,\ldots,T_{\textsf{later}}\} \mid \mathsf{Gradip}_s < \sigma\}}{T_{\textsf{later}}}
        \end{equation*}
        \vspace{0pt}
        \begin{equation*}
        \rho_{\textsf{later\_client}} = \frac{\mathsf{Gradip}_{\text{init\_avg}}}{\mathsf{Gradip}_{\text{later\_avg}}}
        \end{equation*}
        \vspace{-4pt}
        \STATE Record client IDs whose $\rho_{\textsf{later\_client}}$ or $\rho_{\textsf{quie\_client}}$ exceed $\rho_{\textsf{later}}$ or $\rho_{\textsf{quie}}$. 
        \STATE \textbf{Step 3: Early Stopping}
        \STATE Require these identified clients to only perform one local training step.
    \end{algorithmic}
    \end{algorithm}
  \end{minipage}
\end{wrapfigure}
\subsection{Claim 3: Virtual Path Client Selection via GradIP Analysis}\label{subsec:vpcsclaim}

Building upon the traceable virtual path capability introduced in \textsc{Meerkat}, we claim that VPCS, by leveraging GradIP analysis, effectively identifies and manages clients with extremely Non-IID data distribution, thereby improving global model performance and convergence.
As established in Section~\ref{subsec:gradip_phenomenon}, the GradIP score, computable by the server through virtual path reconstruction, provides a effective signal to identify such clients. VPCS utilizes this GradIP signal to detect extremely Non-IID clients. By analyzing the GradIP score trajectory and its behavior over local steps during a calibration phase, using metrics defined in Appendix table~\ref{tab:meerkat_vp_parameter_notation}, the server empirically identifies clients exhibiting the characteristic diminishing GradIP behavior associated with extremely Non-IID data distribution. Upon identification via GradIP analysis, VPCS applies early stopping: these clients perform only one local training step per communication round. To ensure full data utilization over training, a data pointer tracks the batch processed, allowing clients to resume from that point in subsequent rounds. This strategy mitigates client drift from skewed data while ensuring their entire dataset is eventually processed. Algorithm~\ref{algo:Meerkat_vp} outlines the detailed procedure, and Figure~\ref{fig:vpoverview} illustrates the workflow. Our previous theoretical analysis of \textsc{Meerkat-vp} suggests that early stopping on extremely Non-IID clients can lead to improved global model performance.

\section{Experiment}
\label{sec:experiment}

In this section, we aim to validate the effectiveness of \textsc{Meerkat} and \textsc{Meerkat-vp}. We aim to address the following research questions in response to claims in Section~\ref{sec:method}.

\begin{itemize}[nosep,leftmargin=*]
    \item \textbf{RQ 1 for Claim 1 (\ref{subsec:meerkat_outperform}):} Is \textsc{Meerkat} more effective than full parameter federated ZO under the same synchronization frequency, especially in heterogeneous environments?
    \item \textbf{RQ 2 for Claim 2 (\ref{subsec:gradip_phenomenon}):} Can the empirical GradIP phenomenon, observed via the virtual path, effectively reveal data heterogeneity by showing distinct behaviors for IID and Non-IID clients?
    \item \textbf{RQ 3 for Claim 3 (\ref{subsec:vpcsclaim}):} Can \textsc{Meerkat-vp}, leveraging GradIP analysis, mitigate the impact of extreme Non-IID data compared to \textsc{Meerkat}?
\end{itemize}

We focus on models Gemma-2-2b~\cite{gemma_2024}, Qwen2-1.5B~\cite{qwen2}, Llama-3.2-1B~\cite{dubey2024llama}. We conduct experiments on SST2~\cite{socher2013recursive}, AG’s News~\cite{zhang2015character}, Yelp polarity (yelp)~\cite{zhang2015character}, RTE~\cite{wang2018glue}, BoolQ~\cite{clark2019boolq}, WSC~\cite{levesque2012winograd}, WiC~\cite{pilehvar2018wic} datasets. The datasets are partitioned across clients following a Dirichlet distribution to simulate clients with Non-IID data. For more experimental settings, we refer the readers to Appendix~\ref{sec:experiment_settings}.

\begin{table*}[!htbp]
\centering
\captionsetup{width=\textwidth, justification=centering}
\caption{Performance comparison of \textsc{Meerkat} and Full-FedZO on tasks SST-2, AgNews, Yelp, BoolQ, RTE, WSC, WIC under an Non-IID setting. 
``Acc'' is the average test accuracy across tasks. Bold numbers indicate the highest value in each row.}
\label{tab:comparison_full_vs_sparse_noniid}
\resizebox{0.9\textwidth}{!}{
\begin{tabular}{llc cccccccc}
\toprule
& \textbf{Methods} & \textbf{Local Step} & \textbf{SST-2} & \textbf{AgNews} & \textbf{Yelp} & \textbf{BoolQ} & \textbf{RTE} & \textbf{WSC} & \textbf{WIC} & \textbf{Acc} \\
\midrule
\multirow{8}{*}{\textbf{LLaMA-3.2-1B}}
& Full-FedZO & 10     & 0.909 & 0.705 & 0.940 & 0.641 & 0.542 & 0.634 & 0.523 & 0.699 \\
& Weight Magnitude & 10     & 0.902 & 0.857 & 0.951 & 0.696 & 0.551 & 0.519 & 0.546 & 0.717 \\
& Lora-FedZO & 10     & 0.901 & 0.749 & 0.96 & 0.649 & 0.524 & 0.634 & 0.59 & 0.715 \\
& \textsc{Meerkat} & 10 & \textbf{0.916} & \textbf{0.872} & \textbf{0.964} & 0.695 & \textbf{0.600} & \textbf{0.653} & \textbf{0.614} & \textbf{0.759} \\
\cmidrule(lr){2-11}
& Full-FedZO & 30     & 0.904 & 0.706 & 0.935 & 0.636 & 0.533 & 0.634 & 0.539 & 0.698 \\
& Weight Magnitude & 30     & 0.902 & 0.84 & 0.946 & 0.674 & 0.542 & 0.556 & 0.550 & 0.716 \\
& Lora-FedZO & 30     & 0.904 & 0.556 & 0.964 & 0.652 & 0.533 & 0.634 & 0.545 & 0.684 \\
& \textsc{Meerkat} & 30 & 0.897 & \textbf{0.862} & \textbf{0.965} & 0.646 & \textbf{0.577} & \textbf{0.644} & \textbf{0.583} & \textbf{0.739} \\
\cmidrule(lr){2-11}
& Full-FedZO & 50 & 0.889 & 0.696 & 0.935 & 0.633 & 0.542 & 0.634 & 0.529 & 0.694 \\
& Weight Magnitude & 50     & 0.897 & 0.838 & 0.948 & 0.662 & 0.551 & 0.562 & 0.554 & 0.716 \\
& Lora-FedZO & 50     & 0.876 & 0.447 & 0.967 & 0.639 & 0.541 & 0.634 & 0.562 & 0.667 \\
& \textsc{Meerkat} & 50 & \textbf{0.909} & 0.827 & \textbf{0.965} & 0.647 & \textbf{0.595} & 0.634 & \textbf{0.567} & \textbf{0.734} \\
\cmidrule(lr){2-11}
& Full-FedZO & 100  & 0.901 & 0.705 & 0.939 & 0.632 & 0.533 & 0.634 & 0.525 & 0.695 \\
& Weight Magnitude & 100     & 0.885 & 0.83 & 0.946 & 0.66 & 0.56 & 0.534 & 0.548 & 0.709 \\
& Lora-FedZO & 100     & 0.868 & 0.247 & 0.953 & 0.642 & 0.521 & 0.634 & 0.529 & 0.628 \\
& \textsc{Meerkat} & 100 & 0.896 & 0.777 & \textbf{0.961} & 0.658 & \textbf{0.577} & \textbf{0.644} & \textbf{0.573} & \textbf{0.726} \\
\midrule
\multirow{8}{*}{\textbf{Qwen2-1.5b}}
& Full-FedZO & 10 & 0.888 & 0.700 & 0.928 & 0.694 & 0.808 & 0.673 & 0.639 & 0.761 \\
& Weight Magnitude & 10     & 0.881 & 0.84 & 0.939 & 0.681 & 0.795 & 0.672 & 0.623 & 0.776 \\
& Lora-FedZO & 10     & 0.939 & 0.847 & 0.944 & 0.667 & 0.795 & 0.663 & 0.521 & 0.768 \\
& \textsc{Meerkat} & 10 & \textbf{0.949} & \textbf{0.881} & 0.934 & \textbf{0.752} & \textbf{0.813} & \textbf{0.682} & 0.628 & \textbf{0.805} \\
\cmidrule(lr){2-11}
& Full-FedZO & 30  & 0.892 & 0.699 & 0.926 & 0.708 & 0.791 & 0.663 & 0.594 & 0.753 \\
& Weight Magnitude & 30     & 0.88 & 0.843 & 0.939 & 0.681 & 0.786 & 0.673 & 0.594 & 0.771 \\
& Lora-FedZO & 30     & 0.923 & 0.843 & 0.948 & 0.666 & 0.777 & 0.673 & 0.519 & 0.764 \\
& \textsc{Meerkat} & 30 & \textbf{0.944} & \textbf{0.878} & 0.928 & \textbf{0.734} & \textbf{0.800} & 0.663 & \textbf{0.624} & \textbf{0.795} \\
\cmidrule(lr){2-11}
& Full-FedZO & 50     & 0.868 & 0.696 & 0.922 & 0.707 & 0.773 & 0.663 & 0.594 & 0.746 \\
& Weight Magnitude & 50     & 0.883 & 0.855 & 0.938 & 0.703 & 0.768 & 0.673 & 0.595 & 0.774 \\
& Lora-FedZO & 50     & 0.934 & 0.834 & 0.941 & 0.679 & 0.76 & 0.653 & 0.510 & 0.759 \\
& \textsc{Meerkat} & 50 & \textbf{0.948} & \textbf{0.872} & 0.926 & \textbf{0.746} & \textbf{0.795} & 0.663 & 0.594 & \textbf{0.792} \\
\cmidrule(lr){2-11}
& Full-FedZO & 100    & 0.864 & 0.691 & 0.917 & 0.675 & 0.777 & 0.653 & \textbf{0.620} & 0.742 \\
& Weight Magnitude & 100     & 0.888 & 0.842 & 0.934 & 0.695 & 0.768 & 0.656 & 0.579 & 0.766 \\
& Lora-FedZO & 100     & 0.934 & 0.785 & 0.937 & 0.664 & 0.786 & 0.653 & 0.512 & 0.753 \\
& \textsc{Meerkat} & 100 & \textbf{0.936} & \textbf{0.878} & 0.925 & \textbf{0.741} & \textbf{0.795} & \textbf{0.663} & 0.610 & \textbf{0.792} \\
\midrule
\multirow{8}{*}{\textbf{Gemma2-2b}}
& Full-FedZO & 10     & 0.928 & 0.721 & 0.943 & 0.731 & 0.564 & 0.644 & 0.595 & 0.732 \\
& Weight Magnitude & 10     & 0.931 & 0.849 & 0.955 & 0.778 & 0.711 & 0.634 & 0.595 & 0.779 \\
& Lora-FedZO & 10     & 0.936 & 0.853 & 0.966 & 0.763 & 0.568 & 0.663 & 0.605 & 0.765 \\
& \textsc{Meerkat} & 10 & \textbf{0.939} & \textbf{0.869} & 0.96 & \textbf{0.804} & 0.591 & 0.634 & \textbf{0.609} & 0.772 \\
\cmidrule(lr){2-11}
& Full-FedZO & 30     & 0.927 & 0.802 & 0.932 & 0.725 & 0.568 & 0.634 & 0.581 & 0.738 \\
& Weight Magnitude & 30     & 0.935 & 0.851 & 0.951 & 0.771 & 0.653 & 0.634 & 0.598 & 0.770 \\
& Lora-FedZO & 30     & 0.932 & 0.804 & 0.966 & 0.671 & 0.551 & 0.634 & 0.589 & 0.735 \\
& \textsc{Meerkat} & 30 & \textbf{0.94} & \textbf{0.855} & 0.947 & 0.734 & 0.568 & \textbf{0.644} & \textbf{0.601} & 0.756 \\
\cmidrule(lr){2-11}
& Full-FedZO & 50     & 0.932 & 0.791 & 0.943 & 0.712 & 0.582 & \textbf{0.634} & 0.567 & 0.737 \\
& Weight Magnitude & 50     & 0.936 & 0.851 & 0.941 & 0.745 & 0.591 & 0.628 & 0.597 & 0.756 \\
& Lora-FedZO & 50     & 0.91 & 0.779 & 0.942 & 0.664 & 0.557 & \textbf{0.634} & 0.597 & 0.726 \\
& \textsc{Meerkat} & 50 & \textbf{0.945} & \textbf{0.857} & \textbf{0.966} & \textbf{0.767} & \textbf{0.613} & \textbf{0.634} & \textbf{0.623} & \textbf{0.772} \\
\cmidrule(lr){2-11}
& Full-FedZO & 100    & 0.925 & 0.818 & 0.933 & 0.672 & 0.533 & 0.615 & 0.567 & 0.723 \\
& Weight Magnitude & 100     & 0.922 & 0.839 & 0.942 & 0.723 & 0.568 & 0.644 & 0.592 & 0.747 \\
& Lora-FedZO & 100     & 0.922 & 0.247 & 0.942 & 0.62 & 0.541 & 0.634 & 0.573 & 0.640 \\
& \textsc{Meerkat} & 100 & \textbf{0.94} & \textbf{0.851} & \textbf{0.951} & \textbf{0.745} & 0.551 & 0.634 & 0.574 & \textbf{0.749} \\
\bottomrule
\end{tabular}
}
\vspace{-6 mm}
\end{table*}

\subsection{Answer to RQ1: Superiority of \textsc{Meerkat} Compared to Full-FedZO in FL}
\label{subsec:rq1_results}

This section experimentally validates Claim 1 (Section~\ref{subsec:meerkat_outperform}), demonstrating \textsc{Meerkat}'s superiority over full-parameter Federated ZO under the same synchronization frequency and its effectiveness in mitigating Non-IID challenges via high-frequency synchronization.

First, to support the argument for the Advantages of Sparsity, we compared \textsc{Meerkat} against full-parameter Full-FedZO and other sparsity methods (Weight Magnitude, LoRA-FedZO) under equivalent synchronization frequencies (fixed local steps $T=10, 30, 50, 100$). \textsc{Meerkat}'s inherent $0.1\%$ sparsity provides over 1000x communication savings compared to Full-FedZO. Results (Tables~\ref{tab:comparison_full_vs_sparse_noniid}, \ref{tab:comparison_full_vs_sparse_iid}) consistently show \textsc{Meerkat} outperforming Full-FedZO on most tasks. \textsc{Meerkat} also outperforms other sparsity methods on many tasks. These findings validate its superiority at the same communication frequency. Notably, under the same setting, \textsc{Meerkat} demonstrated better performance compared to DeComFL \cite{li2024achievingdimensionfreecommunicationfederated}. The detailed result in Table~\ref{tab:meerkat_vs_decomfl}

Next, we validate our claims by evaluating performance under an extreme communication regime, specifically by setting the local step $T$ to 1. In this high-frequency setup, we compared \textsc{Meerkat} against baselines (Full-FedZO and LoRA-FedZO) under both IID and Non-IID data distributions. For the Non-IID case, we used Dirichlet distributions with $\alpha=0.5, 0.3,$ and $0.1$. Figure~\ref{fig:highfrequency_avg} presents results for $\alpha=0.5$, while results for $\alpha=0.3$ and $0.1$ are available in Appendix~\ref{sec:additional_experiment_results} figure~\ref{fig:high_freq_alpha03_01}. Specifically, Figure~\ref{fig:highfrequency_avg} reveals a remarkable finding: on the Qwen2-1.5b model, \textsc{Meerkat}'s average test accuracy over seven tasks under Non-IID data matches that under IID. Beyond this exact match, results show that at a local step of $T=1$, \textsc{Meerkat} effectively bridges the performance gap between IID and Non-IID settings, achieving nearly comparable test accuracy across both data distributions, and consistently outperforms baselines. We further investigate the performance of \textsc{Meerkat} under high-frequency communication with varying sparsity ratios, as shown in Table~\ref{tab:outlier_percentage} in Appendix~\ref{sec:additional_experiment_results}. Notably, even at extreme sparsity levels such as $10^{-3}$ and $10^{-4}$, \textsc{Meerkat} maintains strong performance. This significantly reduces the memory demand of local computation resources, making it highly suitable for resource-constrained federated settings. Our experimental findings collectively validate Claim 1, demonstrating \textsc{Meerkat}'s superiority over baselines under equivalent communication budgets. The combined use of high-frequency communication and sparsity mechanisms proves effective in mitigating challenges posed by Non-IID data distributions. We also explored sensitive parameter selection using downstream task data. Since performance remained comparable under identical communication frequencies and sparsity levels, we prioritized pre-training data to better preserve client privacy. See Appendix~\ref{sec:additional_experiment_results}, Tables~\ref{tab:c4_vs_taskdata}, \ref{tab:c4_vs_taskdata_IId}, and \ref{tab:meerkat_vs_taskdata_multistep} for details.

\begin{figure*}[!h]
    \centering
    \includegraphics[width=0.9\textwidth]{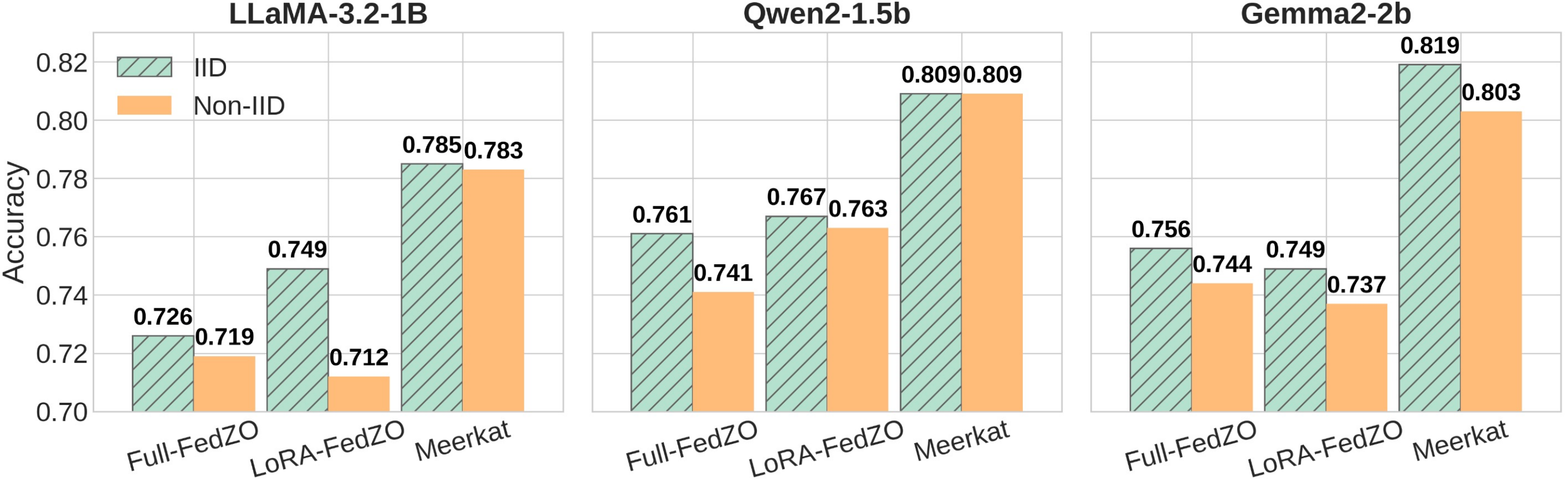}
    \caption{This figure compares three methods—Full-FedZO, LoRA-FedZO, and \textsc{Meerkat}—on three LLMs: LLaMA-3.2-1B, Qwen2-1.5b, and Gemma2-2b. The x-axis shows the different methods, and each method has two bars indicating performance under IID and Non-IID settings. The Non-IID results are obtained under a Dirichlet distribution with $\alpha = 0.5$ .The y-axis represents the average test accuracy across multiple downstream tasks—SST2, AgNews, Yelp, BoolQ, RTE, WSC, and WiC. All detailed results for these tasks are provided in Appendix~\ref{sec:additional_experiment_results}, Table~\ref{tab:outlier_percentage}. }
    \label{fig:highfrequency_avg}
    \vspace{-4pt}
\end{figure*}

\subsection{Answer to RQ2: GradIP Trajectories as Effective Indicators of Data Heterogeneity}
\label{subsec:rq2_results}
\begin{wrapfigure}[17]{R}{0.4\textwidth}
  \vspace{-0pt}
  \centering
  \includegraphics[width=0.4\textwidth]{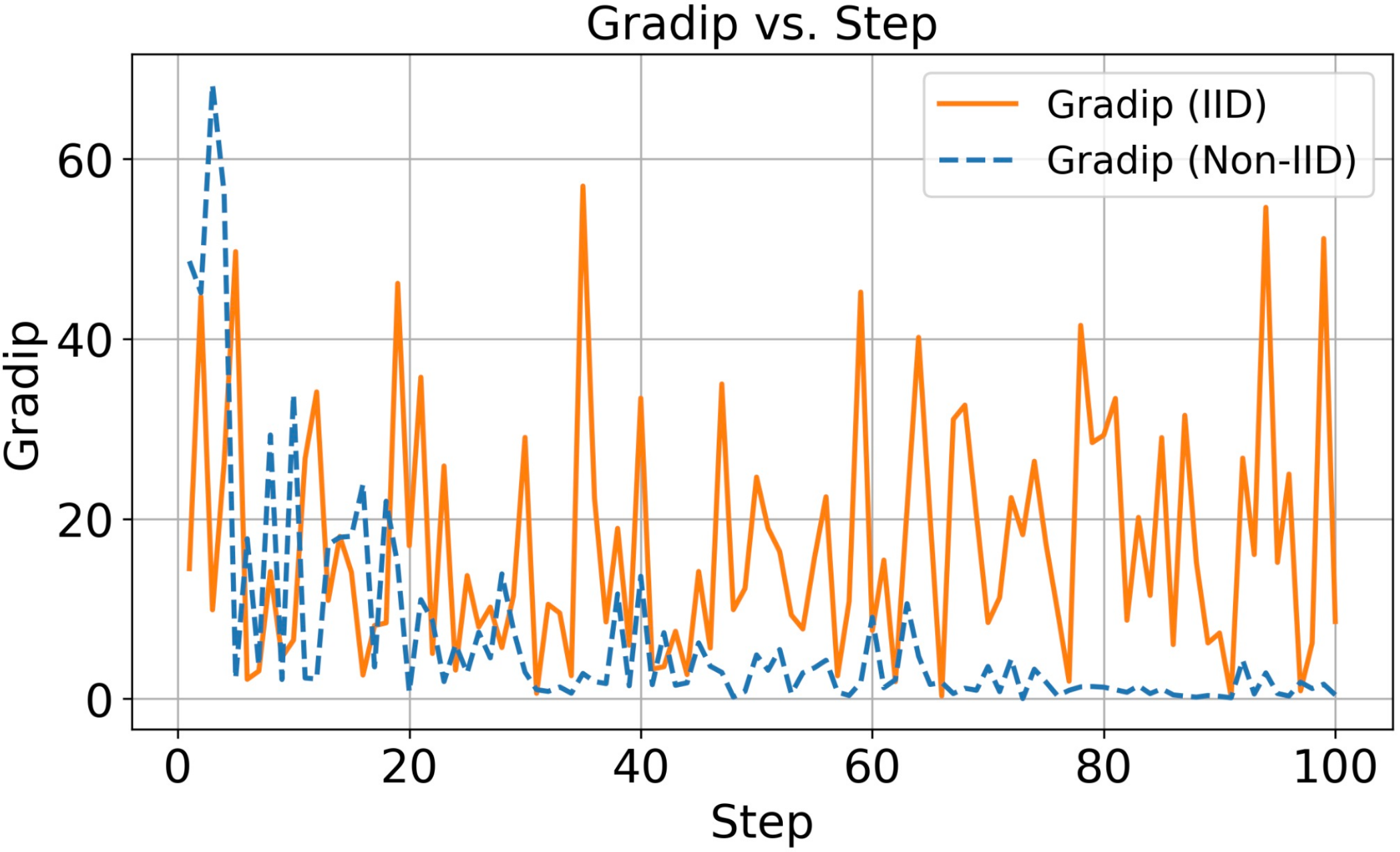}
  \caption{Under a density ratio of $5\times10^{-3}$, we track the GradIP (see Definition~\ref{def:gradip}) over 100 local training steps on the SST-2 dataset using LLaMA-3.2-1B model, comparing a client with IID data to a client with Non-IID data.}
  \label{fig:inner_product_phenomenon}
  \vspace{-0pt}
\end{wrapfigure}
This section experimentally validates Claim 2 (Section~\ref{subsec:gradip_phenomenon}), investigating GradIP trajectories as indicators of data heterogeneity. Based on our theoretical analysis assuming single-label Non-IID data (Section~\ref{sec:empirical_analysis_of_gradip_phenomenon_proof}), we study the dynamics of gradient-related metrics during local training.
We first compare two extremes: IID clients vs. clients with single-label (extreme Non-IID) data. We track three metrics: GradIP score, local gradient norm, and cosine value between the local and pre-training gradients. As shown in Figures~\ref{fig:inner_product_phenomenon} and \ref{fig:gradip_main_fig}, GradIP for extreme Non-IID clients steadily decays to zero over 100 steps, while for IID clients it fluctuates persistently. To understand this, we analyze its components: Figure~\ref{fig:general_phenomenon_of_gradient}(a) shows cosine value stays near zero (i.e., gradients are nearly orthogonal) for both settings, suggesting the gradient norm is the key factor. Indeed, Figure~\ref{fig:general_phenomenon_of_gradient}(b) shows that the gradient norm mirrors GradIP’s behavior across the two settings. Moreover, in later stages, GradIP declines more sharply for Non-IID clients than for IID ones, making this stage-wise mean difference an additional criterion for identifying Non-IID clients. We further extend our analysis to more general Non-IID scenarios (Figure~\ref{fig:gradip_real_noniid_agnews_boolq}, Figure~\ref{fig:model_comparison_gradip_ratiocompare_boolq}, Figure~\ref{fig:model_comparison_gradip_ratiocompare_agnews}), where GradIP exhibits similar dynamics that correlate with the degree of heterogeneity. 

\subsection{Answer to RQ3: VPCS Early Stopping Clients with Extremely Non-IID Data}

This section experimentally validates Claim 3 (Section~\ref{subsec:vpcsclaim}). As established in Section~\ref{subsec:rq2_results}, GradIP trajectories provide an effective signal for identifying clients with extremely Non-IID data, exhibiting distinct behaviors. Leveraging this signal, VPCS detects extremely Non-IID clients during a calibration phase and applies early stopping, limiting them to one local training step per communication round (Algorithm~\ref{algo:Meerkat_vp}). To validate the effectiveness of this VPCS strategy in improving performance, we compared \textsc{Meerkat-vp} with \textsc{Meerkat} and Random Client Selection, which randomly selects the same number of clients for early stopping as VPCS, under Non-IID data distributions dirichlet $\alpha=0.5$ and the same communication frequencies. Crucially, for the same model, dataset, and communication frequency, the three methods employed the same sparsity level. Figure~\ref{fig:meerkatvp_performance} illustrates the average test accuracy across multiple downstream tasks for \textsc{Meerkat-vp} compared to \textsc{Meerkat} and \textsc{Random Client Selection}. Detailed results for individual tasks are presented in Appendix~\ref{sec:additional_experiment_results} Table~\ref{tab:comparison_sparse_vs_virtual_path}. As shown in Figure~\ref{fig:meerkatvp_performance}, \textsc{Meerkat-vp} consistently outperforms both \textsc{Meerkat} and \textsc{Random Client Selection} in different communication frequencies. These experimental results strongly validate Claim 3, confirming that VPCS effectively leverages GradIP analysis to manage extremely Non-IID clients, leading to improved performance for ZO federated LLM fine-tuning.
\begin{figure}[!h]
    \centering
    \includegraphics[width=0.9\textwidth]{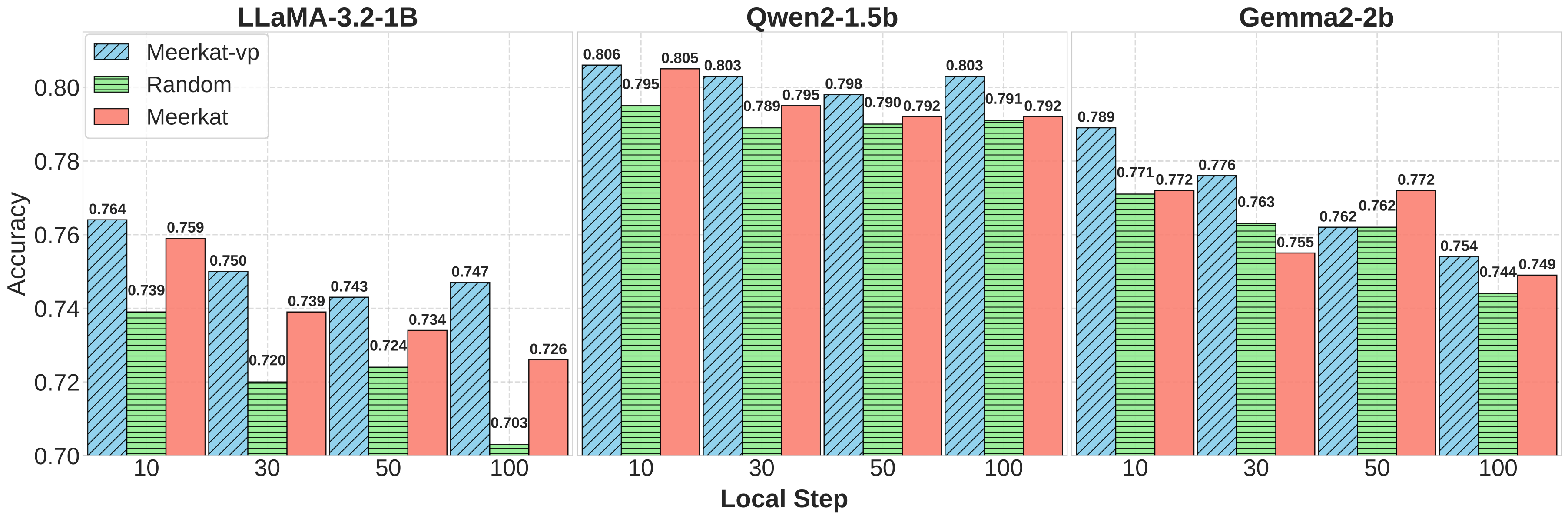}
    \caption{This figure compares two methods—\textsc{Meerkat-vp}, \textsc{Meerkat} and Random Client Selection—across three LLMs: LLaMA-3.2-1B, Qwen2-1.5b, and Gemma2-2b. The x-axis shows the local step values ($10$, $30$, $50$, $100$), while the y-axis indicates the average test accuracy over multiple downstream tasks—SST-2, AgNews, Yelp, BoolQ, RTE, WSC, and WiC—in a Non-IID setting. All detailed results for these tasks are presented in Appendix~\ref{sec:additional_experiment_results} Table~\ref{tab:comparison_sparse_vs_virtual_path}.}
    \label{fig:meerkatvp_performance}
    \vspace{-8pt}
\end{figure}

\section{Related Work}

Our research leverages advances in ZO federated optimization, sparsity techniques for LLMs, and communication frequency adjustments strategies for addressing data heterogeneity. ZO methods significantly reduce computational and communication overhead. Integrating sparsity into LLM fine-tuning amplifies these benefits, substantially decreasing resource demands during training and inference. Concurrently, communication frequency adjustments mitigate performance degradation induced by Non-IID data, emphasizing a crucial trade-off between communication budget and global model performance. A detailed discussion is provided in Appendix~\ref{app:detailed_related_work}.

\section{Conclusion}\label{sec:conclusion}
In this paper, we introduce \textsc{Meerkat}, a sparse zeroth-order federated fine-tuning methodology. Experiments show \textsc{Meerkat} outperforms Full-FedZO and other sparsity methods on most tasks at equivalent communication frequencies. \textsc{Meerkat}'s efficiency enables high-frequency communication, effectively mitigating Non-IID drift. Moreover, we propose \textsc{Meerkat-vp}. This methodology utilizes VPCS, which analyzes GradIP via virtual paths to enable the selective early stopping of extreme Non-IID clients. This approach is shown to improve model performance. Our work thus offers effective methods for efficient ZO federated LLM fine-tuning under varying network conditions and data heterogeneity. Given the technical focus of this work on algorithm, there are no direct negative societal consequences inherent to it that need to be emphasized; potential negative impacts would arise from the specific applications where these methods are deployed.

\bibliographystyle{plainnat}
\bibliography{sections/ref}

\newpage
\appendix

\section*{Appendix}
In Section~\ref{app:detailed_related_work}, we present the related work relevant to this study. In Section~\ref{sec:theoretical_proof}, we present the theoretical convergence analysis of \textsc{Meerkat}, including its high-frequency communication method. Additionally, we analyze the convergence of \textsc{Meerkat-VP} and demonstrate its superior performance compared to \textsc{Meerkat}. We further prove that under extreme Non-IID settings, the gradient norm gradually vanishes during convergence, whereas in IID settings, it tends to oscillate. In Section~\ref{sec:experiment_details}, we provide details on experimental hyperparameters and report supplementary results.
\section{Review of Related Works}\label{app:detailed_related_work}
\textbf{Federated Zeroth-Order Optimization.}
Zeroth-order optimization~\cite{malladi2023fine,zhang2024revisiting} has gained increasing attention in federated learning~\cite{fang2022communication, zhang2021desirable}, particularly for addressing challenges in training costs, privacy, and communication overhead. Fine-Grained ~\cite{chen2024fine} demonstrates how clients can reduce upload overhead by sending estimated gradients rather than full model parameters to the server, though download costs remain significant due to complete model weight transfers. DeComFL~\cite{li2024achievingdimensionfreecommunicationfederated} further advances this approach by using gradient scalars for both uploads and downloads, substantially reducing bidirectional communication costs. However, it does not address the challenges posed by data heterogeneity (Non-IID) in federated learning. The integration of AirComp wireless technology enables direct over-the-air aggregation of model updates~\cite{fang2022communication}. In black-box settings where pre-trained language model parameters are inaccessible, FedBPT~\cite{sun2023fedbptefficientfederatedblackbox} employs ZO to optimize prompt vectors, achieving efficient distributed optimization with reduced computational and communication overhead. FedMeZO~\cite{li2020convergencefedavgnoniiddata} analyzes the convergence properties of ZO for federated LLM fine-tuning.

\noindent \textbf{Sparsity in LLM.}
Current research on sparsity in LLMs explores techniques such as pruning, contextual sparsity prediction, and structured sparsity~\cite{zhang2024dynamicsparsetrainingtrainingfree, liu2023deja,liu2023Scissorhands, lu2024spp, zheng2024learn, shao2024one, wang2019structured, huang2024pruninglargelanguagemodels,zhou2024sirius,su2024defense,xu2024soft}. These methods enhance both training and inference by improving computational efficiency, reducing memory usage, and enabling deployment in resource-constrained environments.  
Sparsity has also proven particularly effective in zeroth-order (ZO) optimization~\cite{guo2024zerothorderfinetuningllmsextreme, liu2024sparsemezoparametersbetter}, especially when combined with weight quantization for fine-tuning LLMs.  
Building on this, our work investigates the role of sparsity in resource-frugal federated fine-tuning of LLMs.  

\noindent \textbf{High-Frequency Communication for Non-IID Federated Learning.} Data heterogeneity across clients is a major challenge in Federated Learning, significantly degrading performance compared to IID settings. Increasing communication frequency, by reducing local training steps per round, is explored as a strategy to mitigate this issue. Early work showed that merely reducing local steps had limited improvements in extreme non-IID scenarios \cite{https://doi.org/10.48550/arxiv.1806.00582}. Theoretical analysis later confirmed that smaller local training steps can improve convergence speed under Non-IID conditions, but at the cost of increased communication budget, highlighting a critical trade-off \cite{li2020convergencefedavgnoniiddata}. To effectively handle challenges arising from non-IID data that often necessitate higher communication, various algorithms have been proposed: SCAFFOLD \cite{karimireddy2021scaffoldstochasticcontrolledaveraging} highlights the 'client-drift' problem in FedAvg, noting it's exacerbated by increased local training steps (reduced communication frequency), and proposes using control variates to mitigate this drift, enabling improved convergence; FedDyn \cite{acar2021federatedlearningbaseddynamic} guarantees consistent convergence to the global optimum even with a larger number of local training steps (lower communication frequency). This overcomes the limitation of traditional methods where high communication frequency is needed to compensate for local-global optimum inconsistency. Empirical studies further demonstrate that performance is highly sensitive to the number of local training steps under different non-IID distributions, and the optimal communication frequency depends on the specific data heterogeneity \cite{li2021federatedlearningnoniiddata}. These works underscore the complex interplay between data heterogeneity, local computation, and communication frequency. This complexity motivates the development of algorithmic solutions to improve efficiency and robustness in FL under Non-IID settings.

\section{Theoretical and Algorithm Analysis}\label{sec:theoretical_proof}

\subsection{Notations and Definitions}

In this subsection, we formally define the assumptions, notations and concepts used in the convergence analysis of \textsc{Meerkat} and \textsc{Meerkat-vp}. Table~\ref{tab:notations_meerkat} summarizes the key symbols.

\begin{table}[H]
\centering
\caption{Notations used in our theoretical analysis.}
\label{tab:notations_meerkat}
\begin{tabular}{|c|l|}
\hline
\textbf{Notation} & \textbf{Meaning} \\
\hline
$\mathbf{w}$ & global model parameter \\
\hline
$K$ & total number of clients in the federated system \\
\hline
$p_k$ & probability or weight assigned to client $k$ \\
\hline
$f_k(\mathbf{w})$ & total loss computed over all data samples of the client k. \\
\hline
$f(\mathbf{w})$ & global loss function evaluated by the global model over all data \\
\hline
$T$ & number of local update steps per communication round \\
\hline
$r$ & communication round \\
\hline
$t$ & local update time step \\
\hline
$\eta$ & local learning rate \\
\hline
$\epsilon$ & perturbation magnitude in ZO estimation \\
\hline
$\mathbf{z}_k^t$ & standard Gaussian vector for client $k$ at local step $t$ from $ \mathcal{N}(0, I_d)$ \\
\hline
$\mathbf{m}$ & binary sparse mask vector ($\mathbf{m} \in \{0,1\}^d$) \\
\hline
$d$ & model dimension \\
\hline
$R$ & federated learning training round \\
\hline
$u$ & sparsity ratio \\
\hline
$c$ & gradient coverage \\
\hline
$g_k^t$ & projected gradient estimate for client $k$ at local step $t$ \\
\hline
$\hat{\nabla} f_k^t$ & zeroth-order gradient of client $k$ at local step $t$ \\
\hline
$L$ & Lipschitz smoothness (Assumption 1) \\
\hline
$\mu$ & PL inequality (Assumption 2) \\
\hline
$f^*$ & minimal global loss achieved by optimizing the global model \\
\hline
$f_k^*$ & minimal client loss achieved by optimizing the local model on client $k$ \\
\hline
$c_h$ and $\sigma_h^2$ & heterogeneity-induced variance (Assumption 4) \\
\hline
$\|\cdot\|_{\text{op}}$ & operator norm of a matrix \\
\hline
$\sigma^2$ & variance of the sparse ZO gradient estimator(Assumption 6) \\
\hline
$\gamma$ & The clients with balanced data distributions contribute to the global model during training. \\
\hline
\end{tabular}
\end{table}

\subsection{Assumptions}

We introduce the assumptions used in the convergence analysis of \textsc{Meerkat} and \textsc{Meerkat-vp}.

\begin{assumption}[ 
 Lipschitz smoothness]\label{assumption:assumption1}
We assume that each client \(k\)'s local objective function \(f_k(\mathbf{w})\) is differentiable and has \(L\)-Lipschitz continuous gradients:
\[
\|\nabla f_k(\mathbf{w}_1) - \nabla f_k(\mathbf{w}_2)\| \leq L \|\mathbf{w}_1 - \mathbf{w}_2\|, \quad \forall \mathbf{w}_1, \mathbf{w}_2 \in \mathbb{R}^d.
\]
Consequently, the global loss \(f(\mathbf{w}) = \sum_{k=1}^K p_k f_k(\mathbf{w})\) is also \(L\)-smooth.
\end{assumption}

\begin{assumption}[PL inequality]\label{assumption:assumption2}
We assume that \(f(\mathbf{w})\) satisfies the Polyak-\L{}ojasiewicz (PL) condition:
\[
f(\mathbf{w}) - f^* \leq \frac{1}{2\mu} \|\nabla f(\mathbf{w})\|^2, \quad \forall \mathbf{w} \in \mathbb{R}^d,
\]
\(\mu > 0\) is the PL constant. This condition holds for a broad class of non-convex objectives and is commonly used in analyzing convergence of gradient-based and zeroth-order methods.
\end{assumption}

\begin{assumption}[Global–Local Disparities in Non-i.i.d.\ Setting]
\label{assumption:assumption3}
For any $\theta\in\mathbb{R}^d$, the discrepancy between the local and global gradient is bounded by
\[
\bigl\|\nabla f(\theta)-\nabla f_i(\theta)\bigr\|^2
\;\le\;
c_h\,\bigl\|\nabla f(\theta)\bigr\|^2
\;+\;
\sigma_h^{2},
\]
where $c_h>0$ and $\sigma_h^{2}\ge 0$ are constants, and $\theta$ is the global model parameter broadcast to all clients at the start of each round. We further assume \(c_h\in(0,1)\).  In particular,
\begin{itemize}
  \item A smaller \(c_h\) corresponds to lower \emph{data heterogeneity}: local gradient deviations from the global gradient are small, indicating that client data distributions are nearly i.i.d.
  \item A larger \(c_h\) signals stronger non-i.i.d\ data distribution effects, with greater variation between each client’s gradient and the global gradient.
\end{itemize}
\end{assumption}

\begin{assumption}[Bounded stochastic gradient variance] \label{assumption:assumption4}
For any sample \((x,y)\sim\mathcal{D}\) and any \(\mathbf{w}\in\mathbb{R}^d\), denote \(f(\mathbf{w};(x,y))\) as the loss on that single data point, and let
\(\bar f(\mathbf{w}):=\mathbb{E}_{(x,y)\sim\mathcal{D}}\bigl[f(\mathbf{w};(x,y))\bigr]\)
be the average full-batch loss.  We assume

\[
\bigl\|\nabla f(\mathbf{w};(\mathbf{x},y)) - \nabla \bar f(\mathbf{w})\bigr\|_2^2
\;\le\;
\sigma^2.
\]
\end{assumption}

\begin{assumption}[Local–Global Optimality Gap]
\label{assumption:local_global_gap}
For each client \(k\), define the local–global optimality gap as
\[
\Delta_k \;=\; \bigl\lVert \mathbf{w}_k^{*} - \mathbf{w}^{*} \bigr\rVert_{2}^{2},
\]
where \(\mathbf{w}_k^{*}\) is the local optimal model on client \(k\) and \(\mathbf{w}^{*}\) is the global optimal model.
\end{assumption}

\begin{assumption}[Sensitive parameters are sparse]\label{assumption:assumption5}
At each local step $t$ (and for every client $k$), there exists a binary mask $\mathbf{m}\in\{0,1\}^{d}$ with exactly $u$ non‑zero entries and a constant $c\in[0,1]$ such that
\[
\left\|\mathbf{m}\odot\nabla f_k\!\left(\mathbf{w}_k^{t};(\mathbf{x}_t,\mathbf{y}_t)\right)\right\|^{2}
   = c\,\left\|\nabla f_k\!\left(\mathbf{w}_k^{t};(\mathbf{x}_t,\mathbf{y}_t)\right)\right\|^{2}.
\]
We further assume $c\gg\frac{u}{d}$, meaning this small subset of “sensitive” parameters captures a disproportionately large fraction of the gradient norm.
\end{assumption}

\noindent These assumptions are standard and foundational in optimization and FL literature\cite{bottou2018optimizationmethodslargescalemachine, li2020federatedoptimizationheterogeneousnetworks, li2020convergencefedavgnoniiddata, 9521822, guo2024zerothorderfinetuningllmsextreme}

\noindent
We start by formulating the expectation of the sensitive sparse ZO surrogate gradient norm square in terms of its corresponding stochastic gradient norm square.

\begin{lemma}[Sensitive sparse ZO surrogate gradient norm square]\label{lem:A1}
\[
\mathbb{E}_{\bar z}
  \Bigl\|
       \hat\nabla f\bigl(w_t,(x_t,y_t),\bar z_t\bigr)
  \Bigr\|^{2}
  \;=\;
  (2+u)c\,
  \Bigl\|
       \nabla f\bigl(w_t;(x_t,y_t)\bigr)
  \Bigr\|^{2}.
\]
\end{lemma}

\begin{proof}
Our masked perturbation $\bar{\mathbf z}$ is sampled as
\(
\bar{\mathbf z}\sim
\mathcal N\!\bigl(
        0,\,
        \tilde I_{d,\mathbf m}
      \bigr),
\)
where
$\tilde I_{d,\mathbf m}$ equals the identity matrix $I_d$ with its
main diagonal masked by $\mathbf m$.

\medskip
\noindent
We expand the sensitive sparse ZO surrogate--gradient covariance matrix:

\begin{equation*}
\begin{array}{l@{}}
\mathbb{E}_{\bar{z}}\hat\nabla f(w,(x,y),\bar{z})\hat\nabla f(w,(x,y),\bar{z})^{\top} \\
= \mathbb{E}_{\bar{z}}[\bar{z}\bar{z}^{\top}((\mathbf{m} \odot \nabla f(w;(x,y)))(\mathbf{m} \odot \nabla f(w;(x,y)))^{\top})\bar{z}\bar{z}^{\top}] \\
= 2((\mathbf{m} \odot \nabla f(w;(x,y)))(\mathbf{m} \odot \nabla f(w;(x,y)))^{\top}) + \|\mathbf{m} \odot \nabla f(w;(x,y))\|^2\tilde{I}_{d,\mathbf{m}}
\end{array}
\end{equation*}

\medskip
\noindent
The above expected squared norm is obtained by summing the diagonal elements of this covariance matrix:
\begin{align*}
\mathbb{E}_{\bar z}
\!\Bigl\|
      \hat\nabla f\bigl(w_t,x_t,\bar z_t\bigr)
\Bigr\|^{2}
&=\,
  \bigl(
      \operatorname{diag}\!
      \bigl[
          \mathbb{E}_{\bar z}
          \hat\nabla f\!\bigl(w,(x,y),\bar z\bigr)
          \hat\nabla f\!\bigl(w,(x,y),\bar z\bigr)^{\!\top}
      \bigr]
  \bigr)^{\!2}                                            \\[4pt]
&=\,
  2c\,
  \bigl\|
      \nabla f\bigl(w_t;(x_t,y_t)\bigr)
  \bigr\|^{2}
  + uc\,
  \bigl\|
      \nabla f\bigl(w_t;(x_t,y_t)\bigr)
  \bigr\|^{2}                                             \\[4pt]
&=\,
  (2+u)c\,
  \bigl\|
      \nabla f\bigl(w_t;(x_t,y_t)\bigr)
  \bigr\|^{2}.
\end{align*}

\end{proof}

\begin{lemma}[Unbiasedness of Masked Sparse ZO Surrogate Gradient]\label{lem:A2}
\begin{equation}\label{eq:masked_unbiased}
\mathbb{E}_{\bar z}\!\bigl[\hat\nabla f_k(\mathbf w_k^t,\bar z)\bigr]
= \mathbf m\!\odot\!\nabla f_k(\mathbf w_k^t),
\quad\text{where } \bar z = z\!\odot\!\mathbf m.
\end{equation}
\end{lemma}

\begin{proof}
First, consider the estimator defined as:
\[
\hat{\nabla} f_k(\mathbf{w}_k^t, z) = \frac{f_k(\mathbf{w}_k^t + \epsilon (z \odot \mathbf{m})) - f_k(\mathbf{w}_k^t - \epsilon (z \odot \mathbf{m}))}{2\epsilon} \cdot (z \odot \mathbf{m}).
\]

To proceed, we apply a first-order Taylor expansion of \(f_k\) around \(\mathbf{w}_k^t\) for small \(\epsilon\):
\[
f_k(\mathbf{w}_k^t \pm \epsilon (z \odot \mathbf{m})) = f_k(\mathbf{w}_k^t) \pm \epsilon \langle \nabla f_k(\mathbf{w}_k^t), z \odot \mathbf{m} \rangle + \mathcal{O}(\epsilon^2).
\]

Substitute these expansions into the numerator of the estimator:
\[
\begin{aligned}
&f_k\bigl(\mathbf{w}_k^t + \epsilon\,(z \odot \mathbf{m})\bigr)
 - f_k\bigl(\mathbf{w}_k^t - \epsilon\,(z \odot \mathbf{m})\bigr) \\[4pt]
&\quad= \bigl[f_k(\mathbf{w}_k^t)
       + \epsilon\,\langle \nabla f_k(\mathbf{w}_k^t),\,z \odot \mathbf{m}\rangle\bigr] \\[4pt]
&\qquad- \bigl[f_k(\mathbf{w}_k^t)
       - \epsilon\,\langle \nabla f_k(\mathbf{w}_k^t),\,z \odot \mathbf{m}\rangle\bigr]
 + \mathcal{O}(\epsilon^2).
\end{aligned}
\]

Simplify the expression:
\[
f_k(\mathbf{w}_k^t + \epsilon (z \odot \mathbf{m})) - f_k(\mathbf{w}_k^t - \epsilon (z \odot \mathbf{m})) = 2\epsilon \langle \nabla f_k(\mathbf{w}_k^t), z \odot \mathbf{m} \rangle + \mathcal{O}(\epsilon^2).
\]

Thus, the estimator becomes:
\[
\hat{\nabla} f_k(\mathbf{w}_k^t, z) = \frac{2\epsilon \langle \nabla f_k(\mathbf{w}_k^t), z \odot \mathbf{m} \rangle + \mathcal{O}(\epsilon^2)}{2\epsilon} \cdot (z \odot \mathbf{m}) = \left[ \langle \nabla f_k(\mathbf{w}_k^t), z \odot \mathbf{m} \rangle + \mathcal{O}(\epsilon) \right] (z \odot \mathbf{m}).
\]

As \(\epsilon \to 0\), the \(\mathcal{O}(\epsilon)\) term disappears, yielding the approximation:
\[
\hat{\nabla} f_k(\mathbf{w}_k^t, z) \approx \langle \nabla f_k(\mathbf{w}_k^t), z \odot \mathbf{m} \rangle \cdot (z \odot \mathbf{m}).
\]

Next, compute the expectation \(\mathbb{E}_{z} \left[ \hat{\nabla} f_k(\mathbf{w}_k^t, z) \right]\). Since the estimator is a vector, consider its \(j\)-th component:
\[
\left[ \hat{\nabla} f_k(\mathbf{w}_k^t, z) \right]_j \approx \langle \nabla f_k(\mathbf{w}_k^t), z \odot \mathbf{m} \rangle \cdot (z_j m_j).
\]

Express the inner product explicitly:
\[
\langle \nabla f_k(\mathbf{w}_k^t), z \odot \mathbf{m} \rangle = \sum_{i=1}^d (\nabla f_k(\mathbf{w}_k^t))_i z_i m_i.
\]

Thus, the \(j\)-th component is:
\[
\left[ \hat{\nabla} f_k(\mathbf{w}_k^t, z) \right]_j \approx \left( \sum_{i=1}^d (\nabla f_k(\mathbf{w}_k^t))_i z_i m_i \right) z_j m_j.
\]

Now, take the expectation over \(z \sim \mathcal{N}(0, \mathbf{I}_d)\), where \(z_i\) are independent standard normal variables:
\[
\mathbb{E}_{z} \left[ \left( \sum_{i=1}^d (\nabla f_k(\mathbf{w}_k^t))_i z_i m_i \right) z_j m_j \right] = \sum_{i=1}^d (\nabla f_k(\mathbf{w}_k^t))_i m_i m_j \mathbb{E}[z_i z_j].
\]

Since \(\mathbb{E}[z_i z_j] = \delta_{ij}\) (1 if \(i = j\), 0 otherwise), the sum reduces to:
\[
(\nabla f_k(\mathbf{w}_k^t))_j m_j^2 \mathbb{E}[z_j^2].
\]

Given \(m_j^2 = m_j\) (as \(m_j = 0\) or \(1\)) and \(\mathbb{E}[z_j^2] = 1\), this becomes:
\[
(\nabla f_k(\mathbf{w}_k^t))_j m_j.
\]

Thus, for each component \(j\):
\[
\mathbb{E}_{z} \left[ \left[ \hat{\nabla} f_k(\mathbf{w}_k^t, z) \right]_j \right] \approx m_j (\nabla f_k(\mathbf{w}_k^t))_j.
\]

This implies:
\[
\mathbb{E}_{z} \left[ \hat{\nabla} f_k(\mathbf{w}_k^t, z) \right] \approx \mathbf{m} \odot \nabla f_k(\mathbf{w}_k^t).
\]

Finally, as \(\epsilon \to 0\), the higher-order terms in the Taylor expansion vanish, making the approximation exact:
\[
\mathbb{E}_{\bar z} \left[ \hat{\nabla} f_k(\mathbf{w}_k^t, \bar z) \right] = \mathbf{m} \odot \nabla f_k(\mathbf{w}_k^t).
\]
\end{proof}

\subsection{\textsc{Meerkat} convergence analysis}\label{sec:meerkat_theory_convergence}

We consider the \textbf{federated zeroth-order optimization problem}, where the objective is to minimize the global loss function\cite{ling2024convergence}:

\begin{align*}
\min_{\mathbf{w}} f(\mathbf{w}) &= \sum_{k=1}^{K} p_k f_k(\mathbf{w})
\end{align*}

Each client performs $T$ local steps:

\begin{align*}
\mathbf{w}^{t+1}_k = \mathbf{w}^{t}_k - \eta\nabla f_k^t(\mathbf{w}), \quad t = 0, 1, \dots, T-1
\end{align*}

starting from the global model $\mathbf{w}^0_k = \mathbf{w}^r$. After clients finish local updates, the server performs weighted aggregation of their model updates.

\begin{align*}
\mathbf{w}^{r+1} = \sum_{k=1}^{K} p_k \mathbf{w}_k^r.
\end{align*}


\begin{theorem}\label{theory:zo-client-convergence}[Client Local ZO Update Convergence]
Let $f_k$ be $L$-smooth and $\hat{\nabla} f_k^t$ be an unbiased sparse zeroth-order gradient estimator with variance bounded by $\sigma^2$. Then we have

If we set constant learning rate $\eta = \frac{1}{L(u+2)}$ and $T$ local steps, the output of client $k$ satisfies:
\begin{equation}
\frac{1}{T} \sum_{t=0}^{T-1} \mathbb{E} \| \nabla f_k(\mathbf{w}_k^t) \|^2 
\leq \mathcal{O} \left( \frac{1}{T} \right) + \mathcal{O}(\sigma^2).
\label{eq:client-zo-convergence}
\end{equation}
\end{theorem}

\begin{proof}
We start by proving Theorem~\ref{theory:zo-client-convergence} euqation~\ref{eq:client-zo-convergence} that each client achieves local convergence during training with sparse zeroth-order finetuning. Next, we demonstrate that server-side aggregation also converge. Finally, by leveraging the PL inequality, we prove that \textsc{Meerkat} exhibits linear convergence to global minimum.

\noindent \textbf{Part 1: Client Local ZO Update Convergence}

\noindent We analyze the effect of one local step of \textsc{Meerkat} under sparse zeroth-order updates. Let client $k$ perform the local update:
\[
\mathbf{w}_k^{t+1} = \mathbf{w}_k^t - \eta \hat{\nabla} f_k^t,
\]
where the stochastic sparse zeroth-order gradient estimator is defined as:
\[
g_k^t = \frac{f_k(\mathbf{w}_k^t + \epsilon (\mathbf{z}_k^t \odot \mathbf{m})) - f_k(\mathbf{w}_k^t - \epsilon (\mathbf{z}_k^t \odot \mathbf{m}))}{2\epsilon}.
\]

\[
\hat{\nabla} f_k^t = g_k^t \cdot (\mathbf{z}_k^t \odot \mathbf{m})
\]

\paragraph{Descent via Lipschitz smoothness.}
Since $f_k(\mathbf{w})$ is Lipschitz smoothness:
\[
f_k(\mathbf{w}_k^{t+1}) \leq f_k(\mathbf{w}_k^t) + \langle \nabla f_k(\mathbf{w}_k^t), \mathbf{w}_k^{t+1} - \mathbf{w}_k^t \rangle + \frac{L}{2} \|\mathbf{w}_k^{t+1} - \mathbf{w}_k^t\|^2.
\]
Substituting the update $\mathbf{w}_k^{t+1} - \mathbf{w}_k^t = -\eta \hat{\nabla} f_k^t$, we obtain:
\[
f_k(\mathbf{w}_k^{t+1}) \leq f_k(\mathbf{w}_k^t) - \eta \langle \nabla f_k(\mathbf{w}_k^t), \hat{\nabla} f_k^t(\mathbf{w},\bar{\mathbf{z}}_t)\ \rangle + \frac{L \eta^2}{2} \|\hat{\nabla} f_k^t(\mathbf{w},\bar{\mathbf{z}}_t)\|^2.
\]

\noindent Taking expectation, we have:
\[
\mathbb{E}_{\bar{\mathbf{z}}}[f_k(\mathbf{w}_k^{t+1})]
\leq \mathbb{E}_{\bar{\mathbf{z}}}[f_k(\mathbf{w}_k^t)] - \eta\mathbb{E}_{\bar{\mathbf{z}}} \|\mathbf{m} \odot \nabla f_k(\mathbf{w}_k^t)\|^2 + \frac{L \eta^2}{2} \mathbb{E}_{\bar{\mathbf{z}}} \|\hat\nabla f_k(\mathbf{w}_k^t,\bar{\mathbf{z}}_t)\|^2.
\]

\[
\mathbb{E}_{\bar{\mathbf{z}}}[f_k(\mathbf{w}_k^{t+1})]
\leq \mathbb{E}_{\bar{\mathbf{z}}}[f_k(\mathbf{w}_k^t)] - c \eta\mathbb{E}_{\bar{\mathbf{z}}} \|\nabla f_k(\mathbf{w}_k^t)\|^2 + \frac{L \eta^2}{2} (2+u)c\mathbb{E}_{\bar{\mathbf{z}}}\bigl\|
      \nabla f_k\bigl(\mathbf{w}_k^t\bigr)
  \bigr\|^{2}.
\]

\[
\mathbb{E}_{\bar{\mathbf{z}}} f_k(\mathbf{w}_k^{t+1}) 
\leq \mathbb{E}_{\bar{\mathbf{z}}} f_k(\mathbf{w}_k^t) - \left( c \eta_t - \frac{L \eta_t^2}{2} c(u + 2) \right) \| \nabla_{\mathbf{w}} f_k(\mathbf{w}_k^t) \|^2 + \frac{L \eta_t^2}{2} c(u + 2) \sigma^2.
\]

Denote $\alpha = L c(u + 2)$, we can rewrite as:

\[
\mathbb{E}_{\bar{\mathbf{z}}} f_k(\mathbf{w}_k^{t+1}) 
\leq \mathbb{E}_{\bar{\mathbf{z}}} \left\{ f_k(\mathbf{w}_k^t) - \eta_t \left( c - \frac{\alpha}{2} \eta_t \right) \| \nabla_{\mathbf{w}} f_k(\mathbf{w}_k^t) \|^2 \right\} + \frac{\alpha}{2} \sigma^2 \eta_t^2.
\]

From the above inequality, we get $\eta < \frac{2c}{\alpha}$. Suppose we use a constant learning rate $\eta_t = \eta = \frac{c}{\alpha} = \frac{1}{L(u + 2)}$, we get:

\begin{equation}
\mathbb{E}_{\bar{\mathbf{z}}} f_k(\mathbf{w}_k^{t+1}) 
\leq \mathbb{E}_{\bar{\mathbf{z}}} \left\{ f_k(\mathbf{w}_k^t) - \frac{c \eta}{2} \| \nabla_{\mathbf{w}} f_k(\mathbf{w}_k^t) \|^2 \right\} + \frac{\alpha}{2} \sigma^2 \eta^2.
\label{eq:local-update-ineq}
\end{equation}

\paragraph{Accumulating over $T$ steps.}
Summing equation~\ref{eq:local-update-ineq} over $t=0$ to $T-1$, we get:

\begin{equation}
\begin{aligned}
\frac{1}{T} \sum_{t=0}^{T-1} \mathbb{E}_{\bar{\mathbf{z}}} \| \nabla f_k(\mathbf{w}_k^t) \|^2
&\leq \frac{2}{c \eta T} (f_k(\mathbf{w}_k^0) - f_k^*) + \frac{1}{T} \sum_{t=0}^{T-1} \frac{\alpha}{2 c \eta} \sigma^2 \eta^2 \\
&= \frac{2 L(u + 2)}{c T} (f_k(\mathbf{w}_k^0) - f_k^*) + \sigma^2 \\
&= \mathcal{O} \left( \frac{u}{T} (f_k(\mathbf{w}_k^0) - f_k^*) \right) + \mathcal{O}(1).
\end{aligned}
\label{eq:client_grad_convergence_rate}
\end{equation}
\end{proof}

\subsection{\textsc{Meerkat} Convergence Analysis}\label{sec:meerkat convergence analysis}
We now proceed to analyze the convergence of the global model in our federated learning framework. Having established the convergence properties of local client updates, we demonstrate how these results extend to guarantee the convergence of the server-aggregated global model.

\begin{proof}
We approach this proof systematically by analyzing how the local convergence properties of clients extend to the global model through the aggregation process.

\noindent \textbf{Global Model Update Representation.}
First, the global model update can be represented as:
\begin{align*}
\mathbf{w}^{r+1} - \mathbf{w}^r = \sum_{k=1}^K p_k (\mathbf{w}_k^T - \mathbf{w}^r)
\end{align*}
where each client $k$ starts from the global model $\mathbf{w}^r$ and performs $T$ local updates to reach $\mathbf{w}_k^T$.

\noindent \textbf{Client Local Update Accumulation}
For any client $k$, the accumulated local updates can be expressed as:
\begin{align*}
\mathbf{w}_k^{r,T} - \mathbf{w}^r = -\eta\sum_{t=0}^{T-1} \hat{\nabla}f_k^t
\end{align*}

\noindent \textbf{Global Loss Descent Analysis}
By the $L$-smoothness property (Assumption~\ref{assumption:assumption1}), we have:
\begin{equation}\label{eq:smoothness_descent}
f(\mathbf{w}^{r+1}) \leq f(\mathbf{w}^r) + \langle\nabla f(\mathbf{w}^r), \mathbf{w}^{r+1} - \mathbf{w}^r\rangle + \frac{L}{2}\|\mathbf{w}^{r+1} - \mathbf{w}^r\|^2
\end{equation}

\noindent For the inner product we can get:
\begin{align*}
\langle\nabla f(\mathbf{w}^r), \mathbf{w}^{r+1} - \mathbf{w}^r\rangle &= \sum_{k=1}^K p_k \langle\nabla f(\mathbf{w}^r), \mathbf{w}_k^{r,T} - \mathbf{w}^r\rangle
\end{align*}

\noindent Accoding to the client local update process, we have:

\begin{align*}
\sum_{k=1}^K p_k \langle\nabla f_k(\mathbf{w}^r), \mathbf{w}_k^{r,T} - \mathbf{w}^r\rangle &= -\eta\sum_{k=1}^K p_k \langle\nabla f(\mathbf{w}^r), \sum_{t=0}^{T-1} \hat{\nabla}f_k(\mathbf{w}^{r,t},\bar{\mathbf{z}}_t)\rangle \\
&= -\eta\sum_{k=1}^K p_k \sum_{t=0}^{T-1} \langle\nabla f(\mathbf{w}^r), \hat{\nabla}f_k(\mathbf{w}^{r,t},\bar{\mathbf{z}}_t)\rangle \\
\end{align*}

\noindent We assume that each client's weight is equal $p_k = 1/K$, by substituting it into the above inequality, we have:

\begin{equation}
\label{eq:client-inner-product}
\sum_{k=1}^K p_k \bigl\langle\nabla f_k(\mathbf{w}^r),\,\mathbf{w}_k^{r,T} - \mathbf{w}^r\bigr\rangle
= -\frac{\eta}{K}
  \sum_{k=1}^K
  \sum_{t=0}^{T-1}
  \bigl\langle\nabla f(\mathbf{w}^r),\,\hat{\nabla}f_k(\mathbf{w}^{r,t},\bar{\mathbf{z}}_t)\bigr\rangle.
\end{equation}

\noindent Based on the equation~\ref{eq:client-inner-product} and $\hat{\nabla} f_k^t$ is unbiased, we have: 

\[
\sum_{k=1}^K \sum_{t=0}^{T-1}
\bigl\langle \nabla f(w^r), \,\hat\nabla f_k(\mathbf{w}^{r,t},\bar{\mathbf z}_t) \bigr\rangle
\;=\;
\sum_{k=1}^K \sum_{t=0}^{T-1}
\bigl\langle \nabla f(w^r), \,\mathbb{E}_{\bar z}\!\bigl[\hat\nabla f_k(\mathbf{w}^{r,t},\bar{\mathbf z}_t)\bigr]\bigr\rangle
\]
\noindent We substitute the equation~\ref{eq:masked_unbiased} and get:
\[
\sum_{k=1}^K \sum_{t=0}^{T-1}
\bigl\langle \nabla f(w^r), \,\mathbb{E}_{\bar z}\!\bigl[\hat\nabla f_k(\mathbf{w}^{r,t},\bar{\mathbf z}_t)\bigr]\bigr\rangle
\;=\;
\sum_{k=1}^K \sum_{t=0}^{T-1}
\bigl\langle \nabla f(w^r), \,\mathbf m\odot\nabla f_k(w^{r,t})\bigr\rangle.
\]

\noindent Under the Cauchy–Schwarz inequality, we have:

\[
\bigl\langle \nabla f(\mathbf{w}^{r}),\;\mathbf{m}\odot\nabla f_{k}(\mathbf{w}^{r,t})\bigr\rangle
\;\le\;
\|\nabla f(\mathbf{w}^{r})\|\;\|\mathbf{m}\odot\nabla f_{k}(\mathbf{w}^{r,t})\|
\]
\noindent We substitute Assumption~\ref{assumption:assumption5} get:
\[
\|\nabla f(\mathbf{w}^{r})\|\;\|\mathbf{m}\odot\nabla f_{k}(\mathbf{w}^{r,t})\|
\;=\;
\sqrt{c}\,\|\nabla f(\mathbf{w}^{r})\|\;\|\nabla f_{k}(\mathbf{w}^{r,t})\|.
\]
\noindent Thus we get:
\[
\bigl\langle \nabla f(\mathbf{w}^{r}),\;\mathbf{m}\odot\nabla f_{k}(\mathbf{w}^{r,t})\bigr\rangle
\;\le\;
\sqrt{c}\,\|\nabla f(\mathbf{w}^{r})\|\;\|\nabla f_{k}(\mathbf{w}^{r,t})\|.
\]
\noindent By the triangle inequality, we have
\[
\|\nabla f_k(\mathbf{w}^{r,t})\|
\;\le\;
\|\nabla f(\mathbf{w}^{r})\| + \|\nabla f_k(\mathbf{w}^{r,t}) - \nabla f(\mathbf{w}^{r})\|
\]
We substitute Assumption~\ref{assumption:assumption3} and use the properties of square roots we get:
\begin{align*}
\|\nabla f(\mathbf{w}^{r})\|
  &+ \|\nabla f_k(\mathbf{w}^{r,t}) - \nabla f(\mathbf{w}^{r})\| \\
  &\le \|\nabla f(\mathbf{w}^{r})\|
       + \sqrt{\,c_h\,\|\nabla f(\mathbf{w}^{r})\|^2 + \sigma_h^2\,} \\
  &\le (1 + \sqrt{c_h})\,\|\nabla f(\mathbf{w}^{r})\| + \sigma_h.
\end{align*}

\noindent Using the bound
\(\langle\nabla f(\mathbf{w}^{r}),\;m\odot\nabla f_k(\mathbf{w}^{r,t})\rangle \le \sqrt{c}\,\|\nabla f(\mathbf{w}^{r})\|\;\|\nabla f_k\|\)
from Cauchy–Schwarz and Assumption~\ref{assumption:assumption5}, and then plugging in the above,
we obtain
\begin{align*}
\langle\nabla f(\mathbf{w}^{r}),\,m\odot\nabla f_k(\mathbf{w}^{r,t})\rangle
&\le
\sqrt{c}\,\|\nabla f(\mathbf{w}^{r})\|\bigl[(1 + \sqrt{c_h})\,\|\nabla f(\mathbf{w}^{r})\| + \sigma_h\bigr]\\
&\le
\sqrt{c}\,(1 + \sqrt{c_h})\,\|\nabla f(\mathbf{w}^{r})\|^2
\;+\;
\sqrt{c}\,\sigma_h\,\|\nabla f(\mathbf{w}^{r})\|.
\end{align*}

\noindent Recall that the server update inner product is
\[
\bigl\langle \nabla f(w^r),\,w^{r+1}-w^r\bigr\rangle
\;=\;
-\,\frac{\eta}{K}\sum_{k=1}^K\sum_{t=0}^{T-1}
\bigl\langle\nabla f,\;m\odot\nabla f_k\bigr\rangle.
\]
Substituting the bound to equation~\ref{eq:client-inner-product}. We have:

\begin{align}\label{eq:global_descent_bound}
\bigl\langle \nabla f(w^r),\,w^{r+1}-w^r\bigr\rangle
\;\ge\;
-\,\eta\,T\,\sqrt{c}\,(1 + \sqrt{c_h})\,\|\nabla f(w^r)\|^2
\;-\;
\eta\,T\,\sqrt{c}\,\sigma_h\,\|\nabla f(w^r)\|.
\end{align}

\noindent Substituting this inequality to equation~\ref{eq:smoothness_descent}, we have:

\begin{align}
f\bigl(w^{r+1}\bigr)
&\le
  f\bigl(w^r\bigr)
  \;-\;\eta\,T\,\sqrt{c}\,\bigl(1 + \sqrt{c_h}\bigr)\,\bigl\|\nabla f(w^r)\bigr\|^2 \nonumber\\
&\quad
  -\,\eta\,T\,\sqrt{c}\,\sigma_h\,\bigl\|\nabla f(w^r)\bigr\|
  +\,\frac{L}{2}\,\bigl\|w^{r+1}-w^r\bigr\|^2
\label{eq:descent_step}
\end{align}

\noindent Applying Jensen's inequality, the last term of the equation~\ref{eq:descent_step} will be:
\[\|\mathbf{w}^{r+1} - \mathbf{w}^r\|^2 \leq \eta^2 \sum_{k=1}^K p_k \|\sum_{t=0}^{T-1} \hat{\nabla}f_k^{r,t}\|^2\] 
\noindent And then we apply Cauchy-Schwarz inequality, the last term of the equation~\ref{eq:descent_step} will be: 
\[\|\mathbf{w}^{r+1} - \mathbf{w}^r\|^2 \leq \eta^2 T \sum_{k=1}^K p_k \sum_{t=0}^{T-1} \|\hat{\nabla}f_k^{r,t}\|^2\]
\noindent Substitute this inequaltiy to equation~\ref{eq:descent_step} We get:

\begin{align*}
f\bigl(w^{r+1}\bigr)
&\le
  f\bigl(w^r\bigr)
  -\,\eta\,T\,\sqrt{c}\,(1 + \sqrt{c_h})\,\bigl\|\nabla f(w^r)\bigr\|^2
  -\,\eta\,T\,\sqrt{c}\,\sigma_h\,\bigl\|\nabla f(w^r)\bigr\|\nonumber\\
&\quad
  +\,\frac{L}{2}\,\eta^2\,T\,
    \sum_{k=1}^K p_k \sum_{t=0}^{T-1}
    \bigl\|\hat{\nabla}f_k^{r,t}\bigr\|^2.
\end{align*}
\noindent Taking Expectation and lemma~\ref{lem:A1}:

\begin{align*}
\mathbb{E}_{\bar z}\,f(\mathbf{w}^{r+1})
&\le
  \mathbb{E}_{\bar z}\,f(\mathbf{w}^r)
  -\;\eta\,T\,\sqrt{c}\,(1+\sqrt{c_h})\,\|\nabla f(\mathbf{w}^r)\|^2 \\[6pt]
&\quad
  -\;\eta\,T\,\sqrt{c}\,\sigma_h\,\|\nabla f(\mathbf{w}^r)\|
  +\;\frac{L\,\eta^2\,T\,(2+u)\,c}{2K}
    \sum_{k=1}^K \sum_{t=0}^{T-1}
    \|\nabla f_k(\mathbf{w}^{r,t})\|^2\,. 
\end{align*}

According to the equation~\ref{eq:client_grad_convergence_rate}, we know that the client-average squared gradient has upper bound. We substitute the equation~\ref{eq:client_grad_convergence_rate} to the above inequality last term we get:

\begin{align}
\label{eq:expected_descent}
\mathbb{E}_{\bar z}\,f(w^{r+1})
&\le
  \mathbb{E}_{\bar z}\,f(w^r)
  -\,\eta\,T\,\sqrt{c}\,(1+\sqrt{c_h})\,\|\nabla f(w^r)\|^2
  -\,\eta\,T\,\sqrt{c}\,\sigma_h\,\|\nabla f(w^r)\| \nonumber\\[4pt]
&\quad
  +\,\frac{L\,\eta^2\,T\,(2+u)\,c}{2K}
     \sum_{k=1}^K\Bigl[\frac{2\,L\,(u+2)}{c}\bigl(f_k(w^r)-f_k^*\bigr)
       +T\,\sigma^2\Bigr] \nonumber\\[6pt]
&\le
  \mathbb{E}_{\bar z}\,f(w^r)
  -\,\eta\,T\,\sqrt{c}\,(1+\sqrt{c_h})\,\|\nabla f(w^r)\|^2 -\,\eta\,T\,\sqrt{c}\,\sigma_h\,\|\nabla f(w^r)\| \nonumber\\[4pt]
&\quad
  +\,\frac{L^2\,\eta^2\,T\,(2+u)\,(u+2)}{K}
     \sum_{k=1}^K \bigl(f_k(w^r)-f_k^*\bigr)
  +\,\frac{L\,\eta^2\,T^2\,(2+u)\,c}{2}\,\sigma^2
\end{align}

\noindent \textbf{Accumulating Over $R$ Rounds.}
Summing equation~\ref{eq:expected_descent} over $r=0$ to $R-1$, we get:

\begin{equation}\label{eq:global-accumulated}
\begin{aligned}
\mathbb{E}_{\bar z}\bigl[f(w^R)\bigr] - \mathbb{E}_{\bar z}\bigl[f(w^0)\bigr]
&\le
-\,\eta\,T\,\sqrt{c}\,(1+\sqrt{c_h})
  \sum_{r=0}^{R-1}\bigl\|\nabla f(w^r)\bigr\|^2 \\[4pt]
&\quad
-\,\eta\,T\,\sqrt{c}\,\sigma_h
  \sum_{r=0}^{R-1}\bigl\|\nabla f(w^r)\bigr\| \\[4pt]
&\quad
+\,\frac{L^2\,\eta^2\,T\,(2+u)\,(u+2)}{K}
  \sum_{r=0}^{R-1}\sum_{k=1}^K\bigl(f_k(w^r)-f_k^*\bigr) \\[4pt]
&\quad
+\,\frac{L\,\eta^2\,T^2\,(2+u)\,c}{2}\,\sigma^2\,R\,. 
\end{aligned}
\end{equation}

From the accumulated global descent inequality over $R$ rounds:

\noindent First we set
\[
S \;=\;\sum_{r=0}^{R-1}\|\nabla f(w^r)\|^2.
\]

This represents the sum of squared gradient norms over \( R \) rounds. The second term in the inequality involves \( \sum_{r=0}^{R-1} \|\nabla f(w^r)\| \), and we apply the Cauchy-Schwarz inequality to it. For the sequence \( a_r = \|\nabla f(w^r)\| \) (with \( r = 0, 1, \ldots, R-1 \)), we consider it as a vector in \( \mathbb{R}^R \) along with a vector of ones:
\[
\sum_{r=0}^{R-1} \|\nabla f(w^r)\| = \sum_{r=0}^{R-1} \|\nabla f(w^r)\| \cdot 1 \leq \sqrt{\sum_{r=0}^{R-1} \|\nabla f(w^r)\|^2} \cdot \sqrt{\sum_{r=0}^{R-1} 1^2}.
\]
Since \( \sum_{r=0}^{R-1} 1^2 = R \), we obtain:
\[
\sum_{r=0}^{R-1} \|\nabla f(w^r)\| \leq \sqrt{\sum_{r=0}^{R-1} \|\nabla f(w^r)\|^2} \cdot \sqrt{R} = \sqrt{R} \sqrt{S} = \sqrt{R S}.
\]
Substituting this into the second term, we have:
\[
\eta T \sqrt{c} \sigma_h \sum_{r=0}^{R-1} \|\nabla f(w^r)\| \leq \eta T \sqrt{c} \sigma_h \sqrt{R S}.
\]
Thus, the inequality becomes:
\[
\begin{aligned}
\mathbb{E}_{\bar z}[f(w^R)] - \mathbb{E}_{\bar z}[f(w^0)] &\leq -\eta T \sqrt{c} (1 + \sqrt{c_h}) S + \eta T \sqrt{c} \sigma_h \sqrt{R S} \\
&\quad + \frac{L^2 \eta^2 T (2+u)(u+2)}{K} \sum_{r=0}^{R-1} \sum_{k=1}^K (f_k(w^r) - f_k^*) \\
&\quad + \frac{L \eta^2 T^2 (2+u) c}{2} \sigma^2 R.
\end{aligned}
\]

\noindent Second, we focus on the term \( \eta T \sqrt{c} \sigma_h \sqrt{R S} \) and apply Young’s Inequality with  \(\delta > 0\) and non-negative real numbers \(x\) and \(y\),
\[
x y \leq \frac{x^2}{2\delta} + \frac{y^2 \delta}{2}.
\]
We identify \( x = \sqrt{S} \) and \( y = \eta T \sqrt{c} \sigma_h \sqrt{R} \), since:
\[
\eta T \sqrt{c} \sigma_h \sqrt{R S} = (\eta T \sqrt{c} \sigma_h \sqrt{R}) \cdot \sqrt{S}.
\]
Applying Young’s Inequality:
\[
\sqrt{S} \cdot (\eta T \sqrt{c} \sigma_h \sqrt{R}) \leq \frac{(\sqrt{S})^2}{2\delta} + \frac{(\eta T \sqrt{c} \sigma_h \sqrt{R})^2\delta}{2}.
\]
Therefore:
\[
\eta T \sqrt{c} \sigma_h \sqrt{R S} \leq \frac{S}{2\delta} + \frac{\eta^2 T^2 c \sigma_h^2 R \delta}{2}.
\]

\[
-\eta T \sqrt{c} \sigma_h \sqrt{R S} \leq \frac{S}{2\delta} + \frac{\eta^2 T^2 c \sigma_h^2 R \delta}{2}.
\]

\noindent Finally we replace the second term in the inequality with the above result:

\[
\begin{aligned}
\mathbb{E}_{\bar z}[f(w^R)] - \mathbb{E}_{\bar z}[f(w^0)] &\leq -\eta T \sqrt{c} (1 + \sqrt{c_h}) S + \left( \frac{S}{2\delta} + \frac{\eta^2 T^2 c \sigma_h^2 R \delta}{2} \right) \\
&\quad + \frac{L^2 \eta^2 T (2+u)^2}{K} \sum_{r=0}^{R-1} \sum_{k=1}^K (f_k(w^r) - f_k^*) \\
&\quad + \frac{L \eta^2 T^2 (2+u) c}{2} \sigma^2 R.
\end{aligned}
\]

\noindent This inequality now depends on \(\delta\).

\begin{equation}\label{eq:grad-sq-bound}
\begin{aligned}
\Bigl(\eta T \sqrt{c}\,\bigl(1 + \sqrt{c_h}\bigr) - \frac{1}{2\delta}\Bigr)
  \sum_{r=0}^{R-1}\|\nabla f(w^r)\|^{2}
\;\le\;&\;
\mathbb{E}_{\bar{z}}\!\bigl[f(w^{0}) - f(w^{R})\bigr] +{\eta^{2}T^{2}c\sigma_{h}^{2}R\delta}{2} \\[4pt]
&+ \frac{L^{2}\eta^{2}T\,(2+u)^{2}}{K}
    \sum_{r=0}^{R-1}\sum_{k=1}^{K}\!\bigl(f_{k}(w^{r}) - f_{k}^{*}\bigr) \\[4pt]
&+ \frac{L\,\eta^{2}T^{2}(2+u)c\,\sigma^{2}R}{2}\;.
\end{aligned}
\end{equation}

\noindent According to Assumption~\ref{assumption:assumption1}, we have:
\[
f_k(\mathbf{w}^*) \leq f_k(\mathbf{w}_k^*) + \langle \nabla f_k(\mathbf{w}_k^*), \mathbf{w}^* - \mathbf{w}_k^* \rangle + \frac{L}{2} \|\mathbf{w}^* - \mathbf{w}_k^*\|_2^2.
\]

\noindent Since \( \mathbf{w}_k^* \) is the minimizer of \( f_k(\mathbf{w}) \), the gradient at the local optimum must be zero:
\[
\nabla f_k(\mathbf{w}_k^*) = 0.
\]

\noindent Substituting this into the inner product term:
\[
\langle \nabla f_k(\mathbf{w}_k^*), \mathbf{w}^* - \mathbf{w}_k^* \rangle = \langle 0, \mathbf{w}^* - \mathbf{w}_k^* \rangle = 0.
\]
\noindent Thus, the inner product term disappears because the gradient at \( \mathbf{w}_k^* \) is zero, making the inner product with any vector (including \( \mathbf{w}^* - \mathbf{w}_k^* \)) equal to zero.

\noindent With the inner product term vanishing, the inequality simplifies to:
\[
f_k(\mathbf{w}^*) \leq f_k(\mathbf{w}_k^*) + \frac{L}{2} \Delta_k.
\]

\noindent This provides an upper bound on \( f_k(\mathbf{w}^*) \) in terms of the local optimal loss \( f_k^* \) and the optimality gap \( \Delta_k \).

\noindent The global optimal loss is defined as:
\[
f^* = f(\mathbf{w}^*) = \sum_{k=1}^K p_k f_k(\mathbf{w}^*).
\]
Using the bound derived for each local loss:
\[
f_k(\mathbf{w}^*) \leq f_k^* + \frac{L}{2} \Delta_k,
\]
we substitute this into the expression for \( f^* \):
\[
f^* = \sum_{k=1}^K p_k f_k(\mathbf{w}^*) \leq \sum_{k=1}^K p_k \left( f_k^* + \frac{L}{2} \Delta_k \right).
\]
Expanding the right-hand side:
\[
f^* \leq \sum_{k=1}^K p_k f_k^* + \frac{L}{2} \sum_{k=1}^K p_k \Delta_k.
\]

\noindent From the above equation, we have:

\[
f^* - \frac{L}{2} \sum_{k=1}^K p_k \Delta_k \leq \sum_{k=1}^K p_k f_k^*.
\]

\[
-\frac{1}{K}\sum_{k=1}^K f_k^* \leq -f^* + \frac{L}{2 K} \sum_{k=1}^K \Delta_k.
\]

\noindent From the equation~\ref{eq:grad-sq-bound}, we have the term:
\[
\frac{L^2 \eta^2 T (2 + u)^2}{K} \sum_{r=0}^{R-1} \sum_{k=1}^K \left( f_k(w^r) - f_k^* \right).
\]
First, we express the double sum as:
\[
\sum_{r=0}^{R-1} \sum_{k=1}^K \left( f_k(w^r) - f_k^* \right) = \sum_{r=0}^{R-1} \left( \sum_{k=1}^K f_k(w^r) - \sum_{k=1}^K f_k^* \right).
\]
Since \( p_k = \frac{1}{K} \), we have:
\[
\sum_{k=1}^K f_k(w^r) = K f(w^r),
\]
where \( f(w^r) = \sum_{k=1}^K p_k f_k(w^r) = \frac{1}{K} \sum_{k=1}^K f_k(w^r) \). Therefore:
\[
\sum_{r=0}^{R-1} \sum_{k=1}^K \left( f_k(w^r) - f_k^* \right) = \sum_{r=0}^{R-1} \left( K f(w^r) - \sum_{k=1}^K f_k^* \right).
\]
From the earlier derivation, we have the inequality:
\[
-\frac{1}{K}\sum_{k=1}^K f_k^* \leq -f^* + \frac{L}{2 K} \sum_{k=1}^K \Delta_k.
\]
Substituting this into the expression above:
\[
\sum_{r=0}^{R-1} \sum_{k=1}^K \left( f_k(w^r) - f_k^* \right) \leq \sum_{r=0}^{R-1} \left( K f(w^r) - \left( K f^* - \frac{L}{2} \sum_{k=1}^K \Delta_k \right) \right).
\]
Thus:
\[
\sum_{r=0}^{R-1} \sum_{k=1}^K \left( f_k(w^r) - f_k^* \right) \leq \sum_{r=0}^{R-1} \left( K f(w^r) - K f^* + \frac{L}{2} \sum_{k=1}^K \Delta_k \right).
\]
Since \( \Delta_k \) is constant across iterations, we can factor it out:
\[
K \sum_{r=0}^{R-1} (f(w^r) - f^*) + \frac{L}{2} \sum_{r=0}^{R-1} \sum_{k=1}^K \Delta_k = K \sum_{r=0}^{R-1} (f(w^r) - f^*) + \frac{L R}{2} \sum_{k=1}^K \Delta_k.
\]
Now, multiply by the coefficient:
\[
\frac{L^2 \eta^2 T (2 + u)^2}{K} \sum_{r=0}^{R-1} \sum_{k=1}^K \left( f_k(w^r) - f_k^* \right) \leq \frac{L^2 \eta^2 T (2 + u)^2}{K} \left[ K \sum_{r=0}^{R-1} (f(w^r) - f^*) + \frac{L R}{2} \sum_{k=1}^K \Delta_k \right].
\]
Simplifying:
\[
L^2 \eta^2 T (2 + u)^2 \sum_{r=0}^{R-1} (f(w^r) - f^*) + \frac{L^3 \eta^2 T (2 + u)^2 R}{2 K} \sum_{k=1}^K \Delta_k.
\]
Substituting this result into the original target inequality, we get:
\[
\begin{aligned}
\left( \eta T \sqrt{c} \left(1 + \sqrt{c_h}\right) - \frac{1}{2\delta} \right) \sum_{r=0}^{R-1} \|\nabla f(w^r)\|^2 
&\leq \mathbb{E}_{\bar{z}} \left[ f(w^0) - f(w^R) \right] + \frac{\eta^2 T^2 c \sigma_h^2 R \delta}{2} \\
&\quad + L^2 \eta^2 T (2 + u)^2 \sum_{r=0}^{R-1} (f(w^r) - f^*) \\
&\quad + \frac{L^3 \eta^2 T (2 + u)^2 R}{2 K} \sum_{k=1}^K \Delta_k \\
&\quad + \frac{L \eta^2 T^2 (2 + u) c \sigma^2 R}{2}.
\end{aligned}
\]
According to the Assumption~\ref{assumption:assumption2} we have:
\[
2\mu (f(\mathbf{w}^r) - f^*) \leq \|\nabla f(\mathbf{w}^r)\|^2, \quad \forall \mathbf{w}^r \in \mathbb{R}^d,
\]

\[
2\mu \sum_{r=0}^{R-1}(f(\mathbf{w}^r) - f^*) \leq \sum_{r=0}^{R-1}\|\nabla f(\mathbf{w}^r)\|^2, \quad \forall \mathbf{w}^r \in \mathbb{R}^d,
\]

\noindent We let $
\eta T \sqrt{c} \left(1 + \sqrt{c_h}\right) - \frac{1}{2\delta} > 0 $ and substitute the above inequality, we have:
\[
\begin{aligned}
2\mu (\eta T \sqrt{c} \left(1 + \sqrt{c_h}\right) - \frac{1}{2\delta})\sum_{r=0}^{R-1}(f(\mathbf{w}^r) - f^*)
&\leq \mathbb{E}_{\bar{z}} \left[ f(w^0) - f(w^R) \right] + \frac{\eta^2 T^2 c \sigma_h^2 R\delta}{2} \\
&\quad + L^2 \eta^2 T (2 + u)^2 \sum_{r=0}^{R-1} (f(w^r) - f^*) \\
&\quad + \frac{L^3 \eta^2 T (2 + u)^2 R}{2 K} \sum_{k=1}^K \Delta_k \\
&\quad + \frac{L \eta^2 T^2 (2 + u) c \sigma^2 R}{2}.
\end{aligned}
\]

\begin{equation}\label{eq:S-bound-expanded}
\begin{aligned}
\sum_{r=0}^{R-1}\bigl(f(w^r)-f^*\bigr)
\;\le\;
&\;\frac{\displaystyle
\mathbb{E}_{\bar z}\bigl[f(w^0)-f(w^R)\bigr]
}
{\displaystyle
2\mu\bigl(\eta T\sqrt{c}(1+\sqrt{c_h})-\frac{1}{2\delta}\bigr)
\;-\;L^2\eta^2T(2+u)^2
}
\\[6pt]
&+\;\frac{\displaystyle
\eta^2T^2\,c\,\sigma_h^2\,R \delta
}
{\displaystyle
2\Bigl[\,
2\mu\bigl(\eta T\sqrt{c}(1+\sqrt{c_h})-\frac{1}{2\delta}\bigr)
\;-\;L^2\eta^2T(2+u)^2
\Bigr]
}
\\[6pt]
&+\;\frac{\displaystyle
L^3\,\eta^2T\,(2+u)^2\,R\;\sum_{k=1}^K\Delta_k
}
{\displaystyle
2K\,
\Bigl[\,
2\mu\bigl(\eta T\sqrt{c}(1+\sqrt{c_h})-\frac{1}{2\delta}\bigr)
\;-\;L^2\eta^2T(2+u)^2
\Bigr]
}
\\[6pt]
&+\;\frac{\displaystyle
L\,\eta^2T^2\,(2+u)\,c\,\sigma^2\,R
}
{\displaystyle
2\Bigl[\,
2\mu\bigl(\eta T\sqrt{c}(1+\sqrt{c_h})-\frac{1}{2\delta}\bigr)
\;-\;L^2\eta^2T(2+u)^2
\Bigr]
}\,.
\end{aligned}
\end{equation}

\begin{equation}\label{eq:S-bound-expanded-2}
\begin{aligned}
\frac{1}{R}\sum_{r=0}^{R-1}\bigl(f(w^r)-f^*\bigr)
\;\le\;
&\;\frac{1}{R}\frac{\displaystyle
\mathbb{E}_{\bar z}\bigl[f(w^0)-f(w^R)\bigr]
}
{\displaystyle
2\mu\bigl(\eta T\sqrt{c}(1+\sqrt{c_h})-\frac{1}{2\delta}\bigr)
\;-\;L^2\eta^2T(2+u)^2
}
\\[6pt]
&+\;\frac{\displaystyle
\eta^2T^2\,c\,\sigma_h^2\ \delta
}
{\displaystyle
2\Bigl[\,
2\mu\bigl(\eta T\sqrt{c}(1+\sqrt{c_h})-\frac{1}{2\delta}\bigr)
\;-\;L^2\eta^2T(2+u)^2
\Bigr]
}
\\[6pt]
&+\;\frac{\displaystyle
L^3\,\eta^2T\,(2+u)^2\;\sum_{k=1}^K\Delta_k
}
{\displaystyle
2K\,
\Bigl[\,
2\mu\bigl(\eta T\sqrt{c}(1+\sqrt{c_h})-\frac{1}{2\delta}\bigr)
\;-\;L^2\eta^2T(2+u)^2
\Bigr]
}
\\[6pt]
&+\;\frac{\displaystyle
L\,\eta^2T^2\,(2+u)\,c\,\sigma^2\
}
{\displaystyle
2\Bigl[\,
2\mu\bigl(\eta T\sqrt{c}(1+\sqrt{c_h})-\frac{1}{2\delta}\bigr)
\;-\;L^2\eta^2T(2+u)^2
\Bigr]
}\,.
\end{aligned}
\end{equation}

\noindent We select \(\delta = \frac{1}{\eta T \sqrt{c} (1 + \sqrt{c_h})}\), which leads to:

\[
\frac{1}{2\delta} = \frac{\eta T \sqrt{c} (1 + \sqrt{c_h})}{2}
\]
Substituting into the denominator:

\[
2\mu \left( \eta T \sqrt{c} (1 + \sqrt{c_h}) - \frac{\eta T \sqrt{c} (1 + \sqrt{c_h})}{2} \right) = \mu \eta T \sqrt{c} (1 + \sqrt{c_h})
\]
With the chosen \(\delta\), we have:

\begin{equation}\label{eq:avg-loss-bound-final}
\begin{aligned}
\frac{1}{R} \sum_{r=0}^{R-1} \left( f(w^r) - f^* \right) &\leq \frac{1}{R} \cdot \frac{ \mathbb{E}_{\bar{z}} \left[ f(w^0) - f(w^R) \right] }{ \mu \eta T \sqrt{c} (1 + \sqrt{c_h}) - L^2 \eta^2 T (2 + u)^2 } \\
&\quad + \frac{ \sqrt{c} \sigma_h^2 }{ 2 (1 + \sqrt{c_h}) \left[ \mu \sqrt{c} (1 + \sqrt{c_h}) - L^2 \eta (2 + u)^2 \right] } \\
&\quad + \frac{ L^3 \eta (2 + u)^2 \sum_{k=1}^K \Delta_k }{ 2 K \left[ \mu \sqrt{c} (1 + \sqrt{c_h}) - L^2 \eta (2 + u)^2 \right] } \\
&\quad + \frac{ L \eta T (2 + u) c \sigma^2 }{ 2 \left[ \mu \sqrt{c} (1 + \sqrt{c_h}) - L^2 \eta (2 + u)^2 \right] },
\end{aligned}
\end{equation}

\noindent where the step-size $\eta$ must satisfy:
$ \eta < \frac{\mu \sqrt{c}(1+\sqrt{c_h})}{L^2 (2+u)^2} \quad$ to ensure denominator positivity.

Plugging in a constant learning rate $\eta = \min\left\{ \frac{1}{L(u+2)}, \frac{\mu\,\sqrt{c}\bigl(1+\sqrt{c_h}\bigr)}
        {2\,L^{2}(2+u)^{2}} \right\}$. We substitute this $\eta$ to equation~\ref{eq:avg-loss-bound-final} and get:

\[
\begin{aligned}
\frac{1}{R}\sum_{r=0}^{R-1}\!\bigl(f(w^{r})-f^{*}\bigr)
\;\le\;&
\frac{4L^{2}(2+u)^{2}}
     {\mu^{2}c\bigl(1+\sqrt{c_h}\bigr)^{2}TR}\;
     \mathbb{E}_{\bar z}\!\bigl[f(w^{0})-f^*\bigr] \\[6pt]
&+\frac{\sigma_{h}^{2}}
       {\mu\bigl(1+\sqrt{c_h}\bigr)^{2}} +\frac{L}{K}\,\sum_{k=1}^{K}\Delta_{k} +\frac{T\,c\,\sigma^{2}}
       {2L\,(2+u)}\;.
\end{aligned}
\]

\begin{equation}\label{eq:avg-loss-bound-fixed}
\begin{aligned}
\frac{1}{R}\sum_{r=0}^{R-1}\bigl(f(w^{r}) - f^{*}\bigr) 
\le \mathcal{O}\!\left( \frac{(2+u)^2}{TR} \cdot \mathbb{E}[f(w^0) - f(w^R)] \right)
+ \mathcal{O}\!\left( \frac{T}{2+u} \right)
+ \mathcal{O}(1).
\end{aligned}
\end{equation}
\end{proof}

\subsection{\textsc{Meerkat-vp} Convergence Analysis}\label{sec:vp-client-selection}
We propose a Virtual Path Client Selection (\textsc{Meerkat-vp}) mechanism that identifies clients with highly heterogeneous data distributions based on their optimization trajectories. Instead of excluding them, \textsc{Meerkat-vp} applies early stopping to these clients to limit their adverse influence on global model updates while still preserving their participation.

\begin{proof}

\noindent \textbf{Motivation for Early Stopping:} In federated learning, clients perform local updates starting from the global model \( w^r \). For \( T > 1 \), clients may drift towards their local optima, introducing bias into the global update due to data heterogeneity. By identifying "bad" clients and limiting them to one update step, we reduce their drift and align their contributions more closely with the global gradient.

We divide the \( K \) clients into two groups:
\begin{itemize}[nosep,leftmargin=*]
    \item \textbf{Balanced‑distribution clients ($K_g$)}: Perform \( T \) local step updates.
    \item \textbf{Skewed‑distribution clients ($K_b$)}: Perform only 1 local step update.
\end{itemize}

The global model update becomes:
\[
w^{r+1} = w^r + \frac{1}{K} \sum_{k \in K_g} (w_k^{r,T} - w^r) + \frac{1}{K} \sum_{k \in K_b} (w_k^{r,1} - w^r)
\]
where:
\[
w_k^{r,T} - w^r = -\eta \sum_{t=0}^{T-1} \hat{\nabla}f_k(w^{r,t}), \quad w_k^{r,1} - w^r = -\eta \hat{\nabla}f_k(w^r)
\]
\paragraph{Loss Descent Analysis} Using the \( L \)-smoothness property:
\[
f(w^{r+1}) \leq f(w^r) + \langle \nabla f(w^r), w^{r+1} - w^r \rangle + \frac{L}{2} \|w^{r+1} - w^r\|^2
\]

We analyze the inner product term:

\begin{align*}
\langle \nabla f(w^r), w^{r+1} - w^r \rangle &= \sum_{k=1}^K p_k \langle\nabla f(\mathbf{w}^r), \mathbf{w}_k^{r,T} - \mathbf{w}^r\rangle
\end{align*}

\begin{align*}
\sum_{k=1}^K p_k \langle\nabla f(\mathbf{w}^r), \mathbf{w}_k^{r,T} - \mathbf{w}^r\rangle &= -\eta\sum_{k=1}^K p_k \langle\nabla f(\mathbf{w}^r), \sum_{t=0}^{T-1} \hat{\nabla}f_k(\mathbf{w}^{r,t},\bar{\mathbf{z}}_t)\rangle \\
&= -\eta\sum_{k=1}^K p_k \sum_{t=0}^{T-1} \langle\nabla f(\mathbf{w}^r), \hat{\nabla}f_k(\mathbf{w}^{r,t},\bar{\mathbf{z}}_t)\rangle \\
\end{align*}

Since we have balanced‑distribution clients and skewed‑distribution clients:

\[
\langle \nabla f(w^r), w^{r+1} - w^r \rangle = \frac{1}{K} \sum_{k \in K_g} \langle \nabla f(w^r), w_k^{r,T} - w^r \rangle + \frac{1}{K} \sum_{k \in K_b} \langle \nabla f(w^r), w_k^{r,1} - w^r \rangle
\]

\begin{equation}
\label{eq:virtual-path-inner-product}
\begin{aligned}
\langle \nabla f(w^r), w^{r+1} - w^r \rangle 
&= -\frac{\eta}{K} \sum_{k \in K_g} \sum_{t=0}^{T-1} \langle \nabla f(w^r), \hat{\nabla}f_k(w^{r,t}) \rangle \\
&\quad - \frac{\eta}{K} \sum_{k \in K_b} \langle \nabla f(w^r), \hat{\nabla}f_k(w^r) \rangle
\end{aligned}
\end{equation}

\noindent Since $\hat{\nabla} f_k^t$ is unbiased, we have: 

\[
\sum_{k=1}^K \sum_{t=0}^{T-1}
\bigl\langle \nabla f(w^r), \,\hat\nabla f_k(\mathbf{w}^{r,t},\bar{\mathbf z}_t) \bigr\rangle
\;=\;
\sum_{k=1}^K \sum_{t=0}^{T-1}
\bigl\langle \nabla f(w^r), \,\mathbb{E}_{\bar z}\!\bigl[\hat\nabla f_k(\mathbf{w}^{r,t},\bar{\mathbf z}_t)\bigr]\bigr\rangle
\]

\noindent We substitute the equation~\ref{eq:masked_unbiased} and get:
\[
\sum_{k=1}^K \sum_{t=0}^{T-1}
\bigl\langle \nabla f(w^r), \,\mathbb{E}_{\bar z}\!\bigl[\hat\nabla f_k(\mathbf{w}^{r,t},\bar{\mathbf z}_t)\bigr]\bigr\rangle
\;=\;
\sum_{k=1}^K \sum_{t=0}^{T-1}
\bigl\langle \nabla f(w^r), \,\mathbf m\odot\nabla f_k(w^{r,t})\bigr\rangle.
\]

Thus taking expectation of equation~\ref{eq:virtual-path-inner-product}, we can get:

\begin{equation}
\label{eq:expected-inner-product-vp}
\begin{aligned}
\mathbb{E}_{\bar z} \langle \nabla f(\mathbf{w}^r), \mathbf{w}^{r+1} - \mathbf{w}^r \rangle 
= -\frac{\eta}{K} \Bigg( 
&\sum_{k \in K_g} \sum_{t=0}^{T-1} \langle \nabla f(\mathbf{w}^r),\, \mathbf{m} \odot \nabla f_k(\mathbf{w}^{r,t}) \rangle \\
&+ \sum_{k \in K_b} \langle \nabla f(\mathbf{w}^r),\, \mathbf{m} \odot \nabla f_k(\mathbf{w}^r) \rangle 
\Bigg)
\end{aligned}
\end{equation}

\noindent Under the Cauchy–Schwarz inequality, we have:

\[
\bigl\langle \nabla f(\mathbf{w}^{r}),\;\mathbf{m}\odot\nabla f_{k}(\mathbf{w}^{r,t})\bigr\rangle
\;\le\;
\|\nabla f(\mathbf{w}^{r})\|\;\|\mathbf{m}\odot\nabla f_{k}(\mathbf{w}^{r,t})\|
\]
\noindent We substitute Assumption~\ref{assumption:assumption5} get:
\[
\|\nabla f(\mathbf{w}^{r})\|\;\|\mathbf{m}\odot\nabla f_{k}(\mathbf{w}^{r,t})\|
\;=\;
\sqrt{c}\,\|\nabla f(\mathbf{w}^{r})\|\;\|\nabla f_{k}(\mathbf{w}^{r,t})\|.
\]
\noindent Thus we get:
\[
\bigl\langle \nabla f(\mathbf{w}^{r}),\;\mathbf{m}\odot\nabla f_{k}(\mathbf{w}^{r,t})\bigr\rangle
\;\le\;
\sqrt{c}\,\|\nabla f(\mathbf{w}^{r})\|\;\|\nabla f_{k}(\mathbf{w}^{r,t})\|.
\]
\noindent By the triangle inequality, we have
\[
\|\nabla f_k(\mathbf{w}^{r,t})\|
\;\le\;
\|\nabla f(\mathbf{w}^{r})\| + \|\nabla f_k(\mathbf{w}^{r,t}) - \nabla f(\mathbf{w}^{r})\|
\]
We substitute Assumption~\ref{assumption:assumption3} and use the properties of square roots we get:
\[
\begin{aligned}
\|\nabla f(\mathbf{w}^{r})\|
+ \|\nabla f_k(\mathbf{w}^{r,t}) - \nabla f(\mathbf{w}^{r})\|
&\le
\|\nabla f(\mathbf{w}^{r})\|
+ \sqrt{\,c_h\,\|\nabla f(\mathbf{w}^{r})\|^2 + \sigma_h^2\,}
\\[4pt]
&\le
(1 + \sqrt{c_h})\,\|\nabla f(\mathbf{w}^{r})\| + \sigma_h.
\end{aligned}
\]
\noindent Using the bound
\(\langle\nabla f(\mathbf{w}^{r}),\;m\odot\nabla f_k(\mathbf{w}^{r,t})\rangle \le \sqrt{c}\,\|\nabla f(\mathbf{w}^{r})\|\;\|\nabla f_k\|\)
from Cauchy–Schwarz and Assumption~\ref{assumption:assumption5}, and then plugging in the above,
we obtain
\begin{align*}
\langle\nabla f(\mathbf{w}^{r}),\,m\odot\nabla f_k(\mathbf{w}^{r,t})\rangle
&\le
\sqrt{c}\,\|\nabla f(\mathbf{w}^{r})\|\bigl[(1 + \sqrt{c_h})\,\|\nabla f(\mathbf{w}^{r})\| + \sigma_h\bigr]\\
&\le
\sqrt{c}\,(1 + \sqrt{c_h})\,\|\nabla f(\mathbf{w}^{r})\|^2
\;+\;
\sqrt{c}\,\sigma_h\,\|\nabla f(\mathbf{w}^{r})\|.
\end{align*}
Since this bound holds uniformly for all \(k\) and \(t\), and based on the equation~\ref{eq:expected-inner-product-vp} we get:
\begin{align*}
&\sum_{k \in K_g} \sum_{t=0}^{T-1} \langle \nabla f(\mathbf{w}^r), \mathbf{m} \odot \nabla f_k(\mathbf{w}^{r,t}) \rangle + \sum_{k \in K_b} \langle \nabla f(\mathbf{w}^r), \mathbf{m} \odot \nabla f_k(\mathbf{w}^r) \rangle \\
&\leq \left( |K_g| T + |K_b| \right) \left[ \sqrt{c} (1 + \sqrt{c_h}) \|\nabla f(\mathbf{w}^r)\|^2 + \sqrt{c} \sigma_h \|\nabla f(\mathbf{w}^r)\| \right].
\end{align*}
We get:
\begin{equation}
\label{eq:expected-function-decrease-vp}
\begin{aligned}
\mathbb{E}_{\bar z}[f(w^{r+1})] 
&\leq \mathbb{E}_{\bar z}[f(w^r)] 
- \frac{\eta \sqrt{c} \alpha}{K}(1 + \sqrt{c_h}) \|\nabla f(w^r)\|^2 \\
&\quad - \frac{\eta \sqrt{c} \alpha}{K} \sigma_h \|\nabla f(w^r)\| 
+ \frac{L}{2} \mathbb{E}_{\bar z} \|w^{r+1} - w^r\|^2
\end{aligned}
\end{equation}
where \(\alpha = |K_g| T + |K_b|\).

Since the global model update is given by:
\[
w^{r+1} = w^r + \frac{1}{K} \sum_{k \in K_g} (w_k^{r,T} - w^r) + \frac{1}{K} \sum_{k \in K_b} (w_k^{r,1} - w^r)
\]
We substitute the local updates and the squared norm is:
\[
\| w^{r+1} - w^r \|^2 = \frac{\eta^2}{K^2} \left\| \sum_{k \in K_g} \sum_{t=0}^{T-1} \hat{\nabla}f_k(w^{r,t}) + \sum_{k \in K_b} \hat{\nabla}f_k(w^r) \right\|^2
\]
Define the update contribution per client:
\[
\hat\Delta_k =
\begin{cases}
-\eta \sum_{t=0}^{T-1} \hat{\nabla}f_k(w^{r,t}) & \text{if } k \in K_g, \\
-\eta \hat{\nabla}f_k(w^r) & \text{if } k \in K_b.
\end{cases}
\]
Then:
\[
w^{r+1} - w^r = \frac{1}{K} \sum_{k=1}^K \hat\Delta_k
\]
\[
\| w^{r+1} - w^r \|^2 = \frac{1}{K^2} \left\| \sum_{k=1}^K \hat\Delta_k \right\|^2
\]
Using the Cauchy-Schwarz inequality:
\begin{equation*}
\left\| \sum_{k=1}^K \hat\Delta_k \right\|^2 \leq K \sum_{k=1}^K \| \hat\Delta_k \|^2,
\quad \text{where } \hat\Delta_k \text{ denotes the actual model update on client } k.
\end{equation*}
So:
\[
\| w^{r+1} - w^r \|^2 \leq \frac{1}{K} \sum_{k=1}^K \| \hat\Delta_k \|^2
\]
Now compute \(\| \hat\Delta_k \|^2\):
\[
\| \hat\Delta_k \|^2 = \eta^2 \left\| \sum_{t=0}^{T-1} \hat{\nabla}f_k(w^{r,t}) \right\|^2 \quad \text{for } k \in K_g,
\]
\[
\| \hat\Delta_k \|^2 = \eta^2 \left\| \hat{\nabla}f_k(w^r) \right\|^2 \quad \text{for } k \in K_b.
\]
Thus:
\[
\| w^{r+1} - w^r \|^2 \leq \frac{\eta^2}{K} \left( \sum_{k \in K_g} \left\| \sum_{t=0}^{T-1} \hat{\nabla}f_k(w^{r,t}) \right\|^2 + \sum_{k \in K_b} \left\| \hat{\nabla}f_k(w^r) \right\|^2 \right)
\]
We take the expectation:
\[
\mathbb{E}_{\bar{z}} \| w^{r+1} - w^r \|^2 \leq \frac{\eta^2}{K} \left( \sum_{k \in K_g} \mathbb{E}_{\bar{z}} \left\| \sum_{t=0}^{T-1} \hat{\nabla}f_k(w^{r,t}) \right\|^2 + \sum_{k \in K_b} \mathbb{E}_{\bar{z}} \left\| \hat{\nabla}f_k(w^r) \right\|^2 \right)
\]
For \( k \in K_b \):
\[
\mathbb{E}_{\bar{z}} \left\| \hat{\nabla}f_k(w^r) \right\|^2 = (2 + u) c \left\| \nabla f_k(w^r) \right\|^2
\]
For \( k \in K_g \):
\[
\mathbb{E}_{\bar{z}} \left\| \sum_{t=0}^{T-1} \hat{\nabla}f_k(w^{r,t}) \right\|^2
\]
Using the Cauchy-Schwarz inequality:
\[
    \mathbb{E}_{\bar{z}} \left\| \sum_{t=0}^{T-1} \hat{\nabla}f_k(w^{r,t}) \right\|^2 \leq T \sum_{t=0}^{T-1} \mathbb{E}_{\bar{z}} \left\| \hat{\nabla}f_k(w^{r,t}) \right\|^2
\]
According to the lemma~\ref{lem:A1}:
\[
\mathbb{E}_{\bar{z}} \left\| \hat{\nabla}f_k(w^{r,t}) \right\|^2 = (2 + u) c \left\| \nabla f_k(w^{r,t}) \right\|^2
\]
So:
\[
\mathbb{E}_{\bar{z}} \left\| \sum_{t=0}^{T-1} \hat{\nabla}f_k(w^{r,t}) \right\|^2 \leq T (2 + u) c \sum_{t=0}^{T-1} \left\| \nabla f_k(w^{r,t}) \right\|^2
\]
Combine the terms we get:
\[
\mathbb{E}_{\bar{z}} \| w^{r+1} - w^r \|^2 \leq \frac{\eta^2 (2 + u) c}{K} \left( T \sum_{k \in K_g} \sum_{t=0}^{T-1} \left\| \nabla f_k(w^{r,t}) \right\|^2 + \sum_{k \in K_b} \left\| \nabla f_k(w^r) \right\|^2 \right)
\]

We substitute this inequality to the equation~\ref{eq:expected-function-decrease-vp}.
\begin{equation}
\label{eq:l-smooth-vp}
\begin{aligned}
\mathbb{E}_{\bar z}[f(w^{r+1})] 
&\leq \mathbb{E}_{\bar z}[f(w^r)] 
- \frac{\eta \sqrt{c} \alpha}{K}(1 + \sqrt{c_h}) \left\| \nabla f(w^r) \right\|^2 \\
&\quad - \frac{\eta \sqrt{c} \alpha}{K} \sigma_h \left\| \nabla f(w^r) \right\| \\
&\quad + \frac{\eta^2 (2 + u) c L}{2K} \left( 
    T \sum_{k \in K_g} \sum_{t=0}^{T-1} \left\| \nabla f_k(w^{r,t}) \right\|^2 
    + \sum_{k \in K_b} \left\| \nabla f_k(w^r) \right\|^2 
\right)
\end{aligned}
\end{equation}

\begin{align*}
\mathbb{E}_{\bar z}\bigl[f(w^{r+1})\bigr]
&\le
\mathbb{E}_{\bar z}\bigl[f(w^r)\bigr]
- \frac{\eta\sqrt{c}\,\alpha}{K}\,(1 + \sqrt{c_h})\,\bigl\|\nabla f(w^r)\bigr\|^2
- \frac{\eta\sqrt{c}\,\alpha}{K}\,\sigma_h \bigl\|\nabla f(w^r)\bigr\|\\[6pt]
&\quad
+ \frac{\eta^2(2+u)\,c\,L\,T}{2K}
  \sum_{k\in K_g}\sum_{t=0}^{T-1}\bigl\|\nabla f_k(w^{r,t})\bigr\|^2
+ \frac{\eta^2(2+u)\,c\,L}{2K}
  \sum_{k\in K_b}\bigl\|\nabla f_k(w^r)\bigr\|^2
\end{align*}

According to the equation~\ref{eq:client_grad_convergence_rate}, we know that the client-average squared gradient has upper bound.
\begin{equation}\label{eq:l-smooth-vp-clientavg-bound}
\begin{aligned}
\mathbb{E}_{\bar z}\bigl[f(w^{r+1})\bigr]
&\le 
\mathbb{E}_{\bar z}\bigl[f(w^r)\bigr]
- \frac{\eta\sqrt{c}\,\alpha}{K}\,(1 + \sqrt{c_h})\,\bigl\|\nabla f(w^r)\bigr\|^2
- \frac{\eta\sqrt{c}\,\alpha}{K}\,\sigma_h\,\bigl\|\nabla f(w^r)\bigr\|
\\[6pt]
&\quad
+ \frac{\eta^2(2+u)\,c\,L\,T}{2K}
  \sum_{k\in K_g}
  \Bigl[
    \frac{2L(2+u)}{c}\bigl(f_k(w_k^{0,r}) - f_k^*\bigr)
    + T\,\sigma^2
  \Bigr]
\\[6pt]
&\quad
+ \frac{\eta^2(2+u)\,c\,L}{2K}
  \sum_{k\in K_b}\bigl\|\nabla f_k(w^r)\bigr\|^2.
\end{aligned}
\end{equation}
Using Assumption~\ref{assumption:assumption3}, which states that for any $\theta \in \mathbb{R}^d$,
\[
\left\| \nabla f(\theta) - \nabla f_i(\theta) \right\|^2 \leq c_h \left\| \nabla f(\theta) \right\|^2 + \sigma_h^2,
\]
we can bound the squared norm of the local gradient $\left\| \nabla f_k(w^r) \right\|^2$. Specifically, by the inequality $(x + y)^2 \leq 2x^2 + 2y^2$, we have:
\[
\left\| \nabla f_k(w^r) \right\|^2 = \left\| \nabla f(w^r) + (\nabla f_k(w^r) - \nabla f(w^r)) \right\|^2 \leq 2 \left\| \nabla f(w^r) \right\|^2 + 2 \left\| \nabla f_k(w^r) - \nabla f(w^r) \right\|^2.
\]
Then, applying Assumption~\ref{assumption:assumption3} with $\theta = w^r$ and $i = k$:
\[
\left\| \nabla f_k(w^r) - \nabla f(w^r) \right\|^2 \leq c_h \left\| \nabla f(w^r) \right\|^2 + \sigma_h^2.
\]
Therefore,
\[
\left\| \nabla f_k(w^r) \right\|^2 \leq 2 \left\| \nabla f(w^r) \right\|^2 + 2 \left( c_h \left\| \nabla f(w^r) \right\|^2 + \sigma_h^2 \right) = (2 + 2 c_h) \left\| \nabla f(w^r) \right\|^2 + 2 \sigma_h^2.
\]
Thus, we obtain the bound:
\[
\left\| \nabla f_k(w^r) \right\|^2 \leq (2 + 2 c_h) \left\| \nabla f(w^r) \right\|^2 + 2 \sigma_h^2.
\]
We substitute the bound to the inequality~\ref{eq:l-smooth-vp-clientavg-bound}, according to the Assumption~\ref{assumption:assumption3}, we substitute the last term:

\begin{align*}
\mathbb{E}_{\bar z}\bigl[f(w^{r+1})\bigr]
&\le
\mathbb{E}_{\bar z}\bigl[f(w^r)\bigr]
- \frac{\eta\sqrt{c}\,\alpha}{K}\,(1 + \sqrt{c_h})\,\|\nabla f(w^r)\|^2
- \frac{\eta\sqrt{c}\,\alpha}{K}\,\sigma_h\,\|\nabla f(w^r)\|
\\[6pt]
&\quad
+ \frac{\eta^2(2+u)\,c\,L\,T}{2K}
  \sum_{k\in K_g}
  \Bigl[
    \frac{2L(2+u)}{c}\bigl(f_k(w_k^{0,r}) - f_k^*\bigr)
    + T\,\sigma^2
  \Bigr]
\\[6pt]
&\quad
+ \frac{\eta^2(2+u)\,c\,L}{2K}
  \sum_{k\in K_b}
  \Bigl[
    (2+2c_h)\,\|\nabla f(w^r)\|^2
    + \sigma_h^2
  \Bigr].
\end{align*}

\begin{align*}
\mathbb{E}_{\bar z}\bigl[f(w^{r+1})\bigr]
&\le
\mathbb{E}_{\bar z}\bigl[f(w^r)\bigr]
- \frac{\eta\sqrt{c}\,\alpha}{K}\,(1 + \sqrt{c_h})\,\|\nabla f(w^r)\|^2
- \frac{\eta\sqrt{c}\,\alpha}{K}\,\sigma_h\,\|\nabla f(w^r)\|
\\[6pt]
&\quad
+ \frac{\eta^2(2+u)\,c\,L\,T}{2K}
  \sum_{k\in K_g}
  \Bigl[
    \frac{2L(2+u)}{c}\bigl(f_k(w_k^{0,r}) - f_k^*\bigr)
    + T\,\sigma^2
  \Bigr]
\\[6pt]
&\quad
+ \frac{\eta^2(2+u)\,c\,L\,|K_b|\,(1 + c_h)}{K}\,\|\nabla f(w^r)\|^2
+ \frac{\eta^2(2+u)\,c\,L\,|K_b|\,\sigma_h^2}{K}.
\end{align*}

\begin{equation}\label{eq:l-smooth-vp-merged}
\begin{aligned}
\mathbb{E}_{\bar z}\bigl[f(w^{r+1})\bigr]
&\le
\mathbb{E}_{\bar z}\bigl[f(w^r)\bigr]
\\[4pt]
&\quad
+ \frac{\eta^2(2+u)\,c\,L\,K_b\,(1+c_h)
       -\;\eta\sqrt{c}\,\alpha}{K}\,
  \bigl\|\nabla f(w^r)\bigr\|^2
\\[6pt]
&\quad
- \frac{\eta\sqrt{c}\,\alpha\,\sigma_h}{K}\,
  \bigl\|\nabla f(w^r)\bigr\|
\\[6pt]
&\quad
+ \frac{\eta^2(2+u)^2\,L^2 T}{K}
  \sum_{k\in K_g}\!\bigl(f_k(w_k^{0,r})-f_k^*\bigr)
\\[6pt]
&\quad
+ \frac{\eta^2(2+u)\,c\,L}{2K}
  \Bigl(T^2\,\sigma^2\, K_g
        +2\, K_b \sigma_h^2\Bigr).
\end{aligned}
\end{equation}

\noindent \textbf{Accumulating Over $R$ Rounds.}
Summing equation~\ref{eq:l-smooth-vp-merged} over $r=0$ to $R-1$,

\begin{equation}\label{eq:aggregate-descent}
\begin{aligned}
\mathbb{E}_{\bar z}\bigl[f(w^R)\bigr] - \mathbb{E}_{\bar z}\bigl[f(w^0)\bigr]
&\le
\frac{\eta^2(2+u)\,c\,L\,K_b\,(1 + c_h)
      - \eta\sqrt{c}\,\alpha}{K}
\sum_{r=0}^{R-1}\bigl\|\nabla f(w^r)\bigr\|^2
\\[6pt]
&\quad
- \frac{\eta\sqrt{c}\,\alpha\,\sigma_h}{K}
  \sum_{r=0}^{R-1}\bigl\|\nabla f(w^r)\bigr\|
\\[6pt]
&\quad
+ \frac{\eta^2(2+u)^2\,L^2\,T}{K}
  \sum_{r=0}^{R-1}\sum_{k\in K_g}\bigl(f_k(w_k^{0,r}) - f_k^*\bigr)
\\[6pt]
&\quad
+ \frac{\eta^2(2+u)\,c\,L\,R}{2K}
  \Bigl(T^2\,\sigma^2\, K_g + 2\,K_b\,\sigma^2\Bigr).
\end{aligned}
\end{equation}

According to our previous derivation, we know that:
\begin{equation}\label{eq:cauchy-schwarz-conclusion}
\sum_{r=0}^{R-1} \bigl\|\nabla f(w^r)\bigr\|
\;\le\;
\sqrt{R}\,
\sqrt{\sum_{r=0}^{R-1} \bigl\|\nabla f(w^r)\bigr\|^2}\,.
\end{equation}

Apply Young's inequality with \(\delta > 0\) and nonnegative real numbers \(x\) and \(y\),
\[
x y \leq \frac{x^2}{2\delta} + \frac{y^2 \delta}{2}.
\]
\begin{align*}
\frac{\eta\sqrt{c}\,\sigma_h \alpha}{K}
\sum_{r=0}^{R}\|\nabla f(w^r)\|
&\le 
\frac{\eta\sqrt{c}\,\alpha \sigma_h\sqrt{R}}{K}
\sqrt{\sum_{r=0}^{R}\|\nabla f(w^r)\|^2}
\\[6pt]
&\le 
\frac{1}{2\delta}\sum_{r=0}^{R}\|\nabla f(w^r)\|^2
\;+\;
\frac{\eta^2\,c\,\alpha^2 \sigma_h^2\,R\,\delta}{2\,K^2}
\\[6pt]
-\frac{\eta\sqrt{c}\,\sigma_h \alpha}{K}
\sum_{r=0}^{R}\|\nabla f(w^r)\|
&\le 
\frac{1}{2\delta}\sum_{r=0}^{R}\|\nabla f(w^r)\|^2
\;+\;
\frac{\eta^2\,c\,\alpha^2 \sigma_h^2\,R\,\delta}{2\,K^2}.
\end{align*}

We substitute this to the equation~\ref{eq:aggregate-descent}.

\begin{equation}\label{eq:final-aggregate-bound}
\begin{aligned}
\mathbb{E}_{\bar z}\bigl[f(w^R)\bigr]\;-\;\mathbb{E}_{\bar z}\bigl[f(w^0)\bigr]
&\le
\Biggl(
  \frac{\eta^2(2+u)\,c\,L\,K_b\,(1 + c_h)
        \;-\;\eta\sqrt{c}\,\alpha}{K}
  \;+\;\frac{1}{2\delta}
\Biggr)
\sum_{r=0}^{R-1}\|\nabla f(w^r)\|^2
\\[6pt]
&\quad
+ \frac{\eta^2\,c\,\alpha^2 \sigma_h^2\,R\,\delta}{2\,K^2}
+ \frac{\eta^2(2+u)^2\,L^2\,T}{K}
  \sum_{r=0}^{R}\sum_{k\in K_g}\bigl(f_k(w_k^{0,r}) - f_k^*\bigr)
\\[6pt]
&\quad
+ \frac{\eta^2(2+u)\,c\,L\,R}{2\,K}
  \Bigl(T^2\,\sigma^2\,K_g
        + 2\,K_b\,\sigma_h^2\Bigr).
\end{aligned}
\end{equation}

Given that $w_k^{0,r} = w^r$, this term is equivalent to $\sum_{r=0}^{R}\sum_{k\in K_g}\bigl(f_k(w^r) - f_k^*\bigr)$.

From our previous discussion, we have the inequality for a single round $r$:
\[\sum_{k\in K_g}\bigl(f_k(w^r) - f_k^*\bigr) \leq \sum_{k=1}^K\bigl(f_k(w^r) - f_k^*\bigr)\]
and the inequality used in Part 2 of the proof:
\[
\sum_{r=0}^{R-1} \sum_{k=1}^K \left( f_k(w^r) - f_k^* \right) \leq \sum_{r=0}^{R-1} \left( K f(w^r) - K f^* + \frac{L}{2} \sum_{k=1}^K \Delta_k \right).
\]

Combining these two inequalities, we obtain a bound for the sum over the set $K_g$:

We set $\gamma \le 1$ which means that the subset clients the effect to the global:
\[
\sum_{r=0}^{R-1}\sum_{k\in K_g}\bigl(f_k(w^r)-f_k^*\bigr)
\;\le\;
\gamma\,
\sum_{r=0}^{R-1}\!\Bigl(
      K\,\bigl(f(w^r)-f^*\bigr)
      +\frac{L}{2}\sum_{k=1}^{K}\!\Delta_k
\Bigr)
\]

We substitute this to the above inequality get:
\begin{align*}
\mathbb{E}_{\bar z}\bigl[f(w^R)\bigr]
- \mathbb{E}_{\bar z}\bigl[f(w^0)\bigr]
&\le 
\left(
  \frac{\eta^2 (2+u)\,c\,L\,K_b\,(1 + c_h)}{K}
  - \frac{\eta \sqrt{c}\,\alpha}{K}
  + \frac{1}{2\delta}
\right)
\sum_{r=0}^{R-1} \|\nabla f(w^r)\|^2
\\[6pt]
&\quad
+ \frac{\eta^2\,c\,\alpha^2\,\sigma_h^2\,R\,\delta}{2\,K^2}
+ \eta^2 (2+u)^2\,L^2\,T\,\gamma 
  \sum_{r=0}^{R-1} \bigl(f(w^r) - f^*\bigr)
\\[6pt]
&\quad
+ \frac{\eta^2 (2+u)^2\,L^3\,T\,R\,\gamma}{2\,K}
  \sum_{k=1}^K \Delta_k
\\[6pt]
&\quad
+ \frac{\eta^2 (2+u)\,c\,L\,R}{2\,K}
  \Bigl(T^2\,\sigma^2\,K_g + 2\,K_b\,\sigma_h^2\Bigr).
\end{align*}

We substitute $\alpha$:
\[
\begin{aligned}
\mathbb{E}_{\bar z}\bigl[f(w^R)\bigr] - \mathbb{E}_{\bar z}\bigl[f(w^0)\bigr]
&\leq
\left(
  \frac{\eta^2 (2+u) c L K_b (1 + c_h)}{K} - \frac{\eta \sqrt{c} (K_g T + K_b)}{K}
  + \frac{1}{2\delta}
\right)
\sum_{r=0}^{R-1} \|\nabla f(w^r)\|^2
\\
&\quad
+ \frac{\eta^2 c (K_g^2 T^2 + 2 K_g T K_b + K_b^2) \sigma_h^2 R \delta}{2 K^2}
+ \eta^2 (2+u)^2 L^2 T \gamma \sum_{r=0}^{R-1} (f(w^r) - f^*)
\\
&\quad
+ \frac{\eta^2 (2+u)^2 L^3 T R \gamma}{2 K} \sum_{k=1}^K \Delta_k
+ \frac{\eta^2 (2+u) c L R}{2 K}
  \left( T^2 \sigma^2 K_g + 2 K_b \sigma_h^2 \right).
\end{aligned}
\]

To simplify the inequality, we solve for \(\delta\):
\[
\frac{1}{2\delta} = -\frac{\eta \sqrt{c} (K_g T + K_b)}{2K} - \frac{\eta^2 (2+u) c L K_b (1 + c_h)}{K} + \frac{\eta \sqrt{c} (K_g T + K_b)}{K},
\]
\[
\frac{1}{2\delta} = \frac{\eta \sqrt{c} (K_g T + K_b)}{2K} - \frac{\eta^2 (2+u) c L K_b (1 + c_h)}{K},
\]
\[
\delta = \frac{K}{\eta \sqrt{c} (K_g T + K_b) - 2 \eta^2 (2+u) c L K_b (1 + c_h)}.
\]
For \(\delta > 0\), the denominator must be positive:
\[
\eta \sqrt{c} (K_g T + K_b) - 2 \eta^2 (2+u) c L K_b (1 + c_h) > 0,
\]
yielding the condition:
\[
\eta < \frac{\sqrt{c} (K_g T + K_b)}{2 (2+u) c L K_b (1 + c_h)}.
\]
Substitute \(\delta\):

\[
\begin{aligned}
\mathbb{E}_{\bar z}\bigl[f(w^R)\bigr] - \mathbb{E}_{\bar z}\bigl[f(w^0)\bigr]
&\leq
-\frac{\eta \sqrt{c} (K_g T + K_b)}{2K} \sum_{r=0}^{R-1} \|\nabla f(w^r)\|^2 \\
&\quad
+ \frac{\eta^2 c (K_g T + K_b)^2 \sigma_h^2 R}{2 K \left( \eta \sqrt{c} (K_g T + K_b) - 2 \eta^2 (2+u) c L K_b (1 + c_h) \right)} \\
&\quad
+ \eta^2 (2+u)^2 L^2 T \gamma \sum_{r=0}^{R-1} (f(w^r) - f^*) \\
&\quad
+ \frac{\eta^2 (2+u)^2 L^3 T R \gamma}{2 K} \sum_{k=1}^K \Delta_k \\
&\quad
+ \frac{\eta^2 (2+u) c L R}{2 K} \left( T^2 \sigma^2 K_g + 2 K_b \sigma_h^2 \right).
\end{aligned}
\]

According to the Assumption~\ref{assumption:assumption2} we have:
\[
2\mu (f(\mathbf{w}^r) - f^*) \leq \|\nabla f(\mathbf{w}^r)\|^2, \quad \forall \mathbf{w}^r \in \mathbb{R}^d,
\]

\[
2\mu \sum_{r=0}^{R-1}(f(\mathbf{w}^r) - f^*) \leq \sum_{r=0}^{R-1}\|\nabla f(\mathbf{w}^r)\|^2, \quad \forall \mathbf{w}^r \in \mathbb{R}^d,
\]

Combine the PL inequality to the above function we get:

\[
\begin{aligned}
\frac{\eta \sqrt{c} (K_g T + K_b)}{2K} \sum_{r=0}^{R-1} \|\nabla f(w^r)\|^2
&\leq
\mathbb{E}_{\bar z}\bigl[f(w^0)\bigr] - \mathbb{E}_{\bar z}\bigl[f(w^R)\bigr] \\
&\quad
+ \frac{\eta^2 c (K_g T + K_b)^2 \sigma_h^2 R}{2 K \left( \eta \sqrt{c} (K_g T + K_b) - 2 \eta^2 (2+u) c L K_b (1 + c_h) \right)} \\
&\quad
+ \eta^2 (2+u)^2 L^2 T \gamma \sum_{r=0}^{R-1} (f(w^r) - f^*) \\
&\quad
+ \frac{\eta^2 (2+u)^2 L^3 T R \gamma}{2 K} \sum_{k=1}^K \Delta_k \\
&\quad
+ \frac{\eta^2 (2+u) c L R}{2 K} \left( T^2 \sigma^2 K_g + 2 K_b \sigma_h^2 \right).
\end{aligned}
\]
Let \( S_E = \sum_{r=0}^{R-1} \mathbb{E}[f(w^r) - f^*] \), \( D_\delta = \eta \sqrt{c} (K_g T + K_b) - 2 \eta^2 (2+u) c L K_b (1 + c_h) \). We require \( D_\delta > 0 \).

Substituting this back into the original inequality:
\[
\begin{aligned}
\mathbb{E}[f(w^R)] - \mathbb{E}[f(w^0)]
&\leq
-\frac{\eta \mu \sqrt{c} (K_g T + K_b)}{K} S_E + \eta^2 (2+u)^2 L^2 T \gamma S_E \\
&\quad
+ \frac{\eta^2 c (K_g T + K_b)^2 \sigma_h^2 R}{2 K D_\delta} \\
&\quad
+ \frac{\eta^2 (2+u)^2 L^3 T R \gamma}{2 K} \sum_{k=1}^K \Delta_k \\
&\quad
+ \frac{\eta^2 (2+u) c L R}{2 K} \left( T^2 \sigma^2 K_g + 2 K_b \sigma_h^2 \right).
\end{aligned}
\]
Collecting terms involving \( S_E \):
\[
\mathbb{E}[f(w^R)] - \mathbb{E}[f(w^0)] \leq \left( \eta^2 (2+u)^2 L^2 T \gamma - \frac{\eta \mu \sqrt{c} (K_g T + K_b)}{K} \right) S_E + \text{other terms}.
\]
Moving \( S_E \) to the left side:
\begin{align}
\Bigl(\frac{\eta\mu\sqrt{c}\,(K_gT+K_b)}{K}
       -\;\eta^2(2+u)^2L^2T\gamma\Bigr)\,S_E
&\le \mathbb{E}[f(w^0)] - \mathbb{E}[f(w^R)] \notag\\
&\quad
+\,\frac{\eta^2c\,(K_gT+K_b)^2\sigma_h^2\,R}
      {2K\,D_\delta} 
+ \frac{\eta^2(2+u)^2L^3T\,R\gamma}{2K}\sum_{k=1}^K\Delta_k
\notag\\
&\quad
+\,\frac{\eta^2(2+u)cL\,R}{2K}
    \bigl(T^2\sigma^2K_g + 2K_b\sigma_h^2\bigr).
\label{eq:shortened}
\end{align}
Since \(\mathbb{E}[f(w^R)] \geq f^*\) (typically \( f^* \) is the minimum), we have \(\mathbb{E}[f(w^0)] - \mathbb{E}[f(w^R)] \leq \mathbb{E}[f(w^0) - f^*]\). Let \( f_0^* = \mathbb{E}[f(w^0) - f^*] \) (the initial expected suboptimality). Let the coefficient of \( S_E \) be \( C'_S = \frac{\eta \mu \sqrt{c} (K_g T + K_b)}{K} - \eta^2 (2+u)^2 L^2 T \gamma \). To ensure \( C'_S > 0 \), we need \(\eta\) sufficiently small such that \(\eta < \frac{\mu \sqrt{c} (K_g T + K_b)}{K (2+u)^2 L^2 T \gamma}\). Then:
\[
\begin{aligned}
C'_S S_E
&\le f_0^* 
   + \frac{\eta^2 c (K_g T + K_b)^2 \,\sigma_h^2\,R}
          {2\,K\,D_\delta}
   + \frac{\eta^2 (2+u)^2\,L^3\,T\,R\,\gamma}{2\,K}
     \sum_{k=1}^K \Delta_k
\\[6pt]
&\quad
   + \frac{\eta^2 (2+u)\,c\,L\,R}{2\,K}
     \Bigl(T^2\,\sigma^2\,K_g + 2\,K_b\,\sigma_h^2\Bigr).
\end{aligned}
\]
Our goal is \(\frac{1}{R} S_E = \frac{1}{R} \sum_{r=0}^{R-1} \mathbb{E}[f(w^{r}) - f^{*}]\). Dividing both sides by \( R \):
\begin{align*}
C'_S\,\frac{1}{R}\sum_{r=0}^{R-1}\mathbb{E}[f(w^r)-f^*]
&\le 
\frac{f_0^*}{R}
+ \frac{\eta^2\,c\,(K_gT+K_b)^2\,\sigma_h^2}{2K\,D_\delta}
\\
&\quad
+ \frac{\eta^2\,(2+u)^2\,L^3\,T\,\gamma}{2K}\sum_{k=1}^K\Delta_k
\\
&\quad
+ \frac{\eta^2\,(2+u)\,c\,L}{2K}
    \bigl(T^2\sigma^2K_g + 2K_b\sigma_h^2\bigr).
\end{align*}
Finally, dividing both sides by \( C'_S \) (assuming \( C'_S > 0 \)):
\[
\begin{aligned}
\frac{1}{R} \sum_{r=0}^{R-1} \mathbb{E}_{\bar z}[f(w^{r}) - f^{*}] &\leq \frac{1}{C'_S} \left[ \frac{f_0^*}{R} + \frac{\eta^2 c (K_g T + K_b)^2 \sigma_h^2}{2 K \left( \eta \sqrt{c} (K_g T + K_b) - 2 \eta^2 (2+u) c L K_b (1 + c_h) \right)} \right. \\
&\quad \left. + \frac{\eta^2 (2+u)^2 L^3 T \gamma}{2 K} \sum_{k=1}^K \Delta_k + \frac{\eta^2 (2+u) c L}{2 K} \left( T^2 \sigma^2 K_g + 2 K_b \sigma_h^2 \right) \right],
\end{aligned}
\]
where 
\[
C'_S \;=\; \frac{\eta\,\mu\sqrt{c}\,(K_gT+K_b)}{K}
           \;-\;\eta^2(2+u)^2L^2T\,\gamma,
\qquad
D_\delta \;=\;\eta\sqrt{c}\,(K_gT+K_b)
           -2\eta^2(2+u)cL\,K_b(1+c_h).
\]
To ensure both \(C'_S>0\) and \(D_\delta>0\), we require
\[
\eta<\underbrace{\frac{\sqrt{c}(K_gT+K_b)}
                   {2(2+u)cL\,K_b(1+c_h)}}_{=\bar\eta_\delta},
\qquad
\eta<\underbrace{\frac{\mu\sqrt{c}(K_gT+K_b)}
                   {K(2+u)^2L^2T\,\gamma}}_{=\bar\eta_S}.
\]
Let 
\[
\eta_{\max}=\min\{\bar\eta_\delta,\;\bar\eta_S\},
\qquad
\theta\in\bigl(0,\tfrac12\bigr].
\]
Choosing \(\theta=\tfrac12\) gives
\[
\eta
=\tfrac12\,\eta_{\max}
=\frac{\mu\sqrt{c}\,(K_gT+K_b)}
     {2\,K\,(2+u)^2L^2T\,\gamma}.
\]

We select 
\[
\eta = \frac{\mu \sqrt{c} (K_g T + K_b)}{2 K (2+u)^2 L^2 T \gamma}
\]
And from previous client convergence conclusion, we pick a constant local learning rate
\[
\eta_{\rm client} = \frac{c}{\alpha} = \frac{1}{L\,(u+2)} < \frac{2c}{\alpha}
\]
Substituting the learning rate $\eta=\min\!\Biggl\{
\frac{1}{L\,(u + 2)},\;
\frac{\mu\,\sqrt{c}\,\bigl(K_g\,T + K_b\bigr)}
     {2\,K\,(2+u)^2\,L^2\,T\,\gamma}
\Biggr\}$, since $\eta$ is a small value, we neglect $\eta^2$.
\begin{align*}
\frac{1}{R}\sum_{r=0}^{R-1}\mathbb{E}_{\bar z}\bigl[f(w^r)-f^*\bigr]
&\le
\frac{4K^2(2+u)^2L^2T\gamma\,\mathbb{E}[f(w^0) - f^*]}
     {\mu^2\,c\,(K_gT+K_b)^2\,R}
\;+\;\frac{\sigma_h^2}{2}
\;+\;\frac{(2+u)\,L}{4K}\sum_{k=1}^K\Delta_k
\\[6pt]
&\quad
+\;\frac{c}{4K(2+u)L\,T\,\gamma}
\Bigl(T^2\sigma^2K_g + 2K_b\sigma_h^2\Bigr).
\end{align*}

\begin{equation}\label{eq:vpcs_convergence_bigO}
\begin{aligned}
\frac{1}{R}\sum_{r=0}^{R-1}\mathbb{E}_{\bar z}\bigl[f(w^r)-f^*\bigr]
&\le
O\!\biggl(\frac{K^2\,(2+u)^2\,\gamma\,T}{c\,(K_gT+K_b)^2\,R}\biggr) \\[4pt]
&\quad +\,O\!\biggl(\frac{1+u}{K}(\sum_{k=1}^{K_g}\Delta_{kg} + \sum_{k=1}^{K_b}\Delta_{kb})\biggr) \\[4pt]
&\quad +\,O\!\biggl(\frac{c\,T\,K_g}{K(1+u)\,\gamma}\biggr) \\[4pt]
&\quad +\,O\!\biggl(\frac{c\,K_b\,\sigma_h^2}{K(1+u)\,T\,\gamma}\biggr)
+O\!\bigl(1\bigr).
\end{aligned}
\end{equation}
\end{proof}

Define the error upper‐bounds for \textsc{Meerkat-vp} and the baseline \textsc{Meerkat} as follows:
\begin{align*}
E_{\mathrm{\textsc{Meerkat-vp}}}
&=
\underbrace{\frac{4K^2(2+u)^2L^2T\,\gamma}
      {\mu^2\,c\,(K_gT+K_b)^2}
      \,\frac{\mathbb{E}[f(w^0)-f^*]}{R}}_{\text{(I) Transient term}}
+
\underbrace{\biggl[
      \frac{\sigma_h^2}{2}
      +\frac{(2+u)L}{4K}\sum_{k=1}^K\Delta_k
      +\frac{c\bigl(T^2\sigma^2K_g+2K_b\sigma_h^2\bigr)}
            {4K(2+u)L\,T\,\gamma}
\biggr]}_{\text{(II) Steady‐state term}},
\\[6pt]
E_{\mathrm{\textsc{Meerkat}}}
&=
\underbrace{\frac{4L^2(2+u)^2}
      {\mu^2\,c\,(1+\sqrt{c_h})^2\,T}
      \,\frac{\mathbb{E}[f(w^0)-f^*]}{R}}_{\text{(I') Transient term}}
+
\underbrace{\biggl[
      \frac{\sigma_h^2}{\mu\,(1+\sqrt{c_h})^2}
      +\frac{L}{K}\sum_{k=1}^K\Delta_k
      +\frac{T\,c\,\sigma^2}{2L\,(2+u)}
\biggr]}_{\text{(II') Steady‐state term}}.
\end{align*}

\begin{itemize}[leftmargin=*]
  \item \textbf{Transient term ratio:}
    \[
      \frac{(I)}{(I')}
      \;\approx\;
      \gamma\,(1+\sqrt{c_h})^2
      \;<\;1,
      \quad
      \text{and as }c_h\to1,\;\gamma(1+\sqrt{c_h})^2\to0.
    \]
  \item \textbf{Noise term ratio:}
    \[
      \frac{\sigma_h^2/2}{\sigma_h^2/(\mu\,(1+\sqrt{c_h})^2)} = \frac{\mu\,(1+\sqrt{c_h})^2}{2}, \quad \text{which is } < 1 \text{ when } \mu\,(1+\sqrt{c_h})^2 < 2.
    \]
    Empirically $\mu < 1$, thus $\mu\,(1+\sqrt{c_h})^2 < 2$ is True. Additionally, VPCS includes an extra term \(\frac{c K_b \sigma_h^2}{2K (2 + u) L T \gamma}\), which decays as \(\frac{1}{T}\) and becomes negligible for large \(T\).
  \item \textbf{Heterogeneity and variance terms:}
    \[
      \frac{(2+u)L}{4K}\sum_{k=1}^K\Delta_k
      \;<\;
      \frac{L}{K}\sum_{k=1}^K\Delta_k,
      \quad
      \text{and the extra variance term decays as }1/K.
    \]
\end{itemize}

Therefore, under the same \(T\) and \(R\),  $E_{\mathrm{\textsc{Meerkat-vp}}}< E_{\mathrm{\textsc{Meerkat}}}$ and this gap widens as data heterogeneity \(c_h\) increases.
\subsection*{Remarks}
The analysis of the upper bound in Equation~\ref{eq:avg-loss-bound-final} reveals how the local training step \(T\), density level \(u\), and communication rounds \(R\) collectively influence the optimization dynamics through a balance of convergence rate, bias–variance trade-offs, and steady-state error control:

\begin{itemize}[nosep,leftmargin=*]
    \item \textbf{Impact of Local Update Steps $T$:}
    A smaller \(T\) amplifies the term \(\mathcal{O}\left( \frac{(2+u)^2}{TR} \cdot \mathbb{E}[f(w^0) - f(w^R)] \right)\), increasing the average optimality gap after \(R\) communication rounds when \(R\) is fixed. However, this effect can be mitigated by increasing \(R\), as the scaling factor \(\frac{1}{R}\) reduces the term's impact. Conversely, reducing \(T\) diminishes the variance term \(\mathcal{O}\left( \frac{T}{2+u} \right)\), leading to a smaller steady-state error. Thus, a smaller \(T\) may prolong the transient phase but ultimately achieves a tighter optimality gap relative to \(f^*\) after sufficient rounds.

    \item \textbf{Density Level \(u\).}  
    Reducing \(u\) (i.e., increasing sparsity) quadratically benefits the transient term, yet it also inflates the steady‑state term through the denominator \(2+u\).  Choosing \(u\) therefore amounts to balancing communication savings against the plateau error; aggressive sparsification should be coupled with smaller \(T\) to avoid performance degradation.

    \item \textbf{\textsc{Meerkat-vp} Client Selection Strategy:}
    By early-stopping extreme data-imbalance clients with a single local training step, \textsc{Meerkat-vp} effectively reduces Non-IID drift in zeroth-order federated llm fine-tuning. This strategy
    lowers the coefficient of the transient term and further reduces heterogeneity‐ and
    variance‐induced steady‐state error. Under fixed \(T\) and \(R\), these effects yield
    strictly faster convergence and a tighter optimality gap in Non-IID settings.
\end{itemize}

\noindent These conclusions illustrate how tuning \(T\), \(R\), \(u\), and the \textsc{Meerkat-vp} client selection
strategy can optimize performance in federated, sparse, and Non-IID learning scenarios.

\subsection{Empirical analysis of the GradIP Phenomenon}\label{sec:empirical_analysis_of_gradip_phenomenon_proof}
By Lemma~\ref{lem:A2}, the masked sparse zeroth‑order (ZO) surrogate gradient is an \emph{unbiased} estimator of the masked first‑order gradient.  
Building on this fact, we define the vector \( g_c(w; x, y) \) is obtained by computing the gradient of the cross-entropy loss for a single sample with respect to a small subset of parameters selected by a mask.

From logits to Softmax Probabilities we have:
\begin{itemize}[nosep,leftmargin=*]
    \item The model's final layer outputs a \textit{logit} for each class:
    \[
    h(x; w) = \bigl( h_1, \dots, h_C \bigr) \in \mathbb{R}^{C}.
    \]
    \item The softmax probabilities are given by:
    \[
    p_j(x; w) = \frac{e^{h_j}}{\sum_{r=1}^{C} e^{h_r}}.
    \]
\end{itemize}

The cross-entropy loss for a single sample is:
\[
\ell(w; x, y) = -\log p_{y}(x; w), \quad \text{where } y \in \{1, \dots, C\}.
\]

For each logit \( h_j \), the partial derivative is:
\[
\frac{\partial \ell}{\partial h_j} = p_j - \mathbf{1}_{\{y = j\}} = p_j - (e_y)_j,
\]
where \( e_y \) is the one-hot vector with 1 in the \( y \)-th component.

Since we are only interested in the sensitive parameters selected by the mask \( m \), the gradient with respect to the parameters can be written as:
\[
\begin{aligned}
g_c(w; x, y) &= \nabla_{w_m} \ell(w; x, y) \\
&= \sum_{j=1}^{C} \frac{\partial \ell}{\partial h_j} \; \nabla_{w_m} h_j(x; w) \\
&= \Bigl( p(x; w) - e_y \Bigr)^{\!\top} \nabla_{w_m} h(x; w).
\end{aligned}
\]
Here:
\begin{itemize}[nosep,leftmargin=*]
    \item \( \nabla_{w_m} h_j(x; w) \) is the gradient/Jacobian of the logit \( h_j \) with respect to the masked parameter \( w_m \).
    \item By collecting the coefficients \( p_j - \mathbf{1}_{y = j} \) into a vector, we obtain the compact form:
    \[
    g_c(w; x, y) = (p - e_y)^{\!\top} \nabla_{w_m} h(x; w).
    \]
\end{itemize}

In our existing local client convergence inequality and from the assumption~\ref{assumption:assumption4}, we can empirically write the key constant estimator variance:

\[
\sigma_k^2 = \frac{1}{d} \operatorname{Var}_{(x,y) \sim D_k} [g_c(w; x, y)].
\]

We write \(g_c\) in matrix form:
Define:
\[
\mathbf{J}(x; w) = \nabla_{w_m} h(x; w) \in \mathbb{R}^{d_m \times C}, \quad \mathbf{a}(x, y; w) = p(x; w) - \mathbf{e}_y \in \mathbb{R}^{C}.
\]
Thus:
\[
g_c(w; x, y) = \mathbf{J}^{\!\top}(x; w) \, \mathbf{a}(x, y; w) \in \mathbb{R}^{d_m}.
\]
We substitute this equation to the above estimator variance:
\[
\sigma_k^2 = \frac{1}{d_m} \underbrace{ \mathbb{E}_{(x,y)} \bigl\| g(w; x, y) - \nabla f_k(w) \bigr\|^2 }_{\text{total variance}} = \frac{1}{d_m} \operatorname{tr} \left( \mathbf{J}^{\!\top} \, \underbrace{ \operatorname{Cov}_{(x,y)} \bigl[ \mathbf{a}(x, y; w) \bigr] }_{\Sigma_{a}} \, \mathbf{J} \right).
\tag{1}
\]
Note:
\begin{itemize}[nosep,leftmargin=*]
    \item \(\Sigma_{a} \in \mathbb{R}^{C \times C}\) is determined solely by the \textbf{label distribution and prediction probabilities}.
    \item \(\mathbf{J}\) reflects the network structure and influences only a similarity coefficient.
\end{itemize}

\textbf{Analysis of Extreme Non-IID (Single Label \(y^\dagger\))}:

\begin{itemize}[nosep,leftmargin=*]
    \item The label is fixed, so \(\mathbf{1}_{y = j}\) is constant.
    \item If the model is mostly correct: \(p \approx \mathbf{e}_{y^\dagger}\), then \(\mathbf{a}(x, y; w) \approx \mathbf{0}\), yielding:
    \[
    \Sigma_{a} \approx \mathbf{0} \quad \Longrightarrow \quad \sigma_{\text{non}}^2 \approx \frac{1}{d_m} \operatorname{tr}(\mathbf{0}) = 0.
    \]
\end{itemize}

\textbf{Analysis of Approximate IID (Balanced Multi-Label)} 

\begin{itemize}[nosep,leftmargin=*]
    \item The label \(y\) varies across \(\{1, \dots, C\}\).
    \item Even as the loss decreases, \(p_j\) differs across classes. The covariance is:
    \[
    (\Sigma_{a})_{rs} = \mathbb{E} \left[ (p_r - \mathbf{1}_{y = r})(p_s - \mathbf{1}_{y = s}) \right] - \underbrace{ \left( \mathbb{E}[p_r - \mathbf{1}_{y = r}] \right) }_{=0} \underbrace{ \left( \mathbb{E}[p_s - \mathbf{1}_{y = s}] \right) }_{=0}.
    \]
    This matrix has diagonal elements \(\mathbb{E} \left[ (p_r - \mathbf{1}_{y = r})^2 \right] > 0\), making \(\Sigma_{a}\) positive definite or semi-definite but non-zero. Thus:
    \[
    \sigma_{\text{iid}}^2 = \frac{1}{d_m} \operatorname{tr} \left( \mathbf{J}^{\!\top} \Sigma_{a} \mathbf{J} \right) > 0.
    \]
\end{itemize}
Our local convergence bound is:

\[
\frac{1}{T} \sum_{t=0}^{T-1} \mathbb{E} \| \nabla f_k(w_k^t) \|^2 \leq O \left( \frac{1}{T} \right) + \sigma_k^2,
\]

which indicates that in the steady state, the upper bound of the gradient norm is determined by \( \sigma_k^2 \). Therefore,

\[
\sigma_{\text{iid}}^2 \gg \sigma_{\text{non-iid}}^2 \approx 0 \quad \Longrightarrow \quad
\begin{cases}
\text{IID clients: Gradient Norm oscillates significantly;} \\
\text{Non-IID clients: Gradient Norm decreases monotonically and approaches 0.}
\end{cases}
\]

\subsection*{Remarks}

In summary, by substituting the explicit form of the cross-entropy gradient into our sparse ZO convergence formula, we can empirically explain that due to the variance differences caused by label distributions, the Gradient Norms of IID clients maintains significant fluctuations, while that of extremely Non-IID clients rapidly decays and converges to zero.

\begin{algorithm}
    \caption{\textsc{Meerkat}: Sparse Zeroth-Order Optimization for Federated LLM Fine-Tuning}
\label{alg:fl-szogeneral}
\begin{algorithmic}
\STATE {\bfseries Input:} pre-trained weight $\mathbf{w}_0$, 
sparse mask $\mathbf{m}$,
learning rate $\eta$, 
perturbation scale $\epsilon$, 
number of rounds $R$, 
total number of \textit{clients} $K$
number of local steps $T$
\STATE \textit{Server} initiate seed list $\{s_1^1, \cdots, s^T_1\}$
\FOR{Round $r = 1$ to $R$}
   \STATE \textbf{Step 1. Local ZO update.}
  \FOR{each \textit{client} $k$ \textbf{in parallel}}
    \STATE Download model from \textit{server}: $\mathbf{w}_k \gets \mathbf{w}_{r-1}$
    \STATE Download seed list $\{s_r^1, \cdots, s^T_r\}$ from \textit{server}
    \FOR{local step $t = 1$ to $T$}
        \STATE Initialize $\mathbf{z}^{t}_k$ with seed $s_r^t$.
        \STATE Sample a batch $\mathcal{B}$ on \textit{client } dataset.
      \STATE $\tilde{\mathbf{w}}_{\text{k}}^{t} \gets \mathbf{w}_k^{t} + \epsilon \cdot( \mathbf{z}^{t}_k \odot \mathbf{m})$
      \STATE Compute loss $f_{+} \gets f(\tilde{\mathbf{w}}_{\text{k}}^{t}; \mathcal{B})$
      \STATE $\tilde{\mathbf{w}}_{\text{k}}^{t} \gets \mathbf{w}_k^{t} - 2\epsilon \cdot (\mathbf{z}^{t}_k \odot \mathbf{m} )$
      \STATE Compute loss: $f_{-} \gets f(\tilde{\mathbf{w}}_{\text{k}}^{t}; \mathcal{B})$
      \STATE Compute projected gradient: $$g^{t}_k \gets(f_{+} - f_{-})/2\epsilon$$
      \STATE Update \textit{client} model:
      \[
        \hat{\nabla} f_k^t \gets g^{t}_k\cdot (\mathbf{z}^t_k \odot \mathbf{m})
      \]
      \[
        \mathbf{w}^{t+1}_k \gets \mathbf{w}^{t}_k - \eta \hat{\nabla} f_k^t
      \]
    
    \ENDFOR
    \STATE Send projected gradients $\{g^{1}_k, g^{2}_k, \dots, g^{T}_k\}$ to \textit{server}.
    
  \ENDFOR

  \STATE \textbf{Step 2. \textit{Server} recover each \textit{client's} update with \textit{virtual path}.}
  \FOR{$k=1$ to $K$}
    \FOR{local step $t = 1$ to $T$}
        \STATE Generate $\mathbf{z}^{t}_k$ with seed $s_r^t$.
        \STATE Perform \textit{virtual path}:
        \[\hat{\nabla} f_k^t = g_k^t \cdot (\mathbf{z}_k^t \odot \mathbf{m})\]
        \[
          \mathbf{w}^{t+1}_k \gets \mathbf{w}^t_k - \eta \hat{\nabla} f_k^t
        \]
        \STATE Store recover \textit{client} model parameters $\mathbf{w}_k^{T}$
    \ENDFOR
  \ENDFOR

  \STATE \textbf{Step 3. \textit{Server} Aggregate reconstructed sparse model update.}
  \[
    \mathbf{w}_r \gets \frac{1}{K} \sum_{k=1}^{K} \mathbf{w}_k^T
  \]
  \STATE Generate new seed list $\{s_{r+1}^1, \cdots, s^T_{r+1}\}$.

\ENDFOR
\STATE {\bfseries Output:} $\mathbf{w}_R$
\end{algorithmic}
\end{algorithm}

\begin{figure*}[!h]
    \centering
    \includegraphics[width=\textwidth]{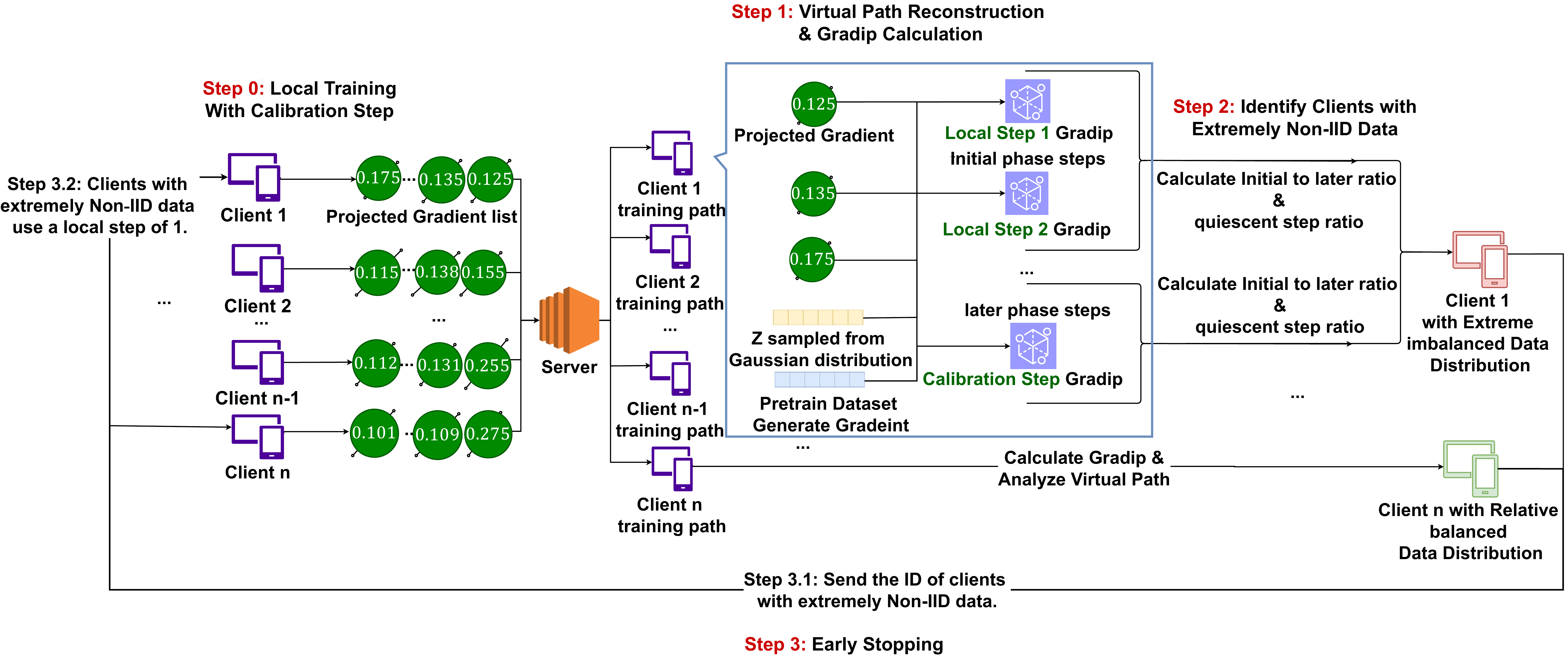}
    \caption{\textsc{Meerkat-vp}: Each client locally trains with a prescribed statistic step, yielding a sequence of projected gradients. The server leverages a randomly sampled vector $z_k^t$ from the Gaussian distribution $\mathcal{N}(\mathbf{0}, \mathbf{I}_d)$ to reconstruct $\nabla f_k^t$, and then computes GradIP (see Definition~\ref{def:gradip}) at every local training step. By analyzing the resulting GradIP values across all clients, the system distinguishes those clients with extremely Non-IID data from those that are relatively balanced. For the parameters 
    \textbf{later phase steps} ,
    \textbf{initial phase steps},
    \textbf{quiescent step ratio},
    and 
    \textbf{initial to later ratio},
    please refer to Table~\ref{tab:meerkat_vp_parameter_notation} in Appendix~\ref{sec:additional_experiment_settings}}
    \label{fig:vpoverview}
\end{figure*}

\begin{algorithm}
\caption{\textsc{Meerkat} with high frequency server-client synchronization}
\label{alg:fl-szofreq}
\begin{algorithmic}
   \STATE {\bfseries Input:} Seed $s$ and projected gradients $g_k^t$ from all clients, 
global model $\mathbf{m}$, 
learning rate $\eta$, 
number of clients $K$, sparse mask $\mathbf{m}$
\STATE \textbf{Aggregate projected gradients from all clients with same seed:}
\[
g \gets \frac{1}{K} \sum_{k=1}^{K} g_k
\]

\STATE \textbf{Calculate Zeroth-Order Gradients:}
\[
    \hat\nabla f \gets g\cdot (\mathbf{z} \odot \mathbf{m})
\]

\STATE \textbf{Update global model parameters:}
\[
\mathbf{w_{r+1}} \gets \mathbf{w_r} - \eta \, (\hat\nabla f \odot \mathbf{m})
\]
\STATE \textbf{Generate new seed $s\_new$}

\STATE {\bfseries Output:} \textbf{Send aggregated global projected gradients $g$ and seed $s\_new$ to all clients.}
\end{algorithmic}
\end{algorithm}

\section{More Experimental Details}\label{sec:experiment_details}

\subsection{Additional Experimental Settings}\label{sec:additional_experiment_settings}\label{sec:experiment_settings}

\noindent \textbf{Testbed.} All experiments are run on servers with the following configurations:
RTX A6000 Setup: Ubuntu 18.04.6 LTS with 2 NVIDIA RTX A6000 GPUs (each with 48GB GPU memory).
GH200 Setup: Ubuntu 20.04 with 1 NVIDIA GH200 GPU (97GB GPU memory).
A100 Setup: Ubuntu 22.04 with 1 NVIDIA A100 GPU (40GB GPU memory).

\noindent \textbf{Dataset.}
We conducted experiments using datasets from the GLUE and SuperGLUE benchmarks, including SST2, AgNews, Yelp, BoolQ, RTE, WSC, and WiC. To create IID clients data, we shuffle the entire dataset and evenly divide it among the clients. To create Non-IID clients data, we split the data using a Dirichlet distribution. For all tasks, the Dirichlet \(\alpha\) parameter is set to $0.5$ to control the degree of data heterogeneity.

\noindent \textbf{Evaluation metric.} 
In our experiments, test accuracy is used as the primary evaluation metric. Accuracy is computed as the proportion of correctly predicted labels across all evaluation samples., Additionally, we incorporate the GradIP score (see  Definition~\ref{def:gradip}) to analyze further the dynamics of local model training under IID and Non-IID client data settings. GradIP provides a metric to measure the quality of client training trajectories, particularly in heterogeneous data distributions.

\noindent \textbf{Notations.} We present the parameters definition used in \textsc{Meerkat-vp} in Table~\ref{tab:meerkat_vp_parameter_notation}.

\begin{table*}[htbp]
\centering
\caption{\textsc{Meerkat-vp} Parameters Notation}
\label{tab:meerkat_vp_parameter_notation}
\begin{tabular}{l l}
\toprule
\textbf{Term} & \textbf{Explanation} \\
\midrule
\textbf{calibration steps} $T_{\textsf{cali}}$ & Number of steps each client performs to measure GradIP. \\
\textbf{initial phase steps} $T_{\textsf{init}}$ & Number of earliest local steps used to measure the early-phase GradIP. \\
\textbf{later phase steps} $T_{\textsf{later}}$  & Number of latest local steps used to observe the late-phase GradIP. \\
\textbf{convergence threshold} $\sigma$ & Threshold indicating when GradIP is effectively zero. \\
\textbf{quiescent step ratio} $\rho_{\textsf{quie}}$ & Fraction of later phase where GradIP stays below threshold \\
\textbf{Initial to later ratio} $\rho_{\textsf{later}}$ & Ratio of average GradIP in the initial phase to that in the later phase. \\
\bottomrule
\end{tabular}
\end{table*}

\noindent \textbf{Hyper-parameters.} We use the following hyper-parameters in our experiments; see Table~\ref{tab:hyper_params}
\begin{table*}[htbp]
\centering
\caption{Hyper-parameters used in our experiments.}
\label{tab:hyper_params}
\begin{tabular}{l l}
\toprule
\textbf{Parameter} & \textbf{Value} \\
\midrule
\textsc{Meerkat} learning rate    & [2e-4, 2e-8]  \\
\textsc{Meerkat-vp} learning rate    & [2e-4, 2e-8]  \\
LoRA-FedZO learning rate & [2e-4, 2e-8] \\
Full-FedZO learning rate & [2e-4, 2e-8] \\
Batch size               & 16 \\
Dirichlet alpha          & 0.5, 0.3, 0.1 \\
LoRA rank                & 16 \\
LoRA alpha               & 16 \\
initial phase steps             & 20 \\
later phase steps              & 20 \\
convergence threshold            & 1 \\
quiescent step ratio              & [0.4, 0.5, 0.7] \\
Initial to later ratio      & [1.5, 2, 5, 10, 15] \\
calibration steps          & 100 \\
Total clients            & 10 \\
\bottomrule
\end{tabular}
\end{table*}

\subsection{Additional Experiment Results}\label{sec:additional_experiment_results}

In this section, we present additional experimental results to compare \textsc{Meerkat}, \textsc{Meerkat-vp}, Full-FedZO, and LoRA-FedZO under various settings. The results include five tables and three figures, providing a detailed evaluation of performance across different models, datasets and experiment settings. Table~\ref{tab:meerkat_vp_parameter_notation} provides a description of the parameters used in \textsc{Meerkat-vp}, and Table~\ref{tab:hyper_params} lists the experiment parameters used in this experiment.Table~\ref{tab:comparison_full_vs_sparse_iid} compares \textsc{Meerkat} and Full-FedZO on multiple tasks at the same communication frequency for Llama-3.2-1B, Qwen2-1.5B, and Gemma-2-2b models. Table~\ref{tab:comparison_sparse_vs_virtual_path} presents results in a Non-IID client data scenario, comparing \textsc{Meerkat-vp} and \textsc{Meerkat} under the same communication frequency and sparsity density, and demonstrating \textsc{Meerkat-vp} improved performance.
Table~\ref{tab:outlier_percentage} investigates the robustness of \textsc{Meerkat} by evaluating test accuracy with local step 1 across different sparsity densities. Table~\ref{tab:all_in_one} compares \textsc{Meerkat}, Full-FedZO and LoRA-FedZO under high communication frequency across IID and Non-IID client data settings. Figure~\ref{fig:gradip_main_fig} and Figure~\ref{fig:gradip_real_noniid_agnews_boolq} further illustrate the phenomenon of GradIP under IID and Non-IID client data settings. 

\begin{table*}[tbp]
\caption{Performance comparison of \textsc{Meerkat} and Full-FedZO on tasks SST-2, AgNews, Yelp, BoolQ, RTE, WSC, WIC under an IID client data setting. 
``Acc'' is the average test accuracy across tasks. Bold numbers indicate the highest value in each row.}
\label{tab:comparison_full_vs_sparse_iid}
\vspace{5pt}
\centering
\resizebox{0.8\textwidth}{!}{
\begin{tabular}{llc cccccccc}
\toprule
& \textbf{Methods} & \textbf{Local Step} & \textbf{SST-2} & \textbf{AgNews} & \textbf{Yelp} & \textbf{BoolQ} & \textbf{RTE} & \textbf{WSC} & \textbf{WIC} & \textbf{Acc} \\
\midrule
\multirow{8}{*}{\textbf{LLaMA-3.2-1B}}
&  Full-FedZO & 10   & 0.913 & 0.700 & 0.938 & 0.646 & 0.537 & 0.634 & 0.540 & 0.701 \\
& \textsc{Meerkat} & 10 & \textbf{0.925} & \textbf{0.881} & \textbf{0.964} & \textbf{0.751} & \textbf{0.684} & 0.634 & \textbf{0.648} & \textbf{0.784} \\
\cline{2-11}
&  Full-FedZO & 30 & 0.913 & 0.700 & 0.935 & 0.643 & 0.542 & 0.634 & 0.528 & 0.699 \\
& \textsc{Meerkat} & 30 & \textbf{0.919} & \textbf{0.865} & \textbf{0.967} & \textbf{0.729} & \textbf{0.644} & \textbf{0.663} & \textbf{0.617} & \textbf{0.772} \\
\cline{2-11}
&  Full-FedZO & 50 & 0.913 & 0.698 & 0.939 & 0.641 & 0.520 & 0.634 & 0.539 & 0.698 \\
& \textsc{Meerkat} & 50 & \textbf{0.920} & \textbf{0.871} & \textbf{0.966} & \textbf{0.734} & \textbf{0.648} & \textbf{0.653} & \textbf{0.614} & \textbf{0.772} \\
\cline{2-11}
&  Full-FedZO & 100 & 0.903 & 0.705 & 0.934 & 0.656 & 0.537 & 0.634 & 0.537 & 0.701 \\
& \textsc{Meerkat} & 100 & \textbf{0.913} & \textbf{0.842} & \textbf{0.945} & \textbf{0.722} & \textbf{0.573} & 0.634 & \textbf{0.595} & \textbf{0.746} \\
\midrule
\multirow{8}{*}{\textbf{Qwen2-1.5b}}
&  Full-FedZO & 10 & 0.891 & 0.701 & 0.931 & 0.696 & 0.800 & 0.682 & 0.579 & 0.754 \\
& \textsc{Meerkat} & 10 & \textbf{0.944} & \textbf{0.889} & \textbf{0.942} & \textbf{0.788} & \textbf{0.817} & \textbf{0.700} & \textbf{0.656} & \textbf{0.819} \\
\cline{2-11}
&  Full-FedZO & 30 & 0.902 & 0.702 & 0.930 & 0.709 & 0.817 & 0.663 & 0.583 & 0.758 \\
& \textsc{Meerkat} & 30 & \textbf{0.942} & \textbf{0.895} & \textbf{0.940} & \textbf{0.786} & \textbf{0.840} & \textbf{0.710} & \textbf{0.659} & \textbf{0.825} \\
\cline{2-11}
&  Full-FedZO & 50 & 0.902 & 0.705 & 0.929 & 0.701 & 0.808 & \textbf{0.663} & 0.590 & 0.757  \\
& \textsc{Meerkat} & 50 & \textbf{0.942} & \textbf{0.885} & \textbf{0.934} & \textbf{0.784} & \textbf{0.840} & 0.634 & \textbf{0.637} & \textbf{0.808}\\
\cline{2-11}
&  Full-FedZO & 100  & 0.899 & 0.714 & 0.928 & 0.705 & \textbf{0.831} & \textbf{0.682} & 0.594 & 0.765 \\
& \textsc{Meerkat} & 100 & \textbf{0.946} & \textbf{0.886} & \textbf{0.930} & \textbf{0.776} & 0.804 & 0.653 & \textbf{0.653} & \textbf{0.807} \\
\midrule
\multirow{8}{*}{\textbf{Gemma2-2b}}
&  Full-FedZO & 10  & 0.87 & 0.732 & 0.944 & 0.717 & 0.564 & 0.634 & 0.592 & 0.723 \\
& \textsc{Meerkat} & 10 & \textbf{0.943} & \textbf{0.892} & \textbf{0.97} & \textbf{0.817} & \textbf{0.724} & \textbf{0.653} & \textbf{0.636} & \textbf{0.805}\\
\cline{2-11}
&  Full-FedZO & 30  & 0.91 & 0.81 & 0.942 & 0.73 & 0.56 & 0.644 & 0.578 & 0.739 \\
& \textsc{Meerkat} & 30 & \textbf{0.943} & \textbf{0.887} & \textbf{0.973} & \textbf{0.812} & \textbf{0.617} & \textbf{0.663} & \textbf{0.608} & \textbf{0.786} \\
\cline{2-11}
&  Full-FedZO & 50     & 0.911 & 0.812 & 0.942 & 0.735 & 0.551 & 0.634 & 0.572 & 0.737 \\
& \textsc{Meerkat} & 50 & \textbf{0.94} & \textbf{0.873} & \textbf{0.964} & \textbf{0.812} & \textbf{0.604} & 0.634 & \textbf{0.617} & \textbf{0.778} \\
\cline{2-11}
&  Full-FedZO & 100    & 0.917 & 0.83 & 0.936 & 0.728 & 0.56 & \textbf{0.644} & 0.59 & 0.744 \\
& \textsc{Meerkat} & 100 & \textbf{0.949} & \textbf{0.87} & \textbf{0.954} & \textbf{0.815} & \textbf{0.568} & 0.634 & \textbf{0.592} & \textbf{0.769} \\
\bottomrule
\end{tabular}
}
\end{table*}

\begin{table*}[tbp]
\centering
\captionsetup{width=\textwidth, justification=centering}
\caption{Comparison of \textsc{Meerkat-vp} and \textsc{Meerkat} under Non-IID client data setting,
with the same local step and sparsity. 
Tasks include SST-2, AgNews, Yelp, BoolQ, RTE, WSC, and WIC. “Acc” indicates the average test accuracy across all tasks. Bold numbers highlight the best result in each row.}
\label{tab:comparison_sparse_vs_virtual_path}
\vspace{5pt}
\resizebox{0.8\textwidth}{!}{
\begin{tabular}{lcc cccccccc}
\toprule
& \textbf{Methods} & \textbf{Local Step} & \textbf{SST-2} & \textbf{AgNews} & \textbf{Yelp} & \textbf{BoolQ} & \textbf{RTE} & \textbf{WSC} & \textbf{WIC} & \textbf{Acc} \\
\midrule
\multirow{8}{*}{\textbf{LLaMA-3.2-1B}}
& \textsc{Meerkat-vp} & 10 & \textbf{0.922} & 0.864 & 0.962 & \textbf{0.713} & \textbf{0.617} & 0.644 & 0.625 & \textbf{0.764} \\
& \textsc{Meerkat} &10 & 0.916 & \textbf{0.872} & \textbf{0.964} & 0.695 & 0.600 & \textbf{0.653} & \textbf{0.614} & 0.759 \\
\cline{2-11}
& \textsc{Meerkat-vp}  &30  & \textbf{0.919} & 0.825 & 0.963 & \textbf{0.685} & \textbf{0.595} & 0.634 & \textbf{0.631} & \textbf{0.750} \\
& \textsc{Meerkat}  & 30 & 0.897 & \textbf{0.862} & \textbf{0.965} & 0.646 & 0.577 & \textbf{0.644} & 0.583 & 0.739 \\
\cline{2-11}
& \textsc{Meerkat-vp} & 50 & 0.909 & \textbf{0.836} & 0.959 & \textbf{0.691} & 0.577 & 0.615 & \textbf{0.615} & \textbf{0.743} \\
& \textsc{Meerkat} & 50 & 0.909 & 0.827 & \textbf{0.965} & 0.647 & \textbf{0.595} & \textbf{0.634} & 0.567 & 0.734  \\
\cline{2-11}
& \textsc{Meerkat-vp} & 100    & \textbf{0.904} & \textbf{0.824} & \textbf{0.962} & \textbf{0.684} & 0.577 & \textbf{0.653} & \textbf{0.630} & \textbf{0.747} \\
& \textsc{Meerkat} & 100 & 0.896 & 0.777 & 0.961 & 0.658 & 0.577 & 0.644 & 0.573 & 0.726 \\
\midrule
\multirow{8}{*}{\textbf{Qwen2-1.5b}}
& \textsc{Meerkat-vp}  & 10   & 0.941 & \textbf{0.886} & \textbf{0.947} & \textbf{0.76} & \textbf{0.822} & 0.653 & \textbf{0.636} & \textbf{0.806} \\
& \textsc{Meerkat} & 10 & \textbf{0.949} & 0.881 & 0.934 & 0.752 & 0.813 & \textbf{0.682} & 0.628 & 0.805 \\
\cline{2-11}
& \textsc{Meerkat-vp}  & 30  & 0.935 & 0.876 & \textbf{0.953} & \textbf{0.759} & \textbf{0.822} & 0.653 & \textbf{0.626} & \textbf{0.803} \\
& \textsc{Meerkat} & 30 & \textbf{0.944} & \textbf{0.878} & 0.928 & 0.734 & 0.800 & \textbf{0.663} & 0.624 & 0.795 \\
\cline{2-11}
& \textsc{Meerkat-vp} & 50   & 0.931 & \textbf{0.882} & \textbf{0.946} & \textbf{0.754} & \textbf{0.804} & 0.644 & \textbf{0.63} & \textbf{0.798} \\
& \textsc{Meerkat} & 50 & \textbf{0.948} & 0.872 & 0.926 & 0.746 & 0.795 & \textbf{0.663} & 0.594 & 0.792 \\
\cline{2-11}
& \textsc{Meerkat-vp} & 100 & 0.935 & 0.874 & \textbf{0.947} & \textbf{0.751} & \textbf{0.817} & 0.653 & \textbf{0.644} & \textbf{0.803} \\
& \textsc{Meerkat} & 100 & \textbf{0.936} & \textbf{0.878} & 0.925 & 0.741 & 0.795 & \textbf{0.663} & 0.61 & 0.792 \\
\midrule
\multirow{8}{*}{\textbf{Gemma2-2b}}
& \textsc{Meerkat-vp} & 10     & \textbf{0.948} & \textbf{0.873} & \textbf{0.971} & 0.802 & \textbf{0.657} & \textbf{0.663} & 0.609 & \textbf{0.789} \\
& \textsc{Meerkat} & 10 & 0.939 & 0.869 & 0.96 & \textbf{0.804} & 0.591 & 0.634 & 0.609 & 0.772 \\
\cline{2-11}
& \textsc{Meerkat-vp} & 30   & \textbf{0.948} & \textbf{0.86} & \textbf{0.974} & \textbf{0.799} & \textbf{0.6} & 0.634 & \textbf{0.619} & \textbf{0.776} \\
& \textsc{Meerkat} & 30 & 0.94 & 0.855 & 0.947 & 0.734 & 0.568 & \textbf{0.644} & 0.601 & 0.755 \\
\cline{2-11}
& \textsc{Meerkat-vp} & 50  & \textbf{0.949} & 0.853 & \textbf{0.969} & \textbf{0.782} & 0.551 & 0.615 & 0.620& 0.762 \\
& \textsc{Meerkat} & 50 & 0.945 & \textbf{0.857} & 0.966 & 0.767 & \textbf{0.613} & \textbf{0.634} & \textbf{0.623} & \textbf{0.772} \\
\cline{2-11}
& \textsc{Meerkat-vp} & 100   & \textbf{0.944} & 0.812 & \textbf{0.97} & 0.733 & 0.551 & 0.634 & \textbf{0.634} & \textbf{0.754} \\
& \textsc{Meerkat} & 100 & 0.94 & \textbf{0.851} & 0.951 & \textbf{0.745} & 0.551 & 0.634 & 0.574 & 0.749 \\
\bottomrule
\end{tabular}
}
\end{table*}

\begin{table*}[tbp]
\centering
\captionsetup{width=\textwidth, justification=centering}
\caption{\textsc{Meerkat} performance at local step $=1$ with varying outlier percentages across the LLaMA-3.2-1B, Qwen2-1.5b, and Gemma2-2b models. We report test accuracy on SST-2, AgNews, Yelp, BoolQ, RTE, WSC, and WIC under both IID 
and Non-IID client data settings. Bold numbers indicate the highest value in each row.}
\label{tab:outlier_percentage}
\vspace{5pt}
\resizebox{\textwidth}{!}{
\begin{tabular}{lc ccccccc ccccccc}
\toprule
& & \multicolumn{7}{c}{\textbf{IID}} 
& \multicolumn{7}{c}{\textbf{Non-IID}} \\
\cmidrule(lr){3-9} \cmidrule(lr){10-16}
\textbf{Model} & \textbf{Outlier Percentage}
& \textbf{SST-2} & \textbf{AgNews} & \textbf{Yelp} & \textbf{BoolQ} & \textbf{RTE} & \textbf{WSC} & \textbf{WIC}
& \textbf{SST-2} & \textbf{AgNews} & \textbf{Yelp} & \textbf{BoolQ} & \textbf{RTE} & \textbf{WSC} & \textbf{WIC} \\
\midrule
\multirow{4}{*}{\textbf{LLaMA-3.2-1B}}
& 5e-1     
  & 0.917  & 0.72 & 0.965 & 0.725 & 0.653  & 0.644 & 0.634 
  & 0.895  & 0.669 & 0.964 & 0.684 & 0.644 & 0.653 & 0.594 \\
& 5e-2    
  & 0.913  & 0.861 & 0.966 & 0.749 & 0.653 & 0.644 & 0.633
  & 0.915 & 0.87 & \textbf{0.97} & 0.722 & 0.653  & 0.644 & 0.619 \\
& 5e-3   
  & 0.900  & \textbf{0.885} & \textbf{0.971} & 0.769 & 0.702 & 0.653 & 0.614
  & \textbf{0.930}  & 0.874 & 0.963 & \textbf{0.753} & 0.620 & 0.66 & 0.62 \\
& 5e-4
  & 0.910  & 0.877 & 0.954 & \textbf{0.773} & \textbf{0.720} & \textbf{0.663} & 0.641
  & 0.911   & \textbf{0.888} & 0.956 & 0.700 & \textbf{0.693}  & \textbf{0.663}  & \textbf{0.628}  \\
& 5e-5
  & \textbf{0.922}  & 0.879 & 0.964 & 0.724 & 0.631 & 0.625 & \textbf{0.648}
  & 0.92   & 0.876 & 0.940 & 0.725 & 0.613 & \textbf{0.663}  & 0.626  \\
\midrule
\multirow{4}{*}{\textbf{Qwen2-1.5b}}
& 5e-1     
  & 0.854 & 0.856 & 0.947 & 0.766 & 0.82 & 0.663 & 0.644
  & 0.845  & 0.854 & 0.946 & 0.753 & 0.826 & 0.682 & 0.631 \\
& 5e-2      
  & 0.925  & 0.868 & 0.949 & 0.778 & 0.826 & 0.692 & 0.647
  & 0.93  & 0.853  & 0.943 & 0.759 & 0.822 & 0.663 & 0.663 \\
& 5e-3  
  & \textbf{0.926}  & 0.851 & 0.945 & 0.765 & \textbf{0.813} & \textbf{0.692} & \textbf{0.658}
  & 0.924  & \textbf{0.866} & 0.94 & 0.759 & 0.822 & \textbf{0.692} & 0.661  \\
& 5e-4
  & 0.92  & 0.764 & 0.943 & 0.774 & 0.813 & 0.682 & 0.645
  & 0.918  & 0.848 & 0.943  & \textbf{0.762} & 0.813 & 0.682 & 0.647 \\
& 5e-5
  & 0.903  & 0.78 & 0.941 & 0.748 & 0.80 & 0.673 & 0.625
  & 0.896  & 0.799 & 0.937  & 0.739 & 0.80 & 0.673 & 0.633 \\
\midrule
\multirow{4}{*}{\textbf{Gemma2-2b}}
& 5e-1     
  & 0.842  & 0.867  & 0.963 & 0.751 & 0.657 & \textbf{0.673} & 0.626
  & 0.871   & 0.855 & 0.952  & 0.695 & 0.653  & \textbf{0.663} & 0.619 \\
& 5e-2     
  & 0.932  & \textbf{0.878} & \textbf{0.977}  & 0.809 & 0.791  & 0.663 & 0.623
  & 0.92 & \textbf{0.863} & 0.968 & 0.786 & 0.706 & 0.653 & 0.634 \\
& 5e-3  
  & \textbf{0.952}  & 0.871 & 0.971 & \textbf{0.837}  & \textbf{0.800} & 0.663 & \textbf{0.639}
  & \textbf{0.942} & 0.853 & \textbf{0.97} & 0.807 & \textbf{0.751}  & 0.653 & \textbf{0.645} \\
& 5e-4
  & 0.941  & 0.824 & 0.967 & 0.83 & 0.764 & 0.663 & 0.612
  & 0.941  & 0.83 & 0.962  & \textbf{0.831} & 0.746 & 0.634 & 0.63 \\
& 5e-5
  & 0.92  & 0.828 & 0.952 & 0.797 & 0.6 & 0.634 & 0.606
  & 0.922  & 0.764 & 0.949  & 0.774 & 0.56 & 0.634 & 0.601 \\
\bottomrule
\end{tabular}
}
\end{table*}

\begin{table*}[tbp]
\centering
\captionsetup{width=\textwidth, justification=centering}
\caption{Performance comparison of Full-FedZO, LoRA-FedZO, and \textsc{Meerkat} under 
synchronous updates with $local step = 1$, evaluated on both IID and Non-IID client data settings(\textbf{Dirichlet $\alpha=0.5$}) 
across LLaMA-3.2-1B, Qwen2-1.5b, and Gemma2-2b. We report test accuracy on SST-2, AgNews, 
Yelp, BoolQ, RTE, WSC, and WIC. Bold numbers indicate the highest value in each row.}
\label{tab:all_in_one}
\vspace{5pt}
\resizebox{0.8\textwidth}{!}{
\begin{tabular}{ll cccccccc}
\toprule
\textbf{Model} & \textbf{Method} 
& \textbf{SST-2} & \textbf{AgNews} & \textbf{Yelp} & \textbf{BoolQ} & \textbf{RTE} & \textbf{WSC} & \textbf{WIC} & \textbf{Acc} \\
\midrule
\multirow{3}{*}{\textbf{LLaMA-3.2-1B (IID)}} 
&  Full-FedZO & \textbf{0.918} & 0.801 & 0.937 & 0.686 & 0.54 & 0.625 & 0.58 & 0.726  \\
& LoRA-FedZO & 0.915 & 0.855 & 0.944 & 0.672 & 0.599 & \textbf{0.663} & 0.599 & 0.749 \\
& \textsc{Meerkat} & 0.900 & \textbf{0.885} & \textbf{0.971} & \textbf{0.773} & \textbf{0.702} & 0.653 & \textbf{0.614} & \textbf{0.785} \\
\cmidrule(lr){2-10} 
\multirow{3}{*}{\textbf{LLaMA-3.2-1B (Non-IID)}} 
&  Full-FedZO & 0.911 & 0.831 & 0.937 & 0.672 & 0.528 & 0.587 & 0.567 & 0.719 \\
& LoRA-FedZO & 0.8669 & 0.842 & 0.944 & 0.659 & 0.53 & 0.567 & 0.578 & 0.712 \\
& \textsc{Meerkat} & \textbf{0.93} & \textbf{0.888} & \textbf{0.963} & \textbf{0.753} & \textbf{0.67} & \textbf{0.66} & \textbf{0.62} & \textbf{0.783} \\
\cmidrule(lr){1-10} 
\multirow{3}{*}{\textbf{Qwen2-1.5b (IID)}} 
&  Full-FedZO & 0.9013 & 0.726 & 0.918 & 0.700 & 0.797 & \textbf{0.710} & 0.579 & 0.761 \\
& LoRA-FedZO & \textbf{0.935} & 0.752 & 0.925 & 0.686 & 0.794 & 0.673 & 0.606 & 0.767 \\
& \textsc{Meerkat} & 0.926 & \textbf{0.851} & \textbf{0.945} & \textbf{0.778} & \textbf{0.813} & 0.692 & \textbf{0.658} & \textbf{0.809} \\
\cmidrule(lr){2-10} 
\multirow{3}{*}{\textbf{Qwen2-1.5b (Non-IID)}} 
&  Full-FedZO & 0.844 & 0.725 & 0.937 & 0.688 & 0.769 & 0.663 & 0.565 & 0.741 \\
& LoRA-FedZO & \textbf{0.932} & 0.76 & \textbf{0.944} & 0.682 & 0.773 & 0.682 & 0.565 & 0.763 \\
& \textsc{Meerkat} & 0.924 & \textbf{0.866} & 0.94 & \textbf{0.762} & \textbf{0.822} & \textbf{0.692} & \textbf{0.661} & \textbf{0.809} \\
\cmidrule(lr){1-10} 
\multirow{3}{*}{\textbf{Gemma2-2b (IID)}} 
&  Full-FedZO & 0.934 & 0.84 & 0.953 & 0.774 & 0.542 & 0.644 & 0.606 & 0.756 \\
& LoRA-FedZO & 0.942 & 0.856 & 0.94 & 0.735 & 0.52 & 0.644 & 0.606 & 0.749 \\
& \textsc{Meerkat} & \textbf{0.952} & \textbf{0.871} & \textbf{0.971} & \textbf{0.837} & \textbf{0.8} & \textbf{0.663} & \textbf{0.639} & \textbf{0.819} \\
\cmidrule(lr){2-10} 
\multirow{3}{*}{\textbf{Gemma2-2b (Non-IID)}} 
&  Full-FedZO & 0.93 & 0.824 & 0.95 & 0.744 & 0.56 & 0.625 & 0.575 & 0.744 \\
& LoRA-FedZO & 0.9415 & 0.825 & 0.954 & 0.711 & 0.528 & 0.625 & 0.578 & 0.737 \\
& \textsc{Meerkat} & \textbf{0.942} & \textbf{0.853} & \textbf{0.97} & \textbf{0.807} & \textbf{0.751} & \textbf{0.653} & \textbf{0.645} &\textbf{0.803} \\
\bottomrule
\end{tabular}
}
\end{table*}

\begin{table*}[tbp]
\centering
\captionsetup{width=\textwidth, justification=centering}
\caption{Performance comparison of LoRA-FedZO, and \textsc{Meerkat} under 
synchronous updates with $local step = 1$, evaluated on Non-IID client data settings (\textbf{Dirichlet $\alpha=0.3$}) 
across LLaMA-3.2-1B, Qwen2-1.5b, and Gemma2-2b. We report test accuracy on SST-2, AgNews, 
Yelp, BoolQ, RTE, WSC, and WIC. Bold numbers indicate the highest value in each row.}
\label{tab:all_in_one_alpha0.3}
\vspace{5pt}
\resizebox{0.8\textwidth}{!}{
\begin{tabular}{ll cccccccc}
\toprule
\textbf{Model} & \textbf{Method} 
& \textbf{SST-2} & \textbf{AgNews} & \textbf{Yelp} & \textbf{BoolQ} & \textbf{RTE} & \textbf{WSC} & \textbf{WIC} & \textbf{Acc} \\
\midrule
\multirow{3}{*}{\textbf{LLaMA-3.2-1B (Non-IID)}} 
& Full-FedZO & 0.891 & 0.759 & 0.94 & 0.623 & 0.528 & 0.644 & 0.551 & 0.705 \\
& LoRA-FedZO & 0.915 & 0.866 & 0.952 & 0.646 & 0.586 & 0.653 & 0.554 & 0.739 \\
& \textsc{Meerkat} & \textbf{0.918} & \textbf{0.843} & \textbf{0.97} & \textbf{0.761} & \textbf{0.626} & \textbf{0.653} & \textbf{0.609} & \textbf{0.769} \\
\cmidrule(lr){1-10} 
\multirow{3}{*}{\textbf{Qwen2-1.5b (Non-IID)}} 
& Full-FedZO & 0.52 & 0.347 & 0.45 & 0.62 & 0.532 & 0.632 & 0.51 & 0.516 \\
& LoRA-FedZO & 0.855 & 0.732 & 0.907 & 0.674 & 0.72 & 0.634 & 0.603 & 0.732 \\
& \textsc{Meerkat} & \textbf{0.91} & \textbf{0.809} & \textbf{0.954} & \textbf{0.772} & \textbf{0.822} & \textbf{0.682} & \textbf{0.661} & \textbf{0.801} \\
\cmidrule(lr){1-10} 
\multirow{3}{*}{\textbf{Gemma2-2b (Non-IID)}} 
& Full-FedZO & 0.881 & 0.761 & 0.94 & 0.688 & 0.552 & 0.613 & 0.603 & 0.720 \\
& LoRA-FedZO & 0.922 & 0.826 & 0.921 & 0.681 & 0.52 & 0.625 & 0.606 & 0.729 \\
& \textsc{Meerkat} & \textbf{0.942} & \textbf{0.873} & \textbf{0.97} & \textbf{0.806} & \textbf{0.688} & \textbf{0.634} & \textbf{0.615} &\textbf{0.79} \\
\bottomrule
\end{tabular}
} 
\end{table*}

\begin{table*}[tbp]
\centering
\captionsetup{width=\textwidth, justification=centering}
\caption{Performance comparison of LoRA-FedZO, and \textsc{Meerkat} under 
synchronous updates with $local step = 1$, evaluated on Non-IID client data settings (\textbf{Dirichlet $\alpha=0.1$}) 
across LLaMA-3.2-1B, Qwen2-1.5b, and Gemma2-2b. We report test accuracy on SST-2, AgNews, 
Yelp, BoolQ, RTE, WSC, and WIC. Bold numbers indicate the highest value in each row.}
\label{tab:all_in_one_alpha0.1}
\vspace{5pt}
\resizebox{0.8\textwidth}{!}{
\begin{tabular}{ll cccccccc}
\toprule
\textbf{Model} & \textbf{Method} 
& \textbf{SST-2} & \textbf{AgNews} & \textbf{Yelp} & \textbf{BoolQ} & \textbf{RTE} & \textbf{WSC} & \textbf{WIC} & \textbf{Acc} \\
\midrule
\multirow{3}{*}{\textbf{LLaMA-3.2-1B (Non-IID)}} 
& Full-FedZO & 0.891 & 0.754 & 0.933 & 0.626 & 0.522 & 0.365 & 0.512 & 0.658 \\
& LoRA-FedZO & 0.902 & 0.845 & 0.942 & 0.643 & 0.533 & 0.365 & 0.559 & 0.684 \\
& \textsc{Meerkat} & \textbf{0.92} & \textbf{0.794} & \textbf{0.965} & \textbf{0.745} & \textbf{0.582} & \textbf{0.644} & \textbf{0.603} & \textbf{0.750} \\
\cmidrule(lr){1-10} 
\multirow{3}{*}{\textbf{Qwen2-1.5b (Non-IID)}} 
& Full-FedZO & 0.49 & 0.247 & 0.44 & 0.62 & 0.528 & 0.634 & 0.5 & 0.494 \\
& LoRA-FedZO & 0.848 & 0.735 & 0.92 & 0.67 & 0.746 & 0.548 & 0.601 & 0.724 \\
& \textsc{Meerkat} & \textbf{0.889} & \textbf{0.78} & \textbf{0.944} & \textbf{0.732} & \textbf{0.822} & \textbf{0.634} & \textbf{0.637} & \textbf{0.777} \\
\cmidrule(lr){1-10} 
\multirow{3}{*}{\textbf{Gemma2-2b (Non-IID)}} 
& Full-FedZO & 0.879 & 0.741 & 0.937 & 0.681 & 0.48 & 0.634 & 0.601 & 0.708 \\
& LoRA-FedZO & 0.91 & 0.78 & 0.914 & 0.682 & 0.551 & 0.567 & 0.608 & 0.716 \\
& \textsc{Meerkat} & \textbf{0.944} & \textbf{0.866} & \textbf{0.971} & \textbf{0.805} & \textbf{0.728} & \textbf{0.605} & \textbf{0.628} &\textbf{0.792} \\
\bottomrule
\end{tabular}
} 
\end{table*}

\begin{figure*}[htbp]
    \centering
    \begin{subfigure}[b]{0.9\textwidth}
        \centering
        \includegraphics[width=\linewidth]{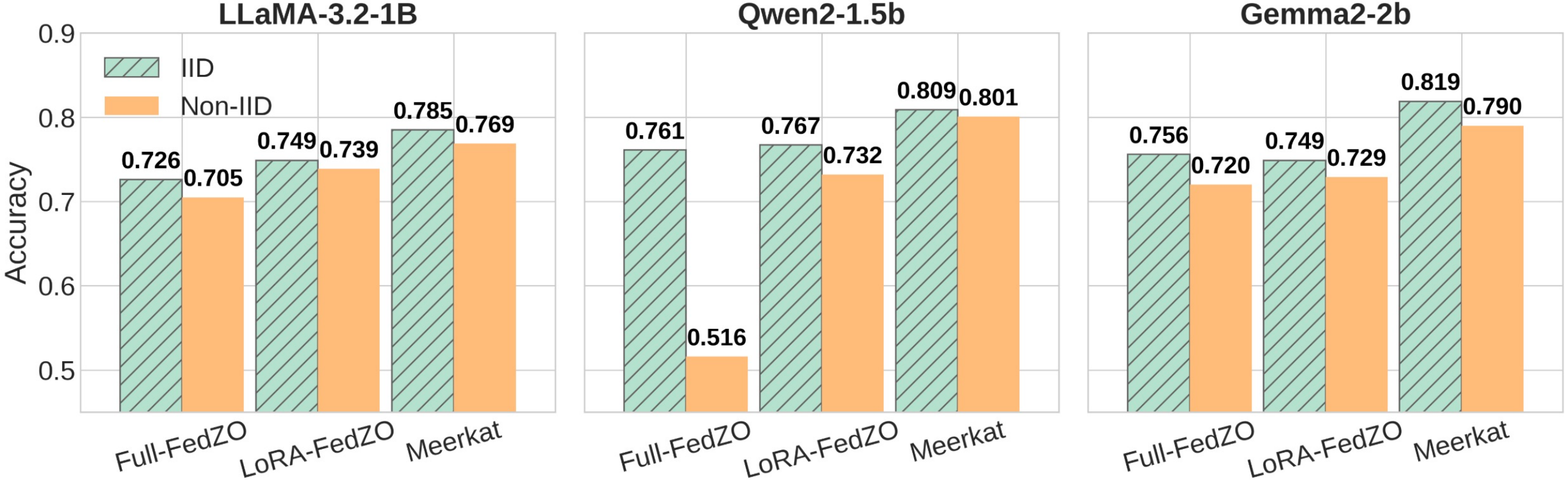}
        \caption{This figure compares three methods—Full-FedZO, LoRA-FedZO, and \textsc{Meerkat}—on three LLMs: LLaMA-3.2-1B, Qwen2-1.5b, and Gemma2-2b. The x-axis shows the different methods, and each method has two bars indicating performance under IID and Non-IID settings. The Non-IID results are obtained under a Dirichlet $\alpha = 0.3$ .The y-axis represents the average test accuracy across multiple downstream tasks—SST2, AgNews, Yelp, BoolQ, RTE, WSC, and WiC.}
        \label{fig:high_freq_alpha03}
    \end{subfigure}

    \vspace{0.2cm}
    
    \begin{subfigure}[b]{0.9\textwidth}
        \centering
        \includegraphics[width=\linewidth]{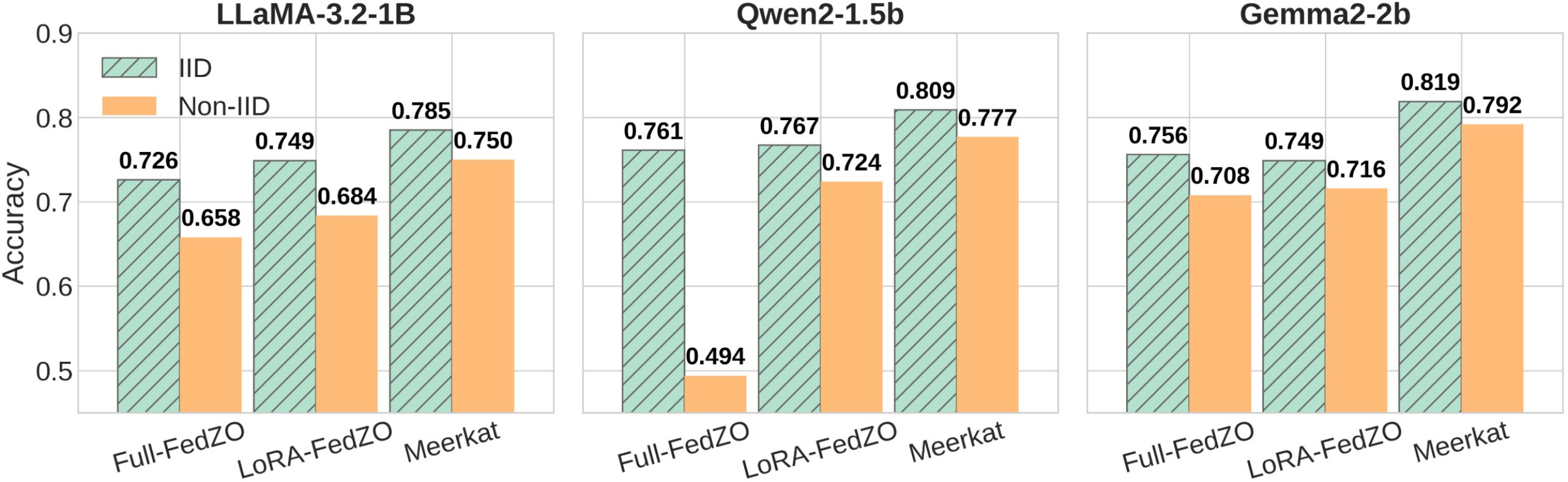}
        \caption{This figure compares three methods—Full-FedZO, LoRA-FedZO, and \textsc{Meerkat}—on three LLMs: LLaMA-3.2-1B, Qwen2-1.5b, and Gemma2-2b. The x-axis shows the different methods, and each method has two bars indicating performance under IID and Non-IID settings. The Non-IID results are obtained under a Dirichlet $\alpha = 0.1$ .The y-axis represents the average test accuracy across multiple downstream tasks—SST2, AgNews, Yelp, BoolQ, RTE, WSC, and WiC.}
        \label{fig:high_freq_alpha01}
    \end{subfigure}
    
    \vspace{0.5cm}
    \caption{Comparison of Full-FedZO, LoRA-FedZO, and \textsc{Meerkat} on LLaMA-3.2-1B, Qwen2-1.5b, and Gemma2-2b under IID and Non-IID settings with varying Dirichlet $\alpha$. Subfigure(a) presents results for Non-IID data generated with $\alpha = 0.3$, while Subfigure(b) shows results for Non-IID data with $\alpha = 0.1$.}
    \label{fig:high_freq_alpha03_01}
\end{figure*}

\begin{table*}[tbp]
\centering
\captionsetup{width=\textwidth, justification=centering}
\caption{Test accuracy of \textsc{Meerkat} versus DecomFL on Qwen2-1.5b with a single local step under Non-IID data settings (Dirichlet $\alpha=1$). Results are shown for SST-2, BoolQ, RTE, and WSC; bold indicates the best score in each row. Experiments use 8 clients in total, with 2 clients participating in each round, following the DecomFL configuration.}
\label{tab:meerkat_vs_decomfl}
\vspace{5pt}
\resizebox{0.6\textwidth}{!}{
\begin{tabular}{ll ccccc}
\toprule
\textbf{Model} & \textbf{Method} 
& \textbf{SST-2} & \textbf{BoolQ} & \textbf{RTE} & \textbf{WSC} \\
\midrule 
\multirow{2}{*}{\textbf{Qwen2-1.5b}} 
& DecomFL & 0.868 & 0.674 & 0.773 & 0.653 \\
& \textsc{Meerkat} & \textbf{0.918} & \textbf{0.734} & \textbf{0.817} & \textbf{0.682} \\
\bottomrule
\end{tabular}
} 
\end{table*}

\begin{table*}[tbp]
\centering
\captionsetup{width=\textwidth, justification=centering}
\caption{Performance comparison of Task-Mask, and \textsc{Meerkat} under 
synchronous updates with $local step = 1$, evaluated on IID client data settings
across LLaMA-3.2-1B, Qwen2-1.5b, and Gemma2-2b. We report test accuracy on SST-2, AgNews, 
Yelp, BoolQ, RTE, WSC, and WIC. Bold numbers indicate the highest value in each row.}
\label{tab:c4_vs_taskdata_IId}
\vspace{5pt}
\resizebox{0.8\textwidth}{!}{
\begin{tabular}{ll ccccccc}
\toprule
\textbf{Model} & \textbf{Method} 
& \textbf{SST-2} & \textbf{AgNews} & \textbf{Yelp} & \textbf{BoolQ} & \textbf{RTE} & \textbf{WSC} & \textbf{WIC} \\
\midrule
\cmidrule(lr){1-9} 
\multirow{2}{*}{\textbf{LLaMA-3.2-1B (IID)}} 
& Task & \textbf{0.910} & 0.847 & 0.957 & 0.718 & 0.661 & 0.644 & \textbf{0.661} \\
& \textsc{Meerkat} & 0.90 & \textbf{0.885} & \textbf{0.971} & \textbf{0.773} & \textbf{0.702} & \textbf{0.653} & 0.614  \\
\cmidrule(lr){1-9} 
\multirow{2}{*}{\textbf{Qwen2-1.5b (IID)}} 
& Task & 0.936 & 0.827 & 0.954 & 0.765 & 0.83 & \textbf{0.711} & \textbf{0.664}  \\
& \textsc{Meerkat} & 0.926 & \textbf{0.851} & 0.945 & \textbf{0.778} & \textbf{0.813} & 0.692 & 0.658  \\
\cmidrule(lr){1-9} 
\multirow{2}{*}{\textbf{Gemma2-2b (IID)}} 
& Task & 0.942 & 0.868 & \textbf{0.972} & 0.78 & 0.728 & 0.644 & 0.6 \\
& \textsc{Meerkat} & \textbf{0.952} & \textbf{0.871} & 0.971 & \textbf{0.837} & \textbf{0.8} & \textbf{0.663} & \textbf{0.639}  \\
\bottomrule
\end{tabular}
} 
\end{table*}

\begin{table*}[tbp]
\centering
\captionsetup{width=\textwidth, justification=centering}
\caption{Performance comparison of Task-Mask, and \textsc{Meerkat} under 
synchronous updates with $local step = 1$, evaluated on Non-IID client data settings (\textbf{Dirichlet $\alpha=0.5$}) 
across LLaMA-3.2-1B, Qwen2-1.5b, and Gemma2-2b. We report test accuracy on SST-2, AgNews, 
Yelp, BoolQ, RTE, WSC, and WIC. Bold numbers indicate the highest value in each row.}
\label{tab:c4_vs_taskdata}
\vspace{5pt}
\resizebox{0.8\textwidth}{!}{
\begin{tabular}{ll ccccccc}
\toprule
\textbf{Model} & \textbf{Method} 
& \textbf{SST-2} & \textbf{AgNews} & \textbf{Yelp} & \textbf{BoolQ} & \textbf{RTE} & \textbf{WSC} & \textbf{WIC} \\
\midrule
\cmidrule(lr){1-9} 
\multirow{2}{*}{\textbf{LLaMA-3.2-1B (Non-IID)}} 
& Task & 0.904 & 0.874 & 0.956 & 0.744 & 0.591 & 0.615 & 0.622 \\
& \textsc{Meerkat} & \textbf{0.93} & \textbf{0.888} & \textbf{0.963} & \textbf{0.753} & \textbf{0.62} & \textbf{0.66} & \textbf{0.62}  \\
\cmidrule(lr){1-9} 
\multirow{2}{*}{\textbf{Qwen2-1.5b (Non-IID)}} 
& Task & 0.938 & 0.863 & 0.956 & 0.779 & 0.817 & 0.692 & 0.65  \\
& \textsc{Meerkat} & \textbf{0.924} & \textbf{0.866} & \textbf{0.94} & \textbf{0.762} & \textbf{0.822} & \textbf{0.692} & \textbf{0.661}  \\
\cmidrule(lr){1-9} 
\multirow{2}{*}{\textbf{Gemma2-2b (Non-IID)}} 
& Task & 0.91 & 0.834 & 0.966 & 0.822 & 0.72 & 0.644 & 0.578 \\
& \textsc{Meerkat} & \textbf{0.942} & \textbf{0.853} & \textbf{0.97} & \textbf{0.807} & \textbf{0.751} & \textbf{0.653} & \textbf{0.645}  \\
\bottomrule
\end{tabular}
} 
\end{table*}

\begin{table*}[tbp]
\centering
\captionsetup{width=\textwidth, justification=centering}
\caption{Test accuracy of \textsc{Meerkat} versus Task-Mask on Qwen2-1.5b with a 10 local step under Non-IID data settings (Dirichlet $\alpha=0.5$). Results are shown for SST-2, BoolQ, RTE, and WSC; bold indicates the best score in each row. Experiments use 8 clients in total, with 2 clients participating in each round, following the DecomFL configuration.}
\label{tab:meerkat_vs_taskdata_multistep}
\vspace{5pt}
\resizebox{0.6\textwidth}{!}{
\begin{tabular}{ll ccccc}
\toprule
\textbf{Model} & \textbf{Method} 
& \textbf{SST-2} & \textbf{BoolQ} & \textbf{RTE} & \textbf{WSC} \\
\midrule 
\multirow{2}{*}{\textbf{Qwen2-1.5b}} 
& Task & 0.932 & 0.784 & 0.823 & 0.681 \\
& \textsc{Meerkat} & \textbf{0.944} & \textbf{0.752} & \textbf{0.813} & \textbf{0.682} \\
\bottomrule
\end{tabular}
} 
\end{table*}

\begin{figure*}[!htbp]
    \centering

    \begin{subfigure}[b]{0.39\textwidth}
        \centering
        \includegraphics[width=\linewidth]{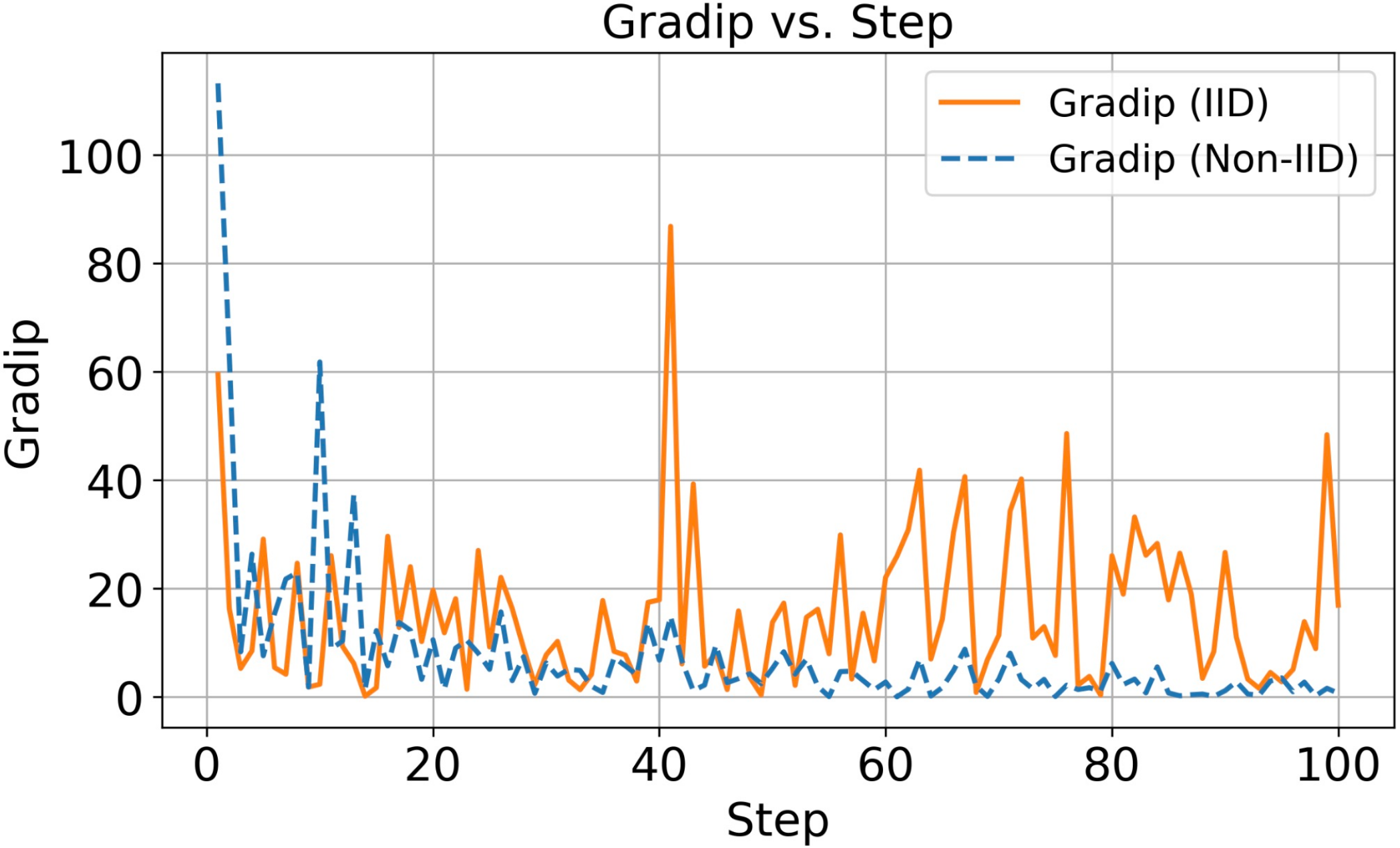}
        \caption{The GradIP measured for IID and Non-IID clients data under the WIC task using the Llama-3.2-1B model.}
        \label{fig:ip_llama_wic}
    \end{subfigure}
    \hfill
    \begin{subfigure}[b]{0.39\textwidth}
        \centering
        \includegraphics[width=\linewidth]{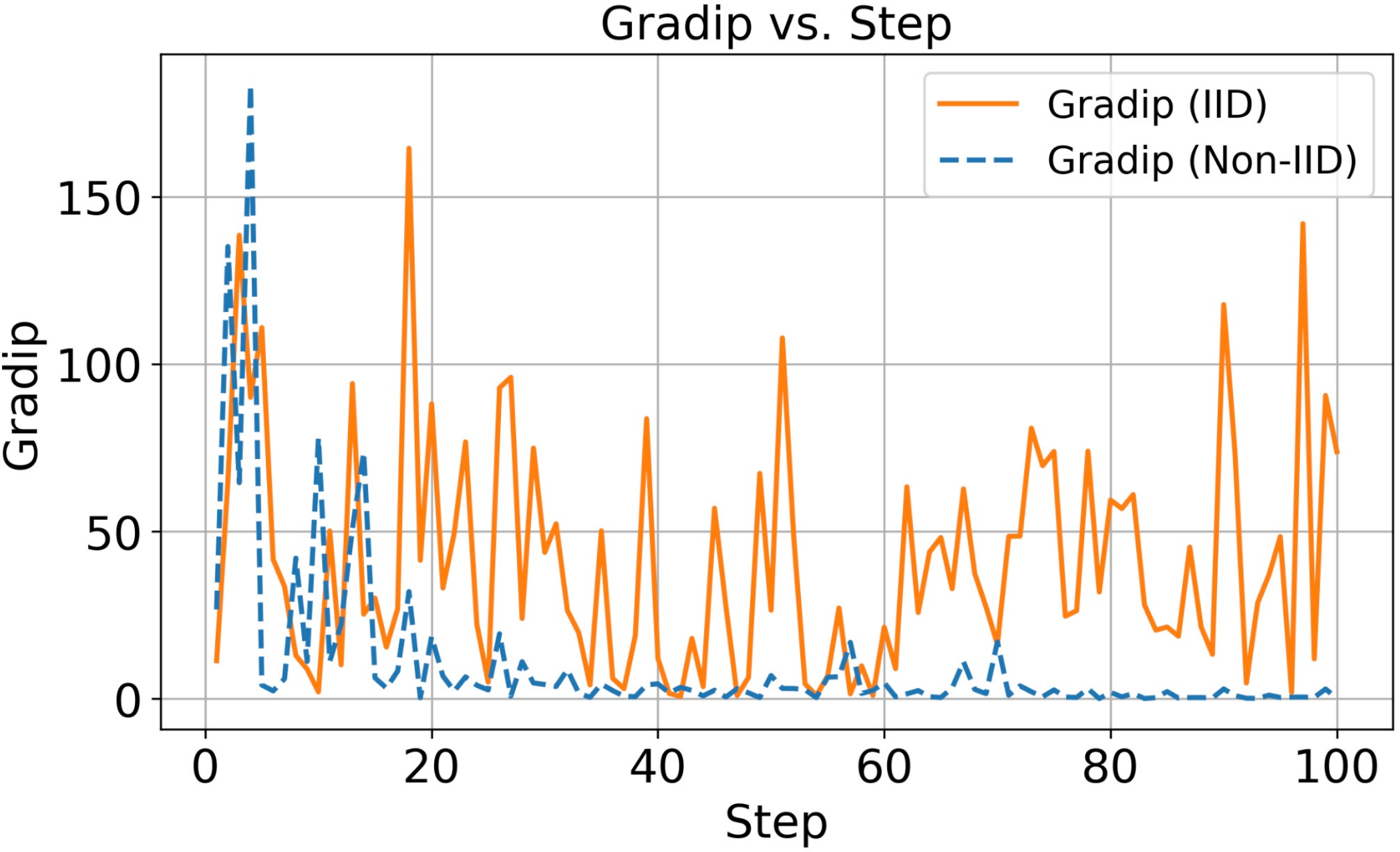}
        \caption{The GradIP measured for IID and Non-IID clients data under the AgNews task using the Llama-3.2-1B model.}
        \label{fig:ip_llama_agnews}
    \end{subfigure}

    \begin{subfigure}[b]{0.39\textwidth}
        \centering
        \includegraphics[width=\linewidth]{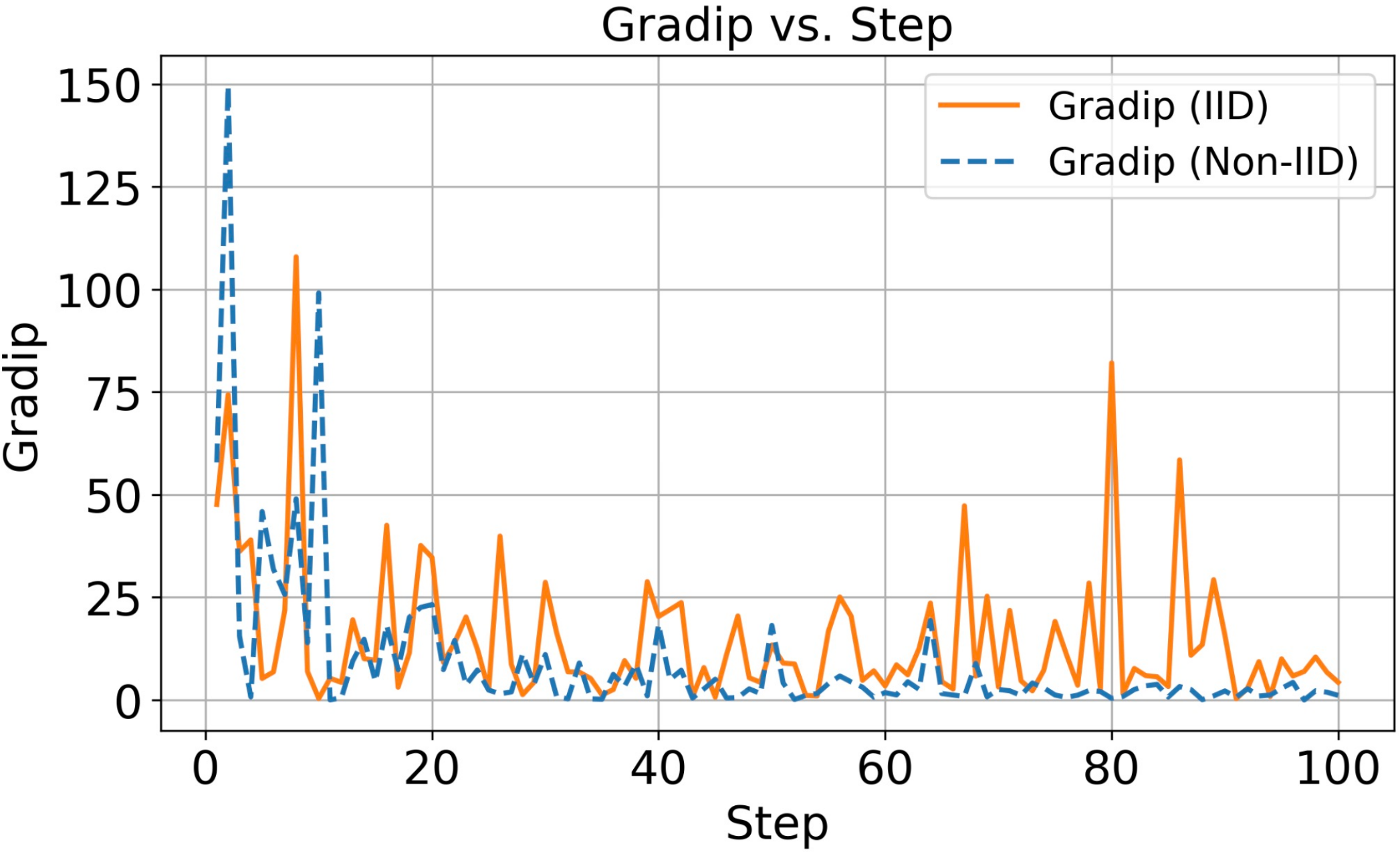}
        \caption{The GradIP measured for IID and Non-IID clients data under the Yelp task using the Llama-3.2-1B model.}
        \label{fig:ip_llama_yelp}
    \end{subfigure}
    \hfill
    \begin{subfigure}[b]{0.39\textwidth}
        \centering
        \includegraphics[width=\linewidth]{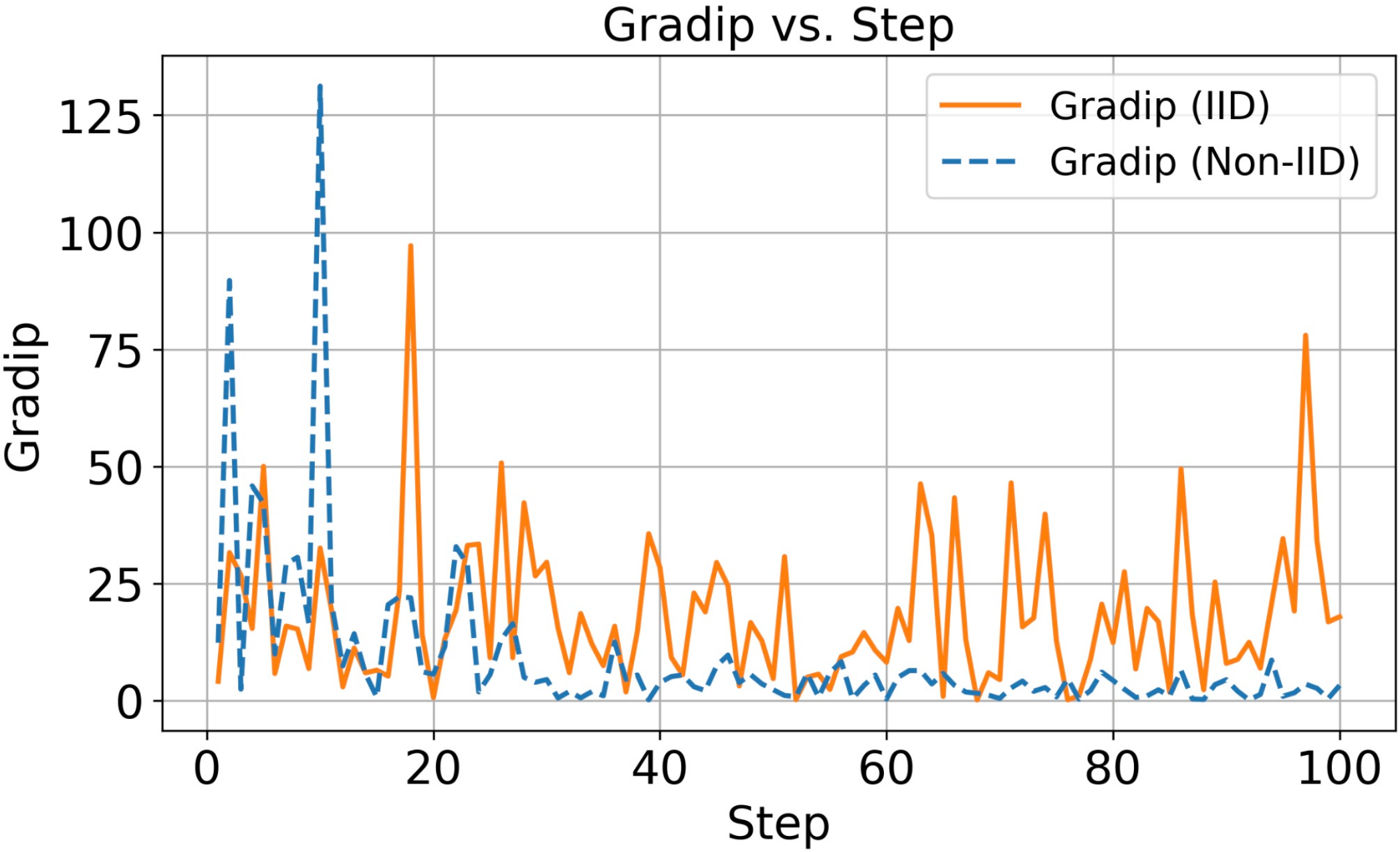}
        \caption{The GradIP measured for IID and Non-IID clients data under the BoolQ task using the Llama-3.2-1B model.}
        \label{fig:subfigD}
    \end{subfigure}

    \begin{subfigure}[b]{0.39\textwidth}
        \centering
        \includegraphics[width=\linewidth]{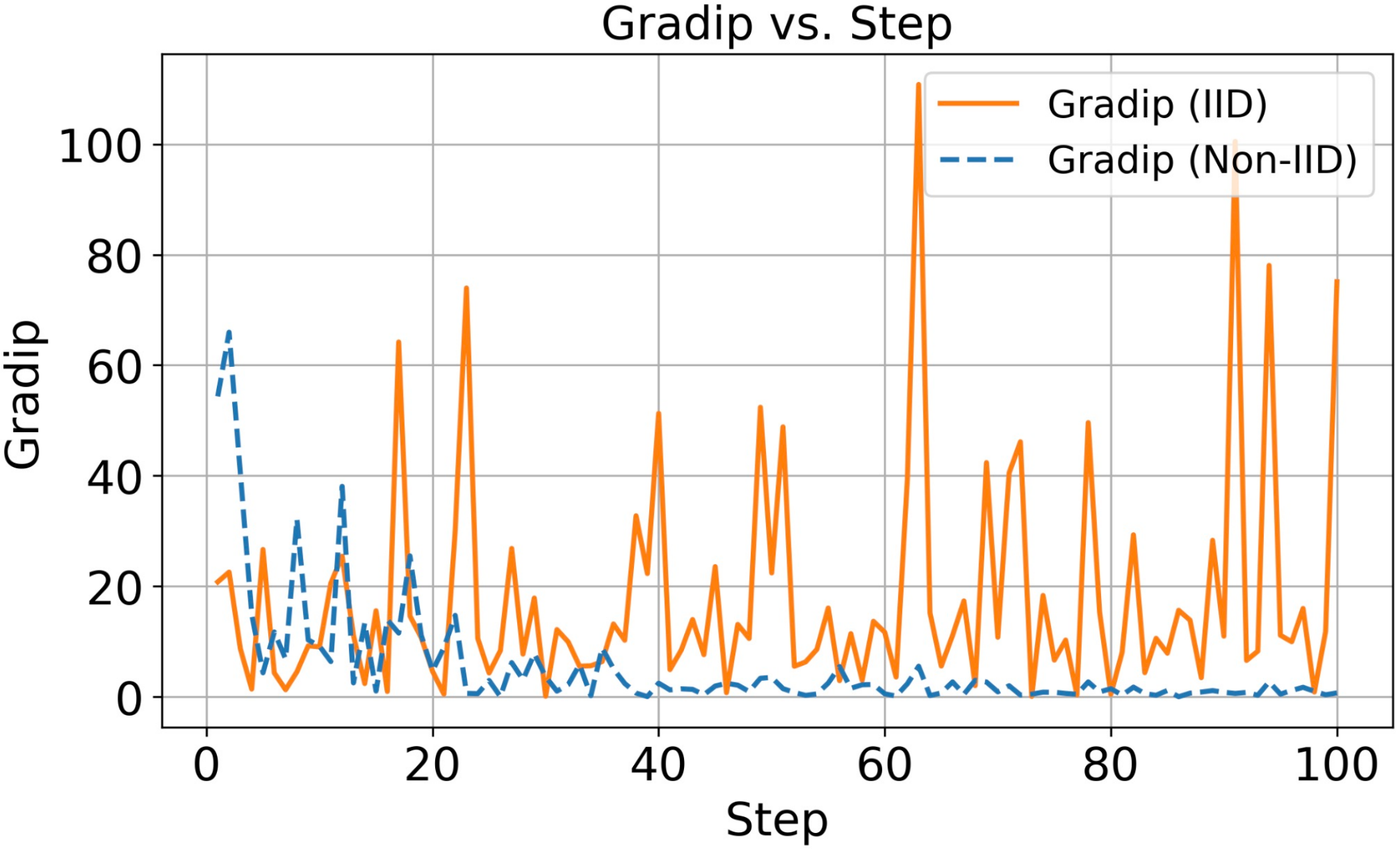}
        \caption{The GradIP measured for IID and Non-IID clients data under the RTE task using the Llama-3.2-1B model.}
        \label{fig:ip_llama_rte}
    \end{subfigure}
    \hfill
    \begin{subfigure}[b]{0.39\textwidth}
        \centering
        \includegraphics[width=\linewidth]{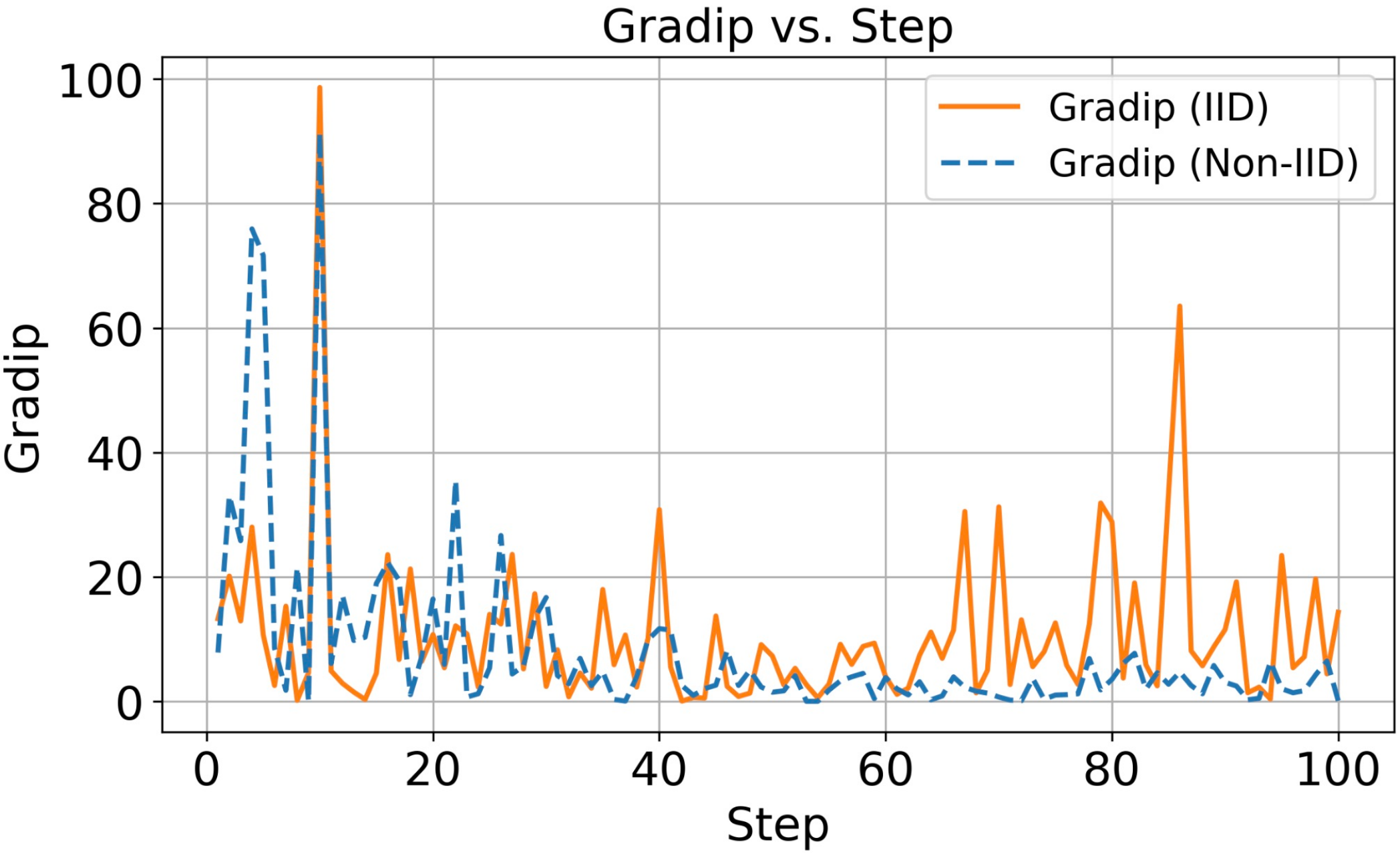}
        \caption{The GradIP measured for IID and Non-IID clients data under the WSC task using the Llama-3.2-1B model.}
        \label{fig:ip_llama_wsc}
    \end{subfigure}


    \begin{subfigure}[b]{0.39\textwidth}
        \centering
        \includegraphics[width=\linewidth]{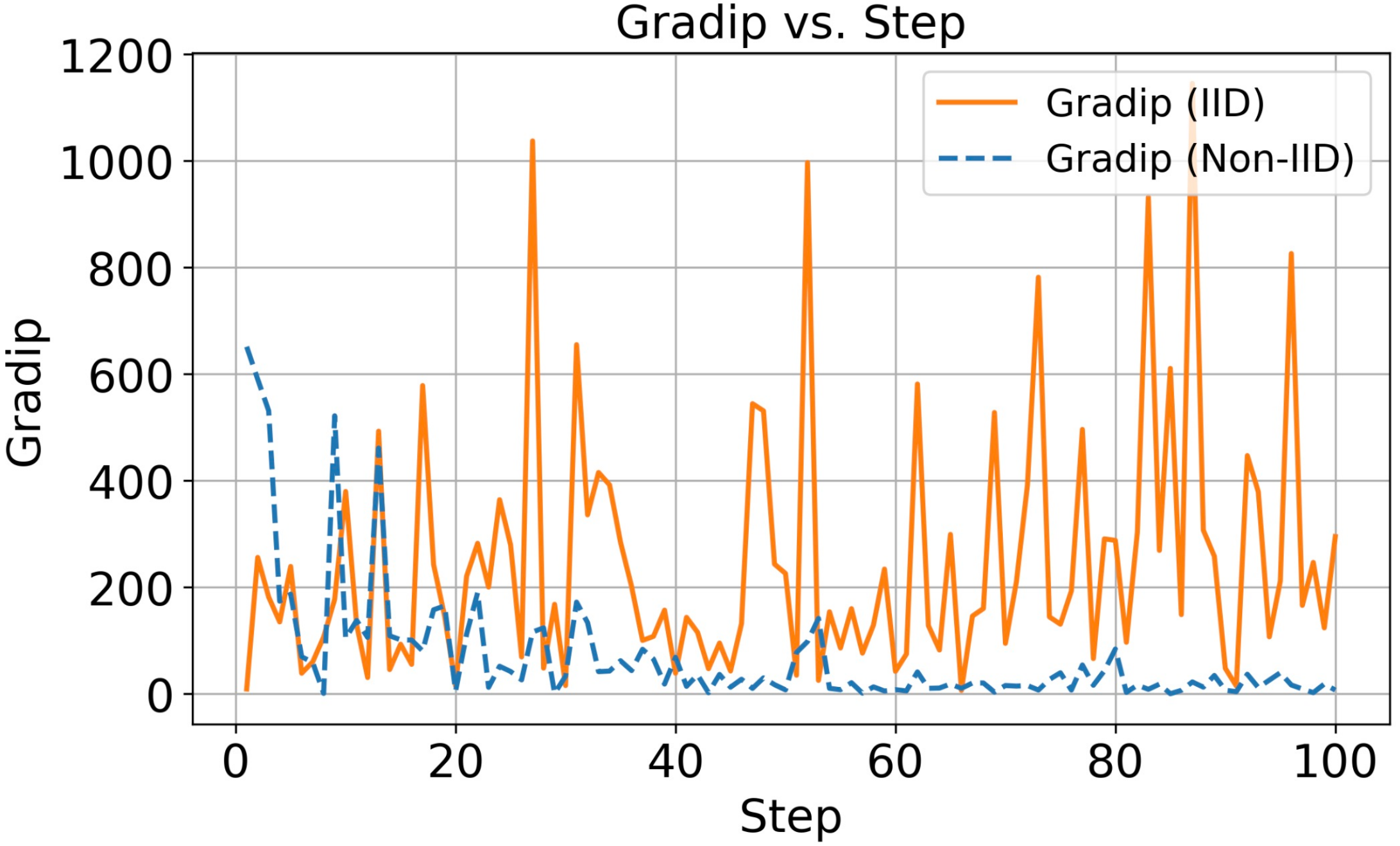}
        \caption{The GradIP measured for IID and Non-IID clients data under the BoolQ task using the Gemma-2-2b model.}
        \label{fig:ip_gemma_boolq}
    \end{subfigure}

    \caption{These figures show GradIP (Definition~\ref{def:gradip}) curves under IID and Non-IID settings, computed over 100 local training steps on six datasets (WSC, BoolQ, RTE, WIC, AgNews, Yelp) using the Llama-3.2-1B model with density level $5\times10^{-3}$. An extra BoolQ result is shown for the Gemma-2-2B model.
}
    \label{fig:gradip_main_fig}
\end{figure*}

\begin{figure*}[htbp]
    \centering

    \begin{subfigure}[b]{0.44\textwidth}
        \centering
        \includegraphics[width=\linewidth]{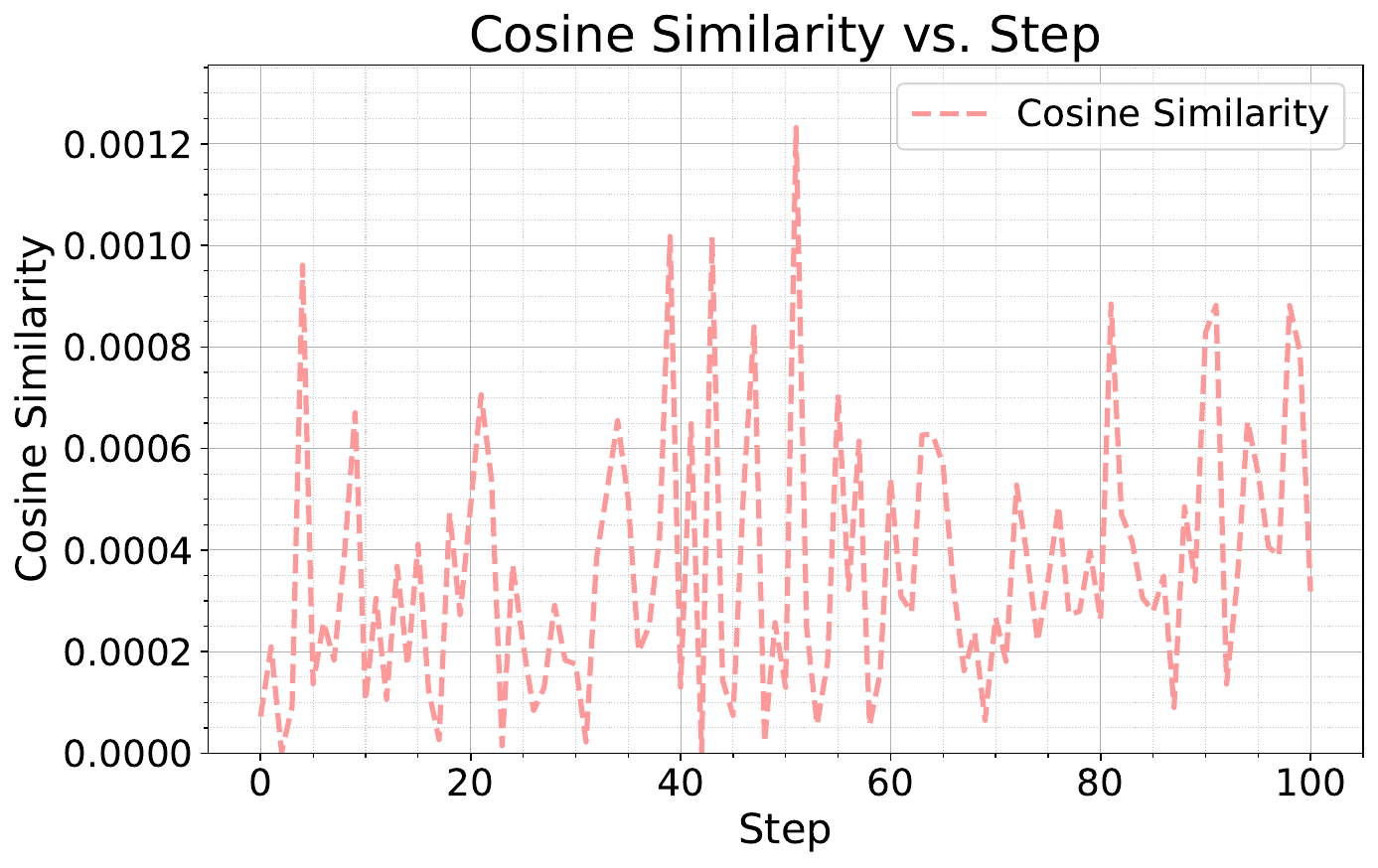}
        \caption{Cosine similarity between local ZO gradients and C4 pre-trained gradients.}

        \label{fig:cosine_sim_abs}
    \end{subfigure}
    \hfill
    \begin{subfigure}[b]{0.44\textwidth}
        \centering
        \includegraphics[width=\linewidth]{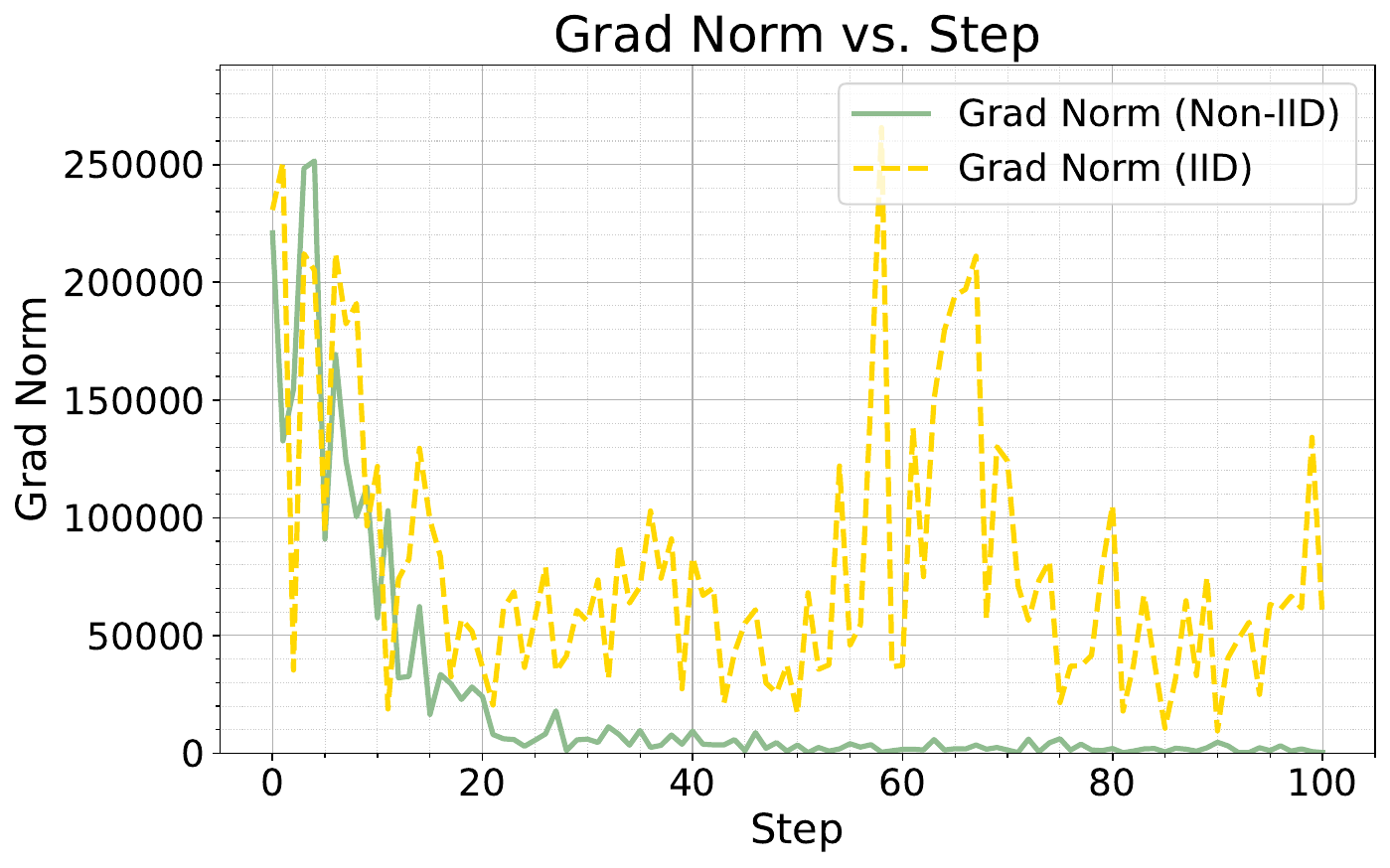}
        \caption{Gradient norm from local ZO training under Non-IID and IID data distribution.}
        \label{fig:gradient_norm_abs}
    \end{subfigure}

    \vspace{0.5cm}

    \caption{The left panel shows the cosine similarity between locally computed ZO gradients and gradients from the C4-pre-trained data, illustrating that the two gradient vectors remain nearly orthogonal throughout training. The right panel presents the norm of local ZO gradients over training steps, showing a consistent decay and convergence in magnitude under Non-IID and IID data distribution. These observations are obtained under density level of $5 \times 10^{-3}$.}
    \label{fig:general_phenomenon_of_gradient}
\end{figure*}

\begin{figure*}[htbp]
    \centering
    \begin{subfigure}[b]{0.44\textwidth}
        \centering
        \includegraphics[width=\linewidth]{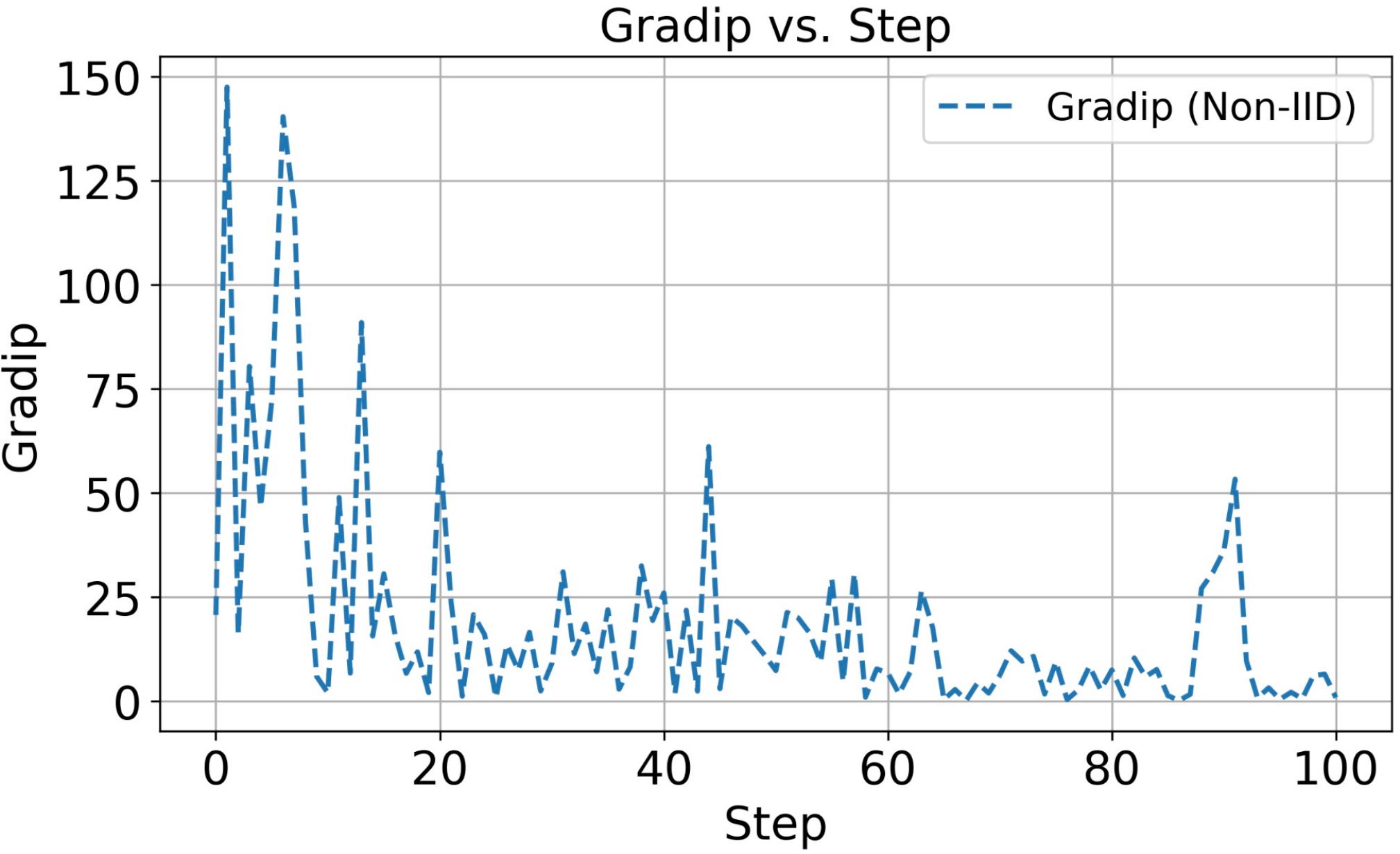}
        \caption{GradIP for Non-IID clients on the AgNews task, 
where the two classes have a highly imbalanced ratio (5 vs.\ 89 samples).}
        \label{fig:ip_llama_agnews_real}
    \end{subfigure}
    \hfill
    \begin{subfigure}[b]{0.44\textwidth}
        \centering
        \includegraphics[width=\linewidth]{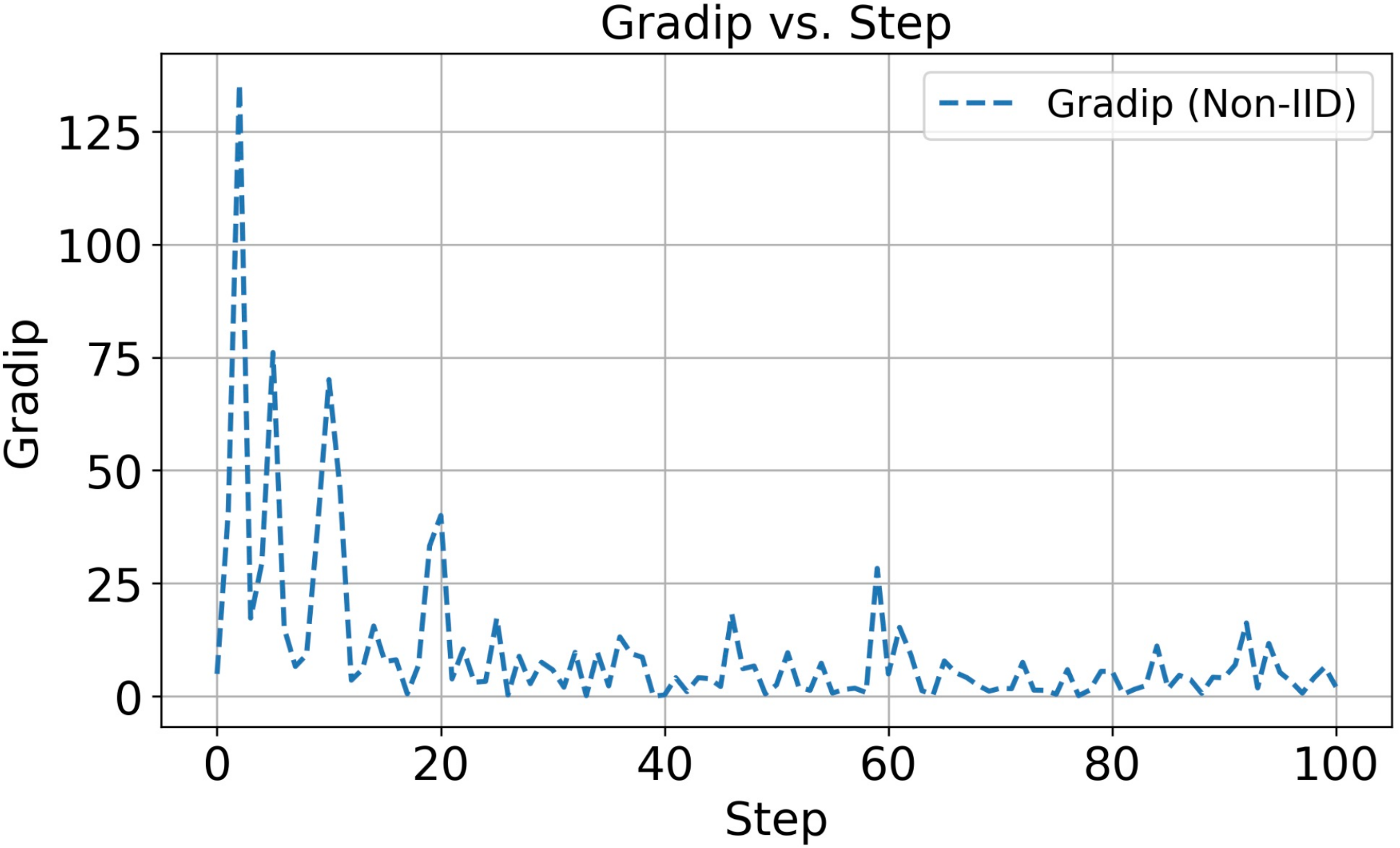}
        \caption{GradIP for Non-IID clients data on the BoolQ task, 
    where the two classes have a highly imbalanced ratio (6 vs.\ 190 samples).}
        \label{fig:ip_llama_boolq_real}
    \end{subfigure}
    \caption{These subfigures show GradIP (see Definition~\ref{def:gradip}) for LLaMA-3.2-1B under Non-IID client data 
with 100 local training steps. Subfigure~(a) uses AgNews (5 vs.\ 89), while Subfigure~(b) uses BoolQ (6 vs.\ 190).}
    \label{fig:gradip_real_noniid_agnews_boolq}
\end{figure*}

\begin{figure*}[htbp]
    \centering
    \begin{subfigure}[b]{0.72\textwidth}
        \centering
        \includegraphics[width=0.72\linewidth]{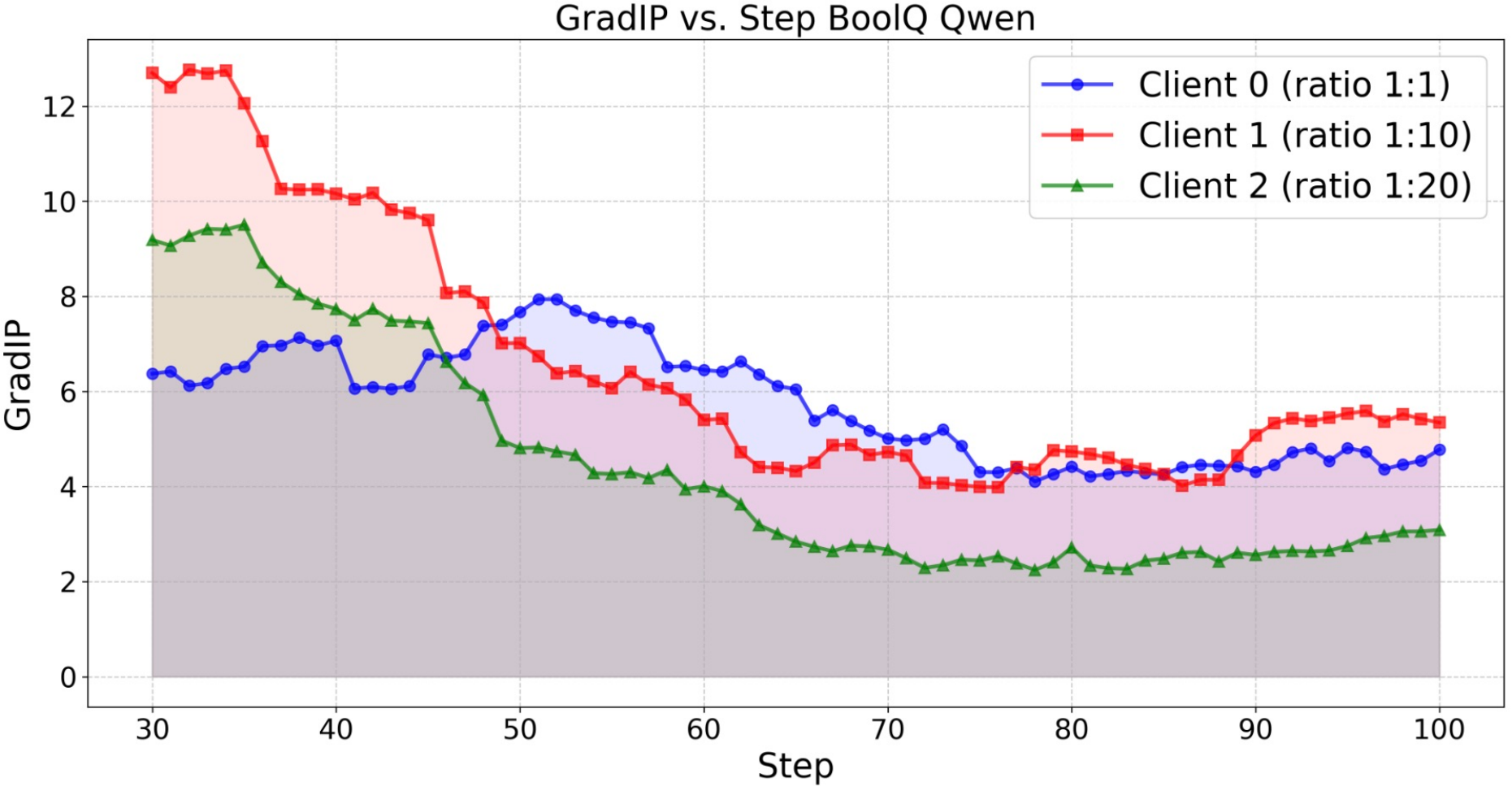}
        \caption{The experiments, conducted using the Qwen2-1.5B model on the BoolQ dataset, reveal that under Non-IID settings—especially with a 1:20 class imbalance—there is a pronounced decline in GradIP between the early and later stages of training. In the extreme Non-IID case, the GradIP values in the later stages tend to approach zero.}
        \label{fig:qwen_large_subfigure}
    \end{subfigure}
    
    \begin{subfigure}[b]{0.72\textwidth}
        \centering
        \includegraphics[width=0.72\linewidth]{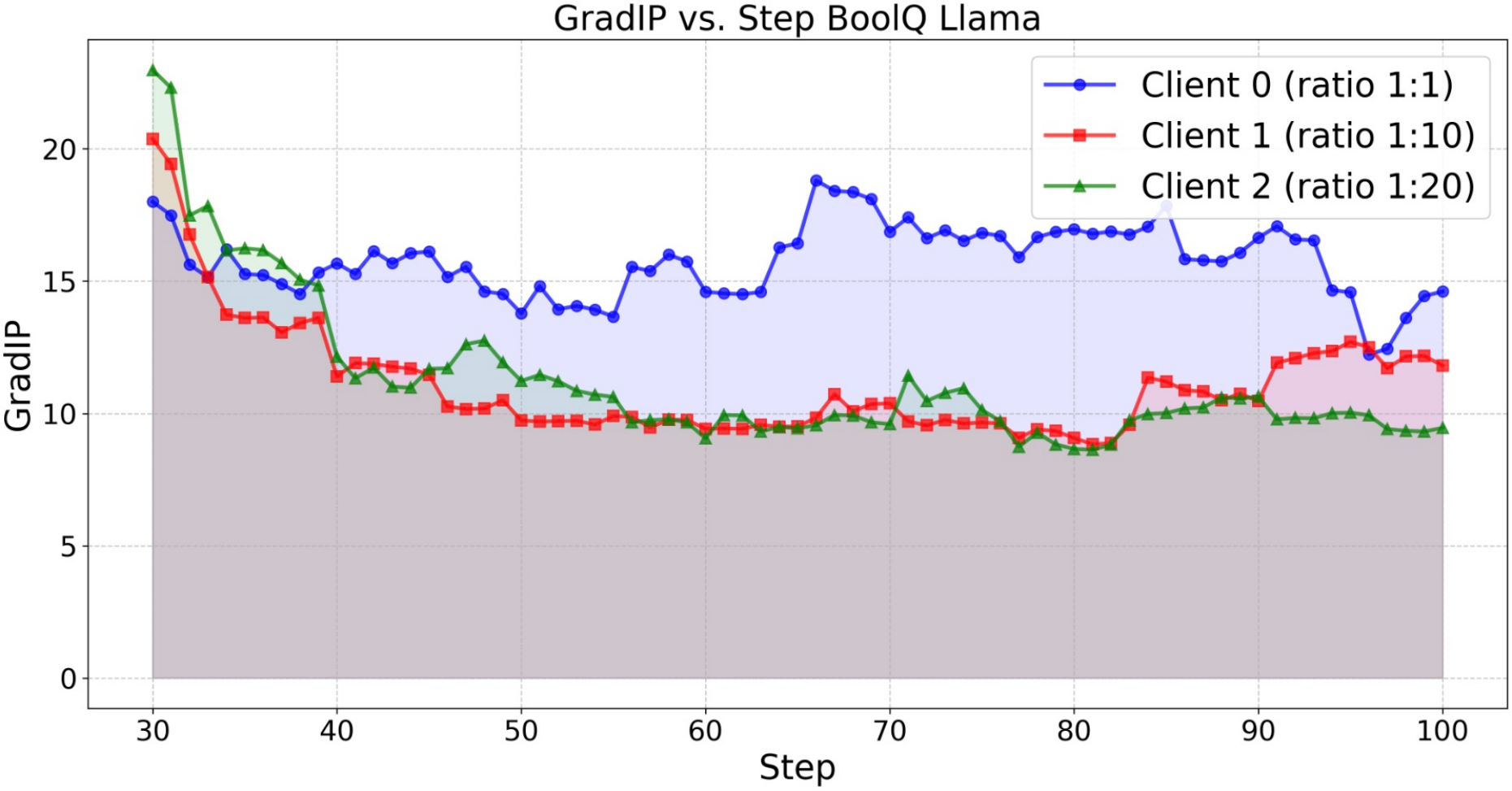}
        \caption{The experiments, conducted using the Llama-3.2-1B model on the BoolQ dataset, reveal that under Non-IID settings—especially with a 1:20 class imbalance—there is a pronounced decline in GradIP between the early and later stages of training. In the extreme Non-IID case, the GradIP values in the later stages tend to approach zero.}
        \label{fig:llama_large_subfigure}
    \end{subfigure}

    \begin{subfigure}[b]{0.72\textwidth}
        \centering
        \includegraphics[width=0.72\linewidth]{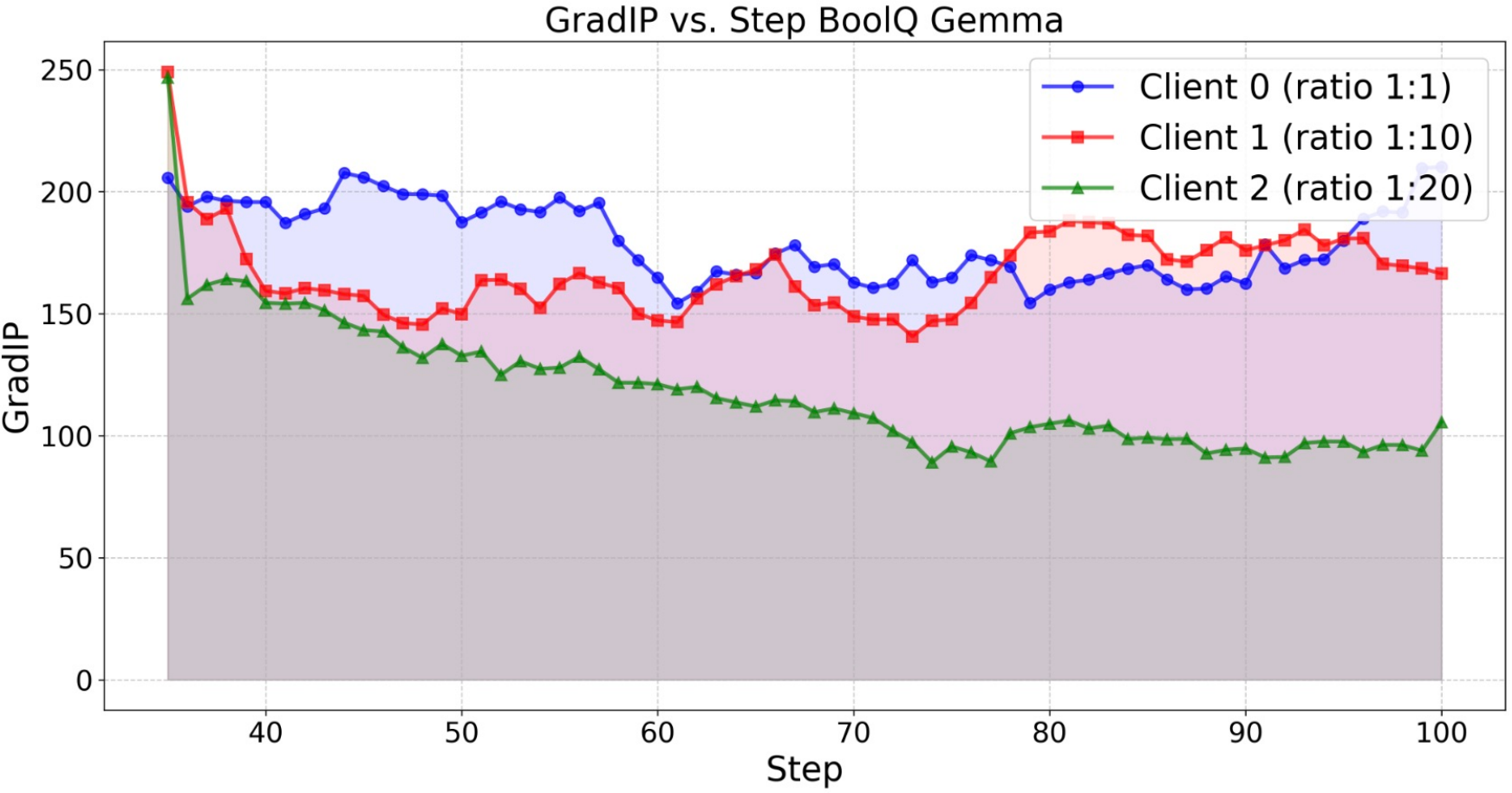}
        \caption{The experiments, conducted using the Gemma-2-2B model on the BoolQ dataset, reveal that under Non-IID settings—especially with a 1:20 class imbalance—there is a pronounced decline in GradIP between the early and later stages of training. In the extreme Non-IID case, the GradIP values in the later stages tend to approach zero.}
        \label{fig:gemma_large_subfigure}
    \end{subfigure}
    
    \caption{GradIP analysis for different models on the BoolQ dataset under Non-IID and IID conditions:
As the class imbalance ratio increases, GradIP in the later training stages tends to approach zero. This decline is more pronounced under Non-IID settings, where the gap between initial and final GradIP values is larger than in the IID case. All trends are visualized using a moving average for clarity.}
    \label{fig:model_comparison_gradip_ratiocompare_boolq}
\end{figure*}

\begin{figure*}[htbp]
    \centering
    \begin{subfigure}[b]{0.72\textwidth}
        \centering
        \includegraphics[width=0.72\linewidth]{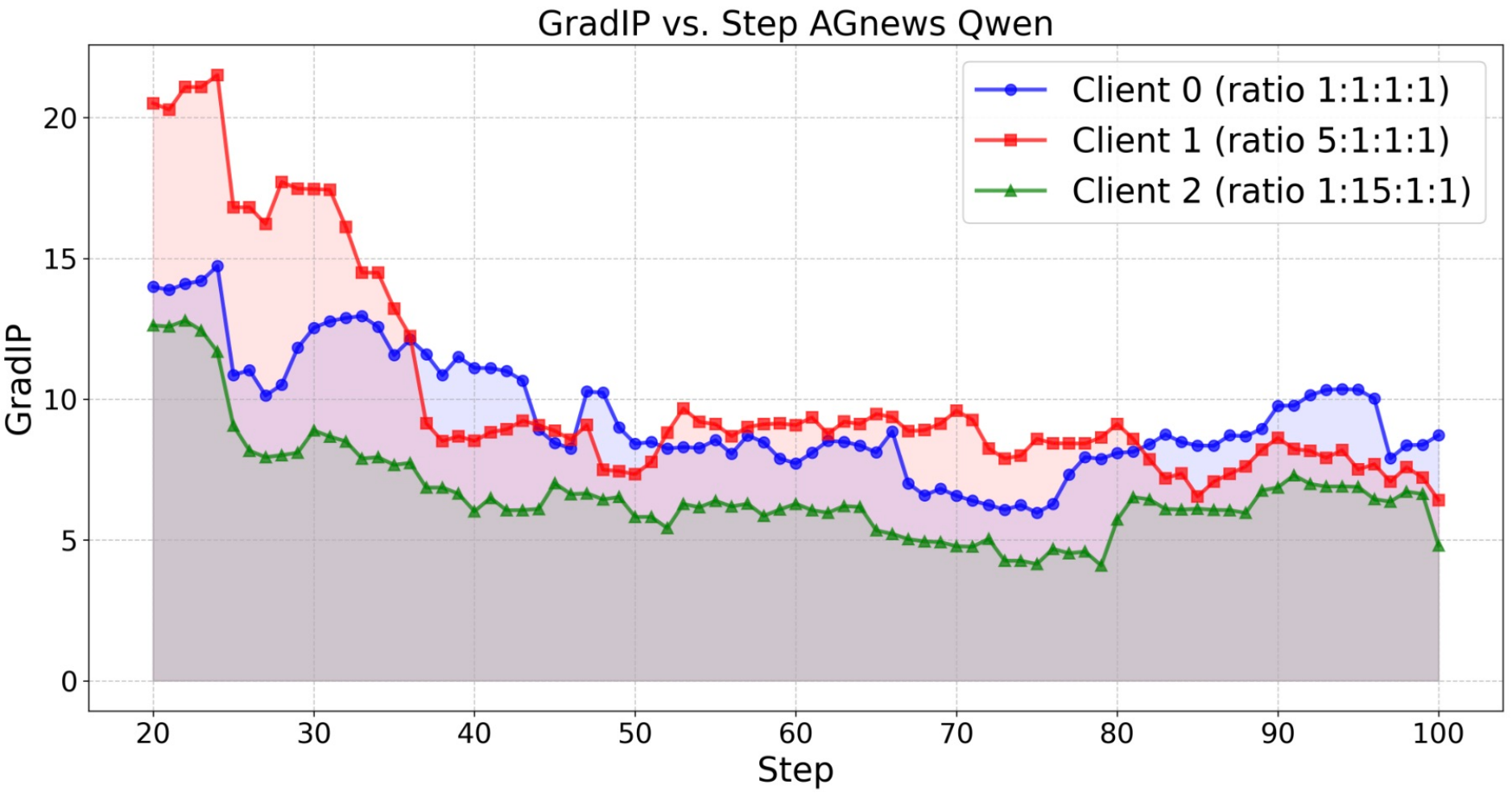}
        \caption{The experiments, conducted using the Qwen2-1.5B model on the AGNews dataset, reveal that under Non-IID settings—especially with a 1:15:1:1 class imbalance—there is a pronounced decline in GradIP between the early and later stages of training. In the extreme Non-IID case, the GradIP values in the later stages tend to approach zero.}
        \label{fig:qwen_large_subfigure_agnews}
    \end{subfigure}
    
    \begin{subfigure}[b]{0.72\textwidth}
        \centering
        \includegraphics[width=0.72\linewidth]{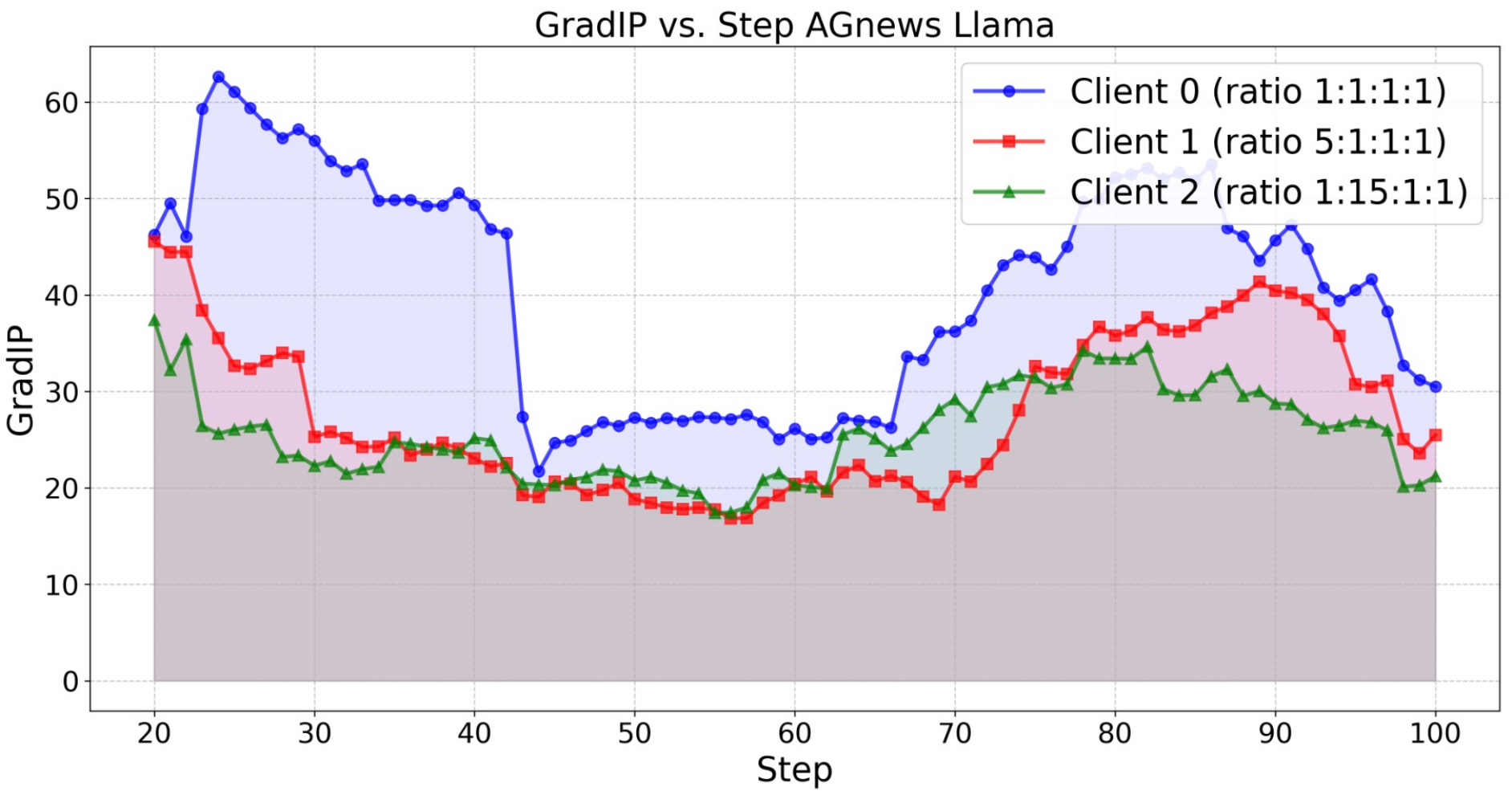}
        \caption{The experiments, conducted using the Llama-3.2-1B model on the AGNews dataset, reveal that under Non-IID settings—especially with a 1:15:1:1 class imbalance—there is a pronounced decline in GradIP between the early and later stages of training. In the extreme Non-IID case, the GradIP values in the later stages tend to approach zero.}
        \label{fig:llama_large_subfigure_agnews}
    \end{subfigure}
    
    \caption{GradIP analysis for different models on the AGNews dataset under Non-IID and IID conditions: As the class imbalance ratio increases, GradIP in the later training stages tends to approach zero. This decline is more pronounced under Non-IID settings, where the gap between initial and final GradIP values is larger than in the IID case. All trends are visualized using a moving average for clarity.}
    \label{fig:model_comparison_gradip_ratiocompare_agnews}
\end{figure*}

\end{document}